\documentclass[main]{article}

\usepackage[preprint]{neurips_2026}
\usepackage[utf8]{inputenc} 
\usepackage[T1]{fontenc}    
\usepackage[hidelinks]{hyperref}       
\usepackage{url}            
\usepackage{booktabs}       
\usepackage{amsfonts}       
\usepackage{nicefrac}       
\usepackage{microtype}      
\usepackage{xcolor}         

\usepackage{amsmath}
\usepackage[scr]{rsfso} 

\newcommand{\1}{\mathbf{1}} 
\newcommand{\amb}{\mathcal{K}} 
\newcommand{\borel}[1]{\mathscr{B}(#1)} 
\newcommand{\cdotx}{\,\cdot\,}
\newcommand{\conf}{\delta} 
\DeclareMathOperator{\diag}{diag}
\newcommand{\data}{\mathcal{D}}
\renewcommand{\d}{\mathrm{d}} 

\let\E\relax
\newcommand{\E}{\mathbb{E}} 

\newcommand{\Hilbert}{\mathscr{H}} 
\newcommand{\innerH}[3]{\langle #1, #2 \rangle_{#3}} 
\newcommand{\inv}{^{-1}} 
\newcommand{\law}{{\textcolor{red}{\boldsymbol{\mu}}}}
\newcommand{\lawseq}{{\textcolor{accent}{\boldsymbol{\mu}_T}}}
\newcommand{\me}{{\Psi}} 
\newcommand{\N}{\mathbb{N}} 
\newcommand{\norm}[1]{\left\lVert#1\right\rVert} 

\renewcommand{\P}{\mathbb{P}} 
\newcommand{\Probs}{\mathscr{P}} 
\newcommand{\p}{\mathbf{p}} 

\newcommand{\R}{\mathbb{R}} 

\newcommand{\SSR}[4]{#1\!\preceq^{#4}_{#3}\!#2} 
\newcommand{\T}{^{\top\!}} 

\newcommand{\x}{\mathbf{x}} 
\newcommand{\X}{\mathbf{X}}
\newcommand{\Xdomain}{\mathbb{X}} 

\newcommand{\y}{\mathbf{y}}
\newcommand{\Ydomain}{\mathbb{Y}} 

\usepackage{scalerel}

\usepackage{enumitem}\setlist{itemsep=6pt, topsep=0pt, parsep=0pt, leftmargin=*}       
\usepackage{subcaption}

\usepackage{tikz}
\usetikzlibrary{calc, automata, arrows.meta, positioning, shadows, fit, shapes.misc, decorations.pathreplacing, backgrounds, math, pgfplots.groupplots}
\tikzstyle{epibox} = [
    fill=grayfilling,
    drop shadow=grayshadow,
    rectangle, 
    text centered,
    rounded corners,
    inner sep=8pt,
    inner ysep=8pt
    ]
\tikzstyle{_epibox} = [
    fill=grayfilling,
    rectangle, 
    text centered,
    rounded corners,
    inner sep=8pt,
    inner ysep=8pt
    ]
    
\definecolor{accent}{rgb}{0.78431373,0,1}
\colorlet{accentwashed}{accent!10}
\definecolor{grayfilling}{gray}{0.95} 
\definecolor{grayshadow}{gray}{0.5} 

\usepackage[framemethod=tikz]{mdframed}

\newmdenv[
  backgroundcolor=accentwashed,
  shadow=false,
  hidealllines=true,
  leftline=true,
  linecolor=accent,
  linewidth=3pt,
  skipabove=3pt,
  skipbelow=-5pt,
  innerleftmargin=10pt,
  innerrightmargin=10pt,
  innertopmargin=7pt,
  innerbottommargin=5pt
]{accentbox}

\newmdenv[
  backgroundcolor=red!10,
  shadow=false,
  hidealllines=true,
  leftline=true,
  linecolor=red,
  linewidth=3pt,
  skipabove=4pt,
  skipbelow=-4pt,
  innerleftmargin=10pt,
  innerrightmargin=10pt,
  innertopmargin=10pt,
  innerbottommargin=10pt
]{redbox}

\newmdenv[
  backgroundcolor=gray!10,
  shadow=false,
  hidealllines=true,
  leftline=true,
  linecolor=gray,
  linewidth=3pt,
  skipabove=4pt,
  skipbelow=-4pt,
  innerleftmargin=10pt,
  innerrightmargin=10pt,
  innertopmargin=10pt,
  innerbottommargin=10pt
]{graybox}

\usepackage{amsthm}
\theoremstyle{plain}
\newtheorem{assum_}{Assumption}
\newtheorem{prop_}{Proposition}
\newtheorem{thm_}{Theorem}
\newtheorem{cor_}{Corollary}

\newtheorem{remark}{Remark}

\theoremstyle{definition}
\newtheorem{rem_}{Remark}
\newtheorem{prob_}{Problem}
\newenvironment{problem}
{\begin{accentbox}\begin{prob_}}
{\end{prob_}\end{accentbox}}
\newenvironment{assumption}
{\begin{assum_}}
{\end{assum_}}
\newenvironment{proposition}
{\begin{accentbox}\begin{prop_}}
{\end{prop_}\end{accentbox}}
\newenvironment{theorem}
{\begin{accentbox}\begin{thm_}}
{\end{thm_}\end{accentbox}}
\newenvironment{corollary}
{\begin{accentbox}\begin{cor_}}
{\end{cor_}\end{accentbox}}

\usepackage{import}
\usepackage{xifthen}
\usepackage{pdfpages}
\usepackage{transparent}

\title{Safety Certification is Classification}

\author{%
  Oliver Schön\thanks{Corresponding author}\\
  ETH Zürich\\
  Zürich, Switzerland \\
  \texttt{oschoen@ethz.ch} \\
\And
  Licio Romao\\
  Technical University of Denmark \\
  Copenhagen, Denmark \\
  \texttt{licio@dtu.dk} \\
\And
  Sadegh Soudjani\\
  Max Planck Institute for Software Systems\\
  Kaiserslautern, Germany \\
  \texttt{sadegh@mpi-sws.org} \\
}

\begin{document}

\maketitle

\begin{abstract}
    The goal of this paper is certifying safety of dynamical systems subject to uncertainty. Existing approaches use trajectory data to estimate transition probabilities, and compute safety probabilities recursively via dynamic programming (DP). This recursion may lead to compounding errors in the certified safety probability, thus collapsing to a vacuous lower bound for growing horizons $T$. 
    We propose a kernel embedding framework that treats safety certification as a classification problem on trajectory data, directly estimating the $T$-step safety probability without recursion.
    We show that the framework subsumes well-established approaches from the literature (e.g., barrier certificates, robust Markov models) as special cases, and allows us to go beyond their limitations. 
    As the main consequence, it bypasses compounding error across the horizon and enables certification for systems with non-Markovian dynamics.
    We demonstrate that direct estimators remain stable independent of the certification horizon and in the non-Markovian setting, whilst DP-based certificates silently go unsound---confirmed in simulation on a neural-controlled quadrotor.
\end{abstract}

\section{Introduction}\label{sec:intro}
The pace of adoption of learning-based control into safety-critical applications has reached an unprecedented level, evidenced by rapidly expanding real-world deployments of autonomous vehicles, medical devices, and robotics across various domains.
With this surge in deployment, the need for rigorous safety certification has become paramount~\citep{brunke2022safe,seshia2022toward}. 
Consider, as a concrete example, certifying that a quadrotor under a learned controller will safely navigate to its goal among obstacles and under stochastic disturbances.
In practice, asserting such safety problems is still approached almost exclusively through costly testing \citep{kalra2016driving}, including recent efforts to make testing more efficient through adversarial perturbations~\citep{zhang2022adversarial} or scenario generation~\citep{rempe2022generating}. 
In a world governed by non-determinism, however, no finite test set can cover every conceivable edge case, rendering testing---however cleverly sampled---unfit as a rigorous certification mechanism for large-scale deployment.

Formal certification methods, by contrast, \textbf{deliver mathematical guarantees} rather than empirical evidence---but most certification methods were designed for systems with known dynamics and rarely scale to the complexities of real-world deployments. For instance, in robotics, the leading approaches based on \emph{Hamilton--Jacobi reachability} suffer from exponential complexity in the system's dimensionality, often remaining limited to low-dimensional, single-agent scenarios~\citep{bansal2017hamilton}. 
Popular certification methods based on \emph{finite-state abstractions} approximate dynamics by grid-world models (e.g., robust \emph{Markov decision processes} (MDPs)) and suffer from exponential blow-up in the state dimension, limiting formal certification to low-dimensional systems~\citep{abate2023arch,suilen2025robust}.
Adaptations attempt to bridge this gap by learning certificates or abstractions from samples (e.g., kernel-based barrier learning~\citep{casablanca2026lucid}, scenario-based PAC abstractions~\citep{badings2022sampling}, and neural abstractions~\citep{abate2022neural}) but inherit similar exponential scaling and require sample-complexity costs of their own.
Evidently, the gap between systems that can be certified and those being deployed is widening, and the cause is structural.

Essentially every existing certification method, despite surface differences, follows the same \textbf{ill-posed structure} (Figure~\ref{fig:overview}, top panel): (1) estimate a one-step transition model from data, then (2) recursively propagate safety probability backward in time via (robust) \emph{dynamic programming} (DP). Certifying safety over a horizon $T$ thus requires $T$ evaluations of a learned estimator. To obtain a certified safety probability, bounds must be inflated at each step---collapsing toward a vacuous lower bound as $T$ grows. In addition, the DP characterisation requires a Markov property that realistic closed-loop systems with neural controllers or temporally correlated disturbances routinely violate, in which case the inflated bound is no longer even sound, and the \textbf{certificate may silently fail at deployment} (as will be demonstrated in the experiments in Section~\ref{sec:experiments}).

\emph{``When solving a given problem, try to avoid solving a more general problem as an intermediate step''}---a principle attributed to \cite{vapnik1995nature} that names exactly the structure existing methods follow. A transition kernel is a high-dimensional object describing every possible future of a dynamical system; its safety probability is a single scalar. Existing approaches estimate the system's transition kernel in order to obtain a scalar safety probability, and with finite data this detour is precisely what forces uncertainty to compound across the horizon. We call such methods \emph{indirect}, in contrast to the \emph{direct} alternative we propose (Figure~\ref{fig:overview}, bottom).

\textbf{Core Claim.}\;
We show that safety certification from trajectory data can be formulated as a classification problem on the trajectory space: given sampled trajectories, a kernel-based regressor from initial conditions yields a distribution-free certified lower bound on the safety probability. The bound remains stable across increasing horizons by construction, valid under non-Markovian dynamics, and emerges from a single embedding framework that recovers existing indirect methods as special cases.

\begin{figure}[t]
    \centering
    \def\svgwidth{1\linewidth}
    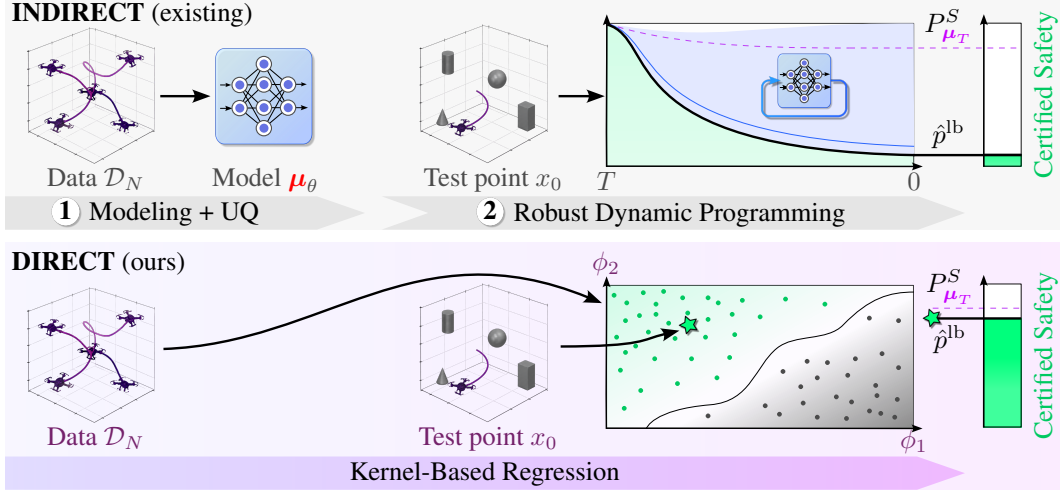

    \vspace{-1em}
    \caption{\textbf{Indirect} (top) vs.\ \textbf{direct} (bottom) safety certification on a quadrotor obstacle-avoidance task. In the learning-based setting, the DP recursion compounds estimation error at every step (UQ: uncert. quantification); for realistic horizons, direct certification is the only viable route.}
    \label{fig:overview}
\end{figure}

\subsection{Contributions}\label{sec:contributions}
\begin{enumerate}[label=\arabic*)]
  \item \textbf{Equivalence
    (Section~\ref{sec:beyond_dp}, Theorem~\ref{thm:empirical_safety_nonMarkov}):}
    A direct, non-recursive estimator that reduces probabilistic safety
    certification to kernel regression on trajectory data, yielding
    distribution-free certified lower bounds valid without Markov assumptions.
  \item \textbf{Unified framework
    (Sec.~\ref{sec:framework}, Fig.~\ref{fig:unified}):}
    A kernel-embedding framework recovering barrier certificates~\citep{prajna2006barrier},
    finite-state abstractions~\citep{tabuada2009verification}, interval MDPs~\citep{suilen2025robust}, and prob.\
    simulation relations~\citep{haesaert2017verification} as indirect solvers
    of the same problem.
  \item \textbf{End-to-end certification of a neural-controlled system
    (Section~\ref{sec:experiments}):}
    On a quadrotor closed under a learned controller with internal adaptive state,
    our method produces calibrated, sound bounds where existing
    baselines silently go unsound.
\end{enumerate}

\subsection{Related Work}\label{sec:related_work}

\textbf{Safety Certification.}\;
We refer the reader to~\citet{Lavaei_Survey,yin2024formal} for surveys; here we focus on the work most directly relevant.
Post-training approaches split into (1)~\emph{intervention methods}---safety filters, shielding, and control barrier functions---which modify the controller at runtime~\cite[see, e.g.,][Section~3.3]{brunke2022safe}, and (2)~\emph{certification methods}---barrier certificates, Hamilton--Jacobi reachability, and finite-state abstractions---which bound safety probabilities without altering behavior~\citep{dawson2022safe,Lavaei_Survey,yin2024formal}.
Our focus is the latter.
Learning-based adaptations include kernel-based barrier learning~\citep{casablanca2026lucid}, PAC scenario abstractions~\citep{badings2022sampling,suilen2022robust}, neural abstractions~\citep{abate2022neural}, and CME-based value iteration~\citep{romao2023distributionally}---all following the indirect, DP-based structure we identify as the source of compounding conservatism.
No learning-enabled toolbox for PAC-IMC currently exists; existing implementations assume a given model~\citep{mathiesen2024intervalmdp}.

\textbf{Safe Reinforcement Learning (RL).}\;
Safe RL and constrained MDP methods~\citep{garcia2015comprehensive,berkenkamp2017safe,ni2023safe} synthesize policies satisfying safety constraints during or at convergence of training---their output is a \emph{policy}; ours takes a given (possibly learned) policy and outputs a certified bound on its safety probability.

\textbf{Conformal Prediction.}\;
Conformal methods offer marginal coverage guarantees over the test distribution~\citep{vovk2005algorithmic}; our bound is \emph{spatially discriminative}: when the predictor is well-calibrated, different initial states receive different certified bounds, using the distribution-free binning procedure of \citet{gupta2020distribution}.
Standard split-conformal prediction yields only a marginal prediction-set guarantee for binary outcomes, not a confidence interval for $P^S_{\lawseq}(x_0)$~\citep{gupta2020distribution}.

\section{Safety Certification of Dynamical Systems}
\textbf{Notation.}\;
All sets and maps are assumed measurable; we defer the underlying measure-theoretic   
setup (probability spaces $(\Xdomain,\borel{\Xdomain},\P)$, Borel $\sigma$-algebras $\borel{\Xdomain}$, probability kernels) to Appendix~\ref{app:notation}.
Furthermore, let $X_N:=[x_i]_{i=1}^N$ be a column vector with $x_i\in\Xdomain$.
We denote the element-wise evaluation of a function $f\colon\Xdomain\rightarrow\R$ on $X_N$ as $f(X_N):=\smash{[f(x_i)]_{i=1}^N}$.
Similarly, we may write $A = \smash{[a_{ij}]_{i,j=1}^N}$ to denote a matrix with its elements.

\subsection{Problem Formulation}
We briefly formalize the class of systems considered in this paper and the safety certification problem.

\textbf{System Dynamics \& Data.}\;
We write a discrete-time stochastic process $\X_{0:T}:=\{X_t\}_{t=0}^T\in\Xdomain^{T+1}$ with time horizon $T\in\N$ and stochastic law \smash{$\lawseq\colon\Xdomain\times\borel{\Xdomain^{T}}\to[0,1]$} with realizations $\x_{1:T}\sim\lawseq(\cdotx| X_0 =x_0)$ for some initial state $x_0\in\Xdomain_0\subset\Xdomain\subset\R^d$.%
\footnote{Random variables are uppercase (\smash{$X_t\in\Xdomain$}), instances lowercase (\smash{$x_t\in\Xdomain$}), and trajectories bold (\smash{$\x_{0:T}\in\Xdomain^{T+1}$}).}
We assume that the system dynamics are unknown but that we have access to a sample of $N\in\N$ i.i.d.\ sequences $\data_N:=\smash{\{\x_{0:T}^{(i)}\}_{i=1}^N}$, obtained by running the system from initial states drawn uniformly from $\Xdomain_0$.

\textbf{Safety Problem.}\;
We want to compute the probability of the system initialized at some $x_0\in\Xdomain_0$ being \emph{safe}---i.e., $x_t$ remaining in a safe area $S\subset\Xdomain$---for $T\in\N$ time steps.
The \emph{safety probability} is 
\begin{equation}
     P^S_{\lawseq}(x_0) := \E_{\lawseq}\!\!\left[\min_{t\in\{0,\ldots,T\}} \1_{S}(X_{t})\mid X_0=x_0\right],\label{eq:safety_probability}
\end{equation} 
with $\1_S$ the indicator function of $S$, evaluating to $1$ if $X_t\in S$ and $0$ otherwise.
As the dynamics are unknown, we are interested in obtaining a certified lower bound on $P^S_{\lawseq}$ from trajectory data $\data_N$:

\begin{problem}\label{prob:data_driven_safety_verification}
    Given data $\data_N\!:=\!\smash{\{\x_{0:T}^{(i)}\}_{i=1}^N}$,
    find a lower bound $\hat{p}^{\mathrm{lb}}\colon\Xdomain_0\to[0,1]$ such that with probability at least $1\!-\!\conf\in[0,1]$ we have
    $P^S_{\lawseq}(x_0)\geq \hat{p}^{\mathrm{lb}}(x_0)$ for any given initial state $x_0\!\in\!\Xdomain_0$.
\end{problem}

\section{Kernel Mean Embedding Framework}\label{sec:CMEs}
We include a brief introduction to \emph{reproducing kernel Hilbert spaces} (RKHSs) and \emph{(conditional) kernel mean embeddings} (CMEs)~\citep{Berlinet2004RKHSProbStat,Song2009CondEmbed}, which are the mathematical tools underlying our framework, in Appendix~\ref{app:RKHS}. Throughout, $k$ denotes a positive-definite kernel with RKHS $\Hilbert$, feature map $\phi$, and reproducing property $f(x)=\smash{\innerH{f}{\phi(x)}{\Hilbert}}$.

\textbf{Conditional Mean Embeddings.}\;
Given a kernel $k\colon\Xdomain\times\Xdomain\to\R$ with an associated RKHS $\Hilbert$, a probability law $\law$ can be encoded as a \emph{(kernel) conditional mean embedding} (CME) $\smash{\me_{\Hilbert|\Hilbert}^{\law}}\colon\Xdomain\to\Hilbert$:
\begin{equation}
    \me_{\Hilbert|\Hilbert}^{\law} := \int_{\Xdomain} \phi(x_+) \,\d\law(x_+\mid \cdotx)\label{eq:CME}
\end{equation}
If the kernel $k$ is characteristic, and thus $\phi$ is injective, the corresponding CME $\smash{\me_{\Hilbert|\Hilbert}^{\law}}$ encodes $\law$ uniquely~\citep{Park2020MeasureTheoretic}.

The reproducing property of the kernel carries on to the CME, facilitating the computation of the conditional expected value of a function $f\in\Hilbert$ under $\law$ via the inner product with the CME, that is
\begin{equation}
	\E_{\law}[f(X_+)\mid X=x] = \innerH{f}{\me_{\Hilbert|\Hilbert}^{\law}(x)}{\Hilbert}\quad\text{almost surely}.\label{eq:innerProdCond}
\end{equation}

\textbf{Empirical Estimator.}\;
In practice, the CME $\smash{\me_{\Hilbert|\Hilbert}^{\law}}$ is generally unknown (and potentially infinite dimensional); it can be approximated via an empirical estimator from a finite set of data $\data_{\hat{N}}$:
\begin{equation}
    \hat\me_{\Hilbert|\Hilbert}^{\data_{\hat{N}}}(x) := \sum_{i=1}^{\hat{N}} w_i(x) \phi(x_+^{(i)}), \quad \text{with} \quad w(x) := k_{\hat{N}}(x)\T[K_{\hat{N}}+\hat{N}\lambda I_{\hat{N}}]\inv,\label{eq:empirical_CME}
\end{equation}
where \smash{$K_{\hat{N}}$} is the Gram matrix given data \smash{$\data_{\hat{N}}$} and $\lambda>0$ is a regularization constant.
The empirical CME converges to the true CME as $\smash{\hat{N}\to\infty}$ under standard assumptions~\citep{Park2020MeasureTheoretic}.

\textbf{Robustness.}\;
To obtain statistically robust estimates, it is common practice to introduce an ambiguity set of candidate embeddings centered at the empirical CME \eqref{eq:empirical_CME} 
that includes the true dynamics with probability at least $1-\conf$.
For this, let $\mathcal{G}$ be a vector-valued RKHS of mappings $\Xdomain\to\Hilbert$ \cite[Section~2.3]{Park2020MeasureTheoretic}.
\begin{assumption}\label{asm:ambiguity_set}
    We assume access to an ambiguity set $\amb_{\hat{N}}\subset\mathcal{G}$ with radius $\varepsilon\geq0$, defined by
    \begin{equation}
        \amb_{\hat{N}} := \big\lbrace \bar\mu\colon\Xdomain\times\borel{\Xdomain}\to[0,1] \,\big\vert\, \big\Vert \hat\me_{\Hilbert|\Hilbert}^{\data_{\hat{N}}} - \me_{\Hilbert|\Hilbert}^{\bar\mu} \big\Vert_{\mathcal{G}}\leq\varepsilon
        \big\rbrace,\label{eq:amb_set}
    \end{equation}
    such that $\me_{\Hilbert|\Hilbert}^{\law}\in\amb_{\hat{N}}$ with probability at least $1-\conf$.
\end{assumption}
There exist results relating the radius $\varepsilon$ of the ambiguity set to the probability that the actual CME of the unknown dynamics $\law$ lies within the set with a minimax rate of $\smash{\mathcal{O}(\log(\hat{N})/\hat{N})}$~\citep{li2022optimal,mollenhauer2022learning}.
We will instead use the distribution-free procedure described in App.~\ref{app:conformal}.

\section{A Unifying Embedding Framework}\label{sec:framework}

Many existing methods restrict to \emph{Markovian} systems, where the following assumption holds.

\begin{assumption}[Markovianity]\label{asm:markovianity}
    There exists a stochastic law $\law\colon\Xdomain\times\borel{\Xdomain}\to[0,1]$ s.t.\ \smash{$\lawseq\equiv\law^T$}.
\end{assumption}
Learning-based approaches from the Markovian regime commonly assume access to i.i.d.\ one-step transition data $\data_{\hat{N}}:=\smash{\{(x^{(i)},x^{(i)}_+)\}_{i=1}^{\hat{N}}}$ sampled via
$x_+^{(i)}\sim\law(\cdotx\mid X_0 = x^{(i)})\,\mathcal{U}_{\Xdomain_0}(dx^{(i)})$.
Weaker conditions (ergodicity, mixing, burn-in time, or risk-based notions) can replace the i.i.d.\ assumption at the cost of a lower effective sample size~\citep{massiani2024consistency,steinwart2009fast,ziemann2022learning,zhang2024guarantees}.

\textbf{Dynamic Programming Reformulation.}\;
For Markovian systems, \emph{dynamic programming} (DP) allows to compute the safety probability \smash{$P^S_{\lawseq}$} in \eqref{eq:safety_probability} by unfolding its computation into recursive updates of a latent value function (backwards in time) \citep{abate2008probabilistic}:
\begin{proposition}[Model-Based Dynamic Programming]\label{prop:value_iteration}
    Under Assumption~\ref{asm:markovianity}, consider the value function $V_l\colon\Xdomain\to[0,1]$, with $l=T,\ldots,0$, where
    \begin{equation}
        V_T(x) := \1_{S}(x)\quad\text{and}\quad
        V_l(x) := \1_{S}(x) \, \E_{\law}[V_{l+1}(X_+)\mid X=x].
    \label{eq:value_iteration}
    \end{equation}
    Then, the safety probability \eqref{eq:safety_probability} is $P^S_{\lawseq}(x_0) = V_0(x_0)$, for any initial state $x_0\in\Xdomain_0$.
\end{proposition}

Substituting the empirical CME \eqref{eq:empirical_CME} into \eqref{eq:value_iteration} and robustifying against the ambiguity set $\amb_{\hat{N}}$ (Asm.~\ref{asm:ambiguity_set}) gives a learning-based safety bound (recovering the results of \cite{romao2023distributionally}):
\begin{proposition}[Empirical DP]\label{prop:empirical_value_update}
    Suppose Asms.~\ref{asm:ambiguity_set} \& \ref{asm:markovianity} hold.
    Let $\underline V_T(x) := \1_{S}(x)$ and
    \begin{equation*}
        \underline V_l(x) := \1_{S}(x) \left[ \sum_{i=1}^{N} w_i(x)\, \underline V_{l+1}(x_+^{(i)}) - \varepsilon \kappa \norm{\underline V_{l+1}}_\Hilbert \right]_0^1, \quad \kappa\geq\sup_{x\in\Xdomain}\sqrt{k(x,x)}.
    \end{equation*}
    Then $P^S_\lawseq(x_0)\geq\underline V_0(x_0)$ for all $x_0\in\Xdomain_0$ with probability at least $1-\conf$.
\end{proposition}
The proof is given in Appendix~\ref{app:proof_empirical_value_update}, which also discusses convergence.

\begin{figure}[b]
    \centering
    \begin{tikzpicture}[
            node distance=.5cm,
            every node/.style={font=\footnotesize},
            box/.style={draw, epibox, align=center, minimum height=1cm, minimum width=2.2cm, fill=white},
            _box/.style={_epibox, align=center, minimum height=1cm, minimum width=2.2cm, fill=white},
            accentbox/.style={draw, epibox, align=center, minimum height=1cm, minimum width=2.2cm, accent, fill=accentwashed, text=black, inner xsep=5pt, inner ysep=5pt},
            arrow/.style={-latex},
        ]
        \pgfdeclarelayer{layer_framework}
        \pgfdeclarelayer{layer_dp}
        \pgfdeclarelayer{layer_barrierorabstractions}
        \pgfsetlayers{layer_framework,layer_dp,layer_barrierorabstractions,main} 

        \node[align=center, inner sep=0pt] (input) {Data $\data_N$};

        \node[accentbox, right=.5cm of input] (safety) {\textbf{Certified Safety Probability (Problem~\ref{prob:data_driven_safety_verification}):}\; 
        $\P^S_{\lawseq}\!(x_0) \!:=\! \E_{\lawseq}\!\!\left[\min_{t\in\{0,\ldots,T\}} \1_{S}(X_{t})\mid X_0=x_0\right]$};

        \node[box, below=2.7cm of safety, fill=black!5, xshift=-4.41cm, minimum width=5.5cm, inner xsep=5pt] (barriers) {\textbf{Barrier Certificates (Proposition~\ref{prop:barrier}):}\\[.2em]
        $\beta \geq \sup_{x\in S} \left(\E_{\law}[B(X_+)\mid X=x]-B(x)\right)$};

        \node[_box, right=.5cm of barriers, inner ysep=0pt, yshift=35pt, inner xsep=0pt, fill=black!5] (abstractions) {\textbf{Finite-State Abstractions (App.~\ref{app:additional_material_abstractions}, Eqs.~\ref{eq:robust_finite_value_update}--\ref{eq:amb_set_finite}):}\\[.2em]
        $ V_l(\hat{x}_i) \!\geq\! \inf_{\bar\mu\in\amb^f_{\hat N},\,x\in A_i} \! \1_{S}(x) \innerH{V_{l+1}}{\me_{\Hilbert_f|\Hilbert_f}^{\bar\mu}\!(\hat{x}_i)}{\!\Hilbert_f}\!$};
        \node[box, below=.2cm of abstractions, inner xsep=5pt] (imps) {\textbf{Interval Markov Processes (Proposition~\ref{prop:imp}):}\\[.2em]
        Interval-based $\amb^{\text{IMP}}_{\hat{N}}$ in \eqref{eq:amb_set_IMP}};
        \node[box, below=.2cm of imps, inner xsep=5pt] (simrels) {\textbf{Probab. Simulation Relations (Theorem~\ref{thm:ssr}):}\\[.2em]
        Total-variation-based $\amb^{\text{SSR}}_{\hat{N}}$ in \eqref{eq:amb_set_SSR}};
        \begin{pgfonlayer}{layer_barrierorabstractions}
            \node[box, fit=(abstractions)(imps)(simrels), fill=black!5, minimum width=7.0cm, inner xsep=5pt, inner ysep=5pt] (abstractions_box) {};
        \end{pgfonlayer}

        \node[below=1.9cm of safety] (_dp_text) {};
        \begin{pgfonlayer}{layer_dp}
            \node[box, fit=(_dp_text)(barriers)(abstractions_box), inner xsep=5pt, inner ysep=5pt] (dp) {};
        \end{pgfonlayer}
        \node[anchor=north east, black, align=left] at ($(dp.north)+(-1.2,-0.1)$) {\textbf{Dynamic Programming (Proposition~\ref{prop:value_iteration}):}\\[.2em]
        {$\begin{aligned}
        V_T(x) &:= \1_{S}(x),\\
        V_l(x) &:= \1_{S}(x) \, \E_{\law}[V_{l+1}(X_+)\mid X=x]
        \end{aligned}$}};

        \node[below=.6cm of safety, inner xsep=5pt] (_embedding_text) {};
        \begin{pgfonlayer}{layer_framework}
            \node[accentbox, fit=(_embedding_text)(dp), inner xsep=5pt, inner ysep=5pt] (embedding_framework) {};
        \end{pgfonlayer}
        \node[anchor=north east, black, align=left] at ($(embedding_framework.north)+(5.4,-0.1)$) {\textbf{Embedding Framework (Theorem~\ref{thm:empirical_safety_nonMarkov}):}\;
        $P^S_{\lawseq}\!(x) 
        \!\geq\!  \sum_{i\colon \x^{(i)}_t\in S,\,\forall t\in\{0,\ldots,T\}}  w_i(x) \!-\! (\varepsilon_1(x) \!+\! \varepsilon_2 \!+\! \varepsilon_3)$};        
        
        \draw[arrow] (input) -- (safety);
        \draw[arrow] (safety) -- ($(embedding_framework.north)+(0.75,0.0)$);
        
    \end{tikzpicture}
    
    \caption{Overarching embedding framework recovering existing approaches and going beyond}
    \label{fig:unified}
\end{figure}
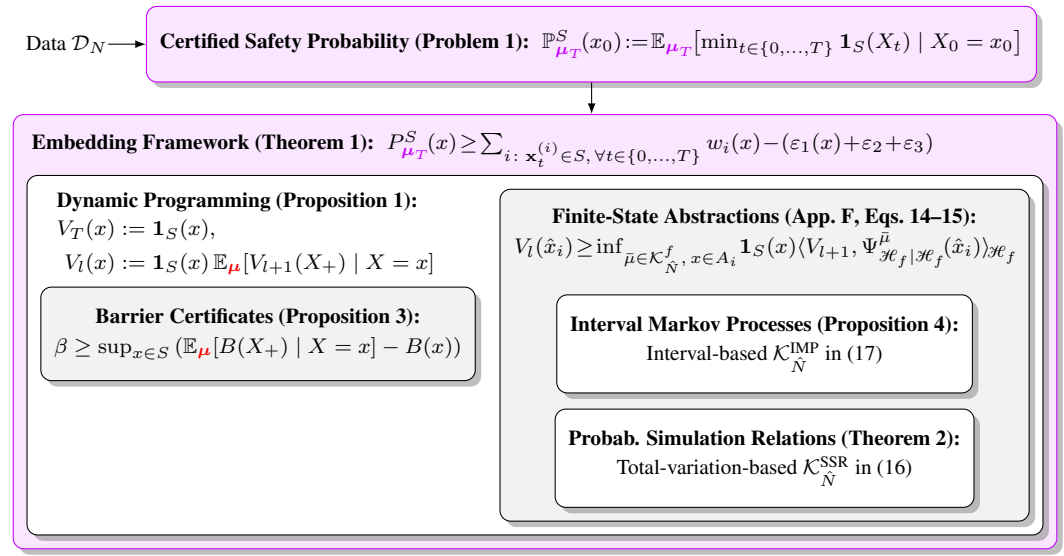

\subsection{Recovering Existing Approaches from the Embedding Framework}

The embedding framework introduced above is general enough to capture existing certification approaches, e.g., through specific choices of the kernel and the ambiguity set (see Figure~\ref{fig:unified} for an overview).
In particular, we recover both barrier certificates and well-known finite-state abstraction methods as special cases (full details in Appendix~\ref{app:embedding_framework_recoveries}).

\subsubsection{Barrier Certificates}\label{sec:safety_via_barriers}
In contrast to the point-wise safety probability $\smash{P^S_{\lawseq}\colon}\Xdomain_0\to[0,1]$ queried in Problem~\ref{prob:data_driven_safety_verification}, barrier methods usually aim to obtain a uniform robust bound $p\in[0,1]$ of the form
\[
P^S_{\lawseq}(\Xdomain_0):=\inf_{x_0\in\Xdomain_0} P^S_{\lawseq}(x_0)\geq p,
\]
i.e., a robust safety probability that holds uniformly for all initial states $x_0\in\Xdomain_0$~\citep{prajna2006barrier}.
This is achieved by finding a barrier function $B$ approximating the value function solving the ``unsafety'' probability DP $\Lambda_l(x):=1-V_l(x)$~\citep{laurenti2025unifying}; condition~(c) below bounds the maximum growth of the probability of exiting the safe set $S$ in a single time step, giving $P^S_{\lawseq}(\Xdomain_0)\geq 1-\sup_{x\in\Xdomain_0}\Lambda_0(x)$.
We defer the DP derivation to Appendix~\ref{app:barriers}.

Given a barrier $B\in\Hilbert$ with $\norm{B}_\Hilbert<\infty$ satisfying the standard level-set conditions ($B\leq\eta$ on $\Xdomain_0$, $B\geq\gamma$ on $\Xdomain\setminus S$, with level sets $\gamma>\eta\geq 0$), the learning-based one-step bound requires
\begin{equation}
    \beta \geq \sup_{x\in S} \sum_{i=1}^N w_i(x) B(x^{(i)}) - B(x) + \varepsilon\kappa\norm{B}_\Hilbert, \label{eq:barrier_condition}
\end{equation}
yielding \smash{$P^S_{\lawseq}(\Xdomain_0) \geq 1 - (\eta+\beta T)/\gamma$} with probability at least $1-\conf$ (Proposition~\ref{prop:barrier}, App.~\ref{app:barriers}).
Alternative approaches (scenario-based, sample-averaged~\citep{salamati2024data,mathiesen2024data}) have been proposed for bounding $\beta$, yielding analogous forms of \eqref{eq:barrier_condition}, all inheriting the same error-compounding structure: the $\beta T$ term in $1-(\eta+\beta T)/\gamma$ grows linearly with $T$.

\subsubsection{Finite-State Abstractions}\label{sec:safety_via_finite_abstractions}
Many data-driven verification approaches build on \emph{finite-state abstractions}, obtained by partitioning $\Xdomain$ into finitely many disjoint cells $A_i\subset\Xdomain$, $i=1,\ldots,n$, each represented by a symbol $\smash{\hat{x}_i\in\Xdomain}$ (see Fig.~\ref{fig:finite_abstractions}, App.~\ref{app:additional_material_abstractions}), and inducing a projection $\Pi_{\Xdomain_f}\!(x):=\{\hat{x}_i\in\Xdomain_f\colon x\in A_i\}$.
Abstracting continuous-state systems into finite-state counterparts offers computational tractability and enables the use of established model-checking techniques~\citep{clarke1997model}.
Guarantees for the original continuous-state system require accounting for both \emph{statistical error} (finite data) and \emph{abstraction error} (discretization).

\textbf{Finite-State Embedding.}\;
In the embedding view adopted throughout this paper, finite-state abstractions correspond to a finite-dimensional, non-characteristic RKHS $\Hilbert_f$ of piecewise-constant functions, with feature map $\phi_f\colon\Xdomain\to\Hilbert_f$ mapping each $x_t$ to the representative $\hat{x}_i$ of its cell $A_i$.
The kernel $\smash{k_f(x,x')} := \smash{\innerH{\phi_f(x)}{\phi_f(x')}{\Hilbert_f}}$ is partition-wise constant, and the trajectory $\x_{0:T}$ is encoded as a tensor \smash{$\phi_f(\x_0)\otimes\cdots\otimes\phi_f(\x_T)\in\Hilbert_f^{T+1}$} (see App.~\ref{app:additional_material_abstractions} for an illustration).
The corresponding ambiguity set \smash{$\amb^f_{\hat N}$} must account for both \emph{statistical error} (finite data) and \emph{abstraction error} (discretization), both of which enter at every recursion step; see App.~\ref{app:additional_material_abstractions} for the explicit construction.

Two canonical constructions of \smash{$\amb^f_{\hat N}$} recover well-known abstraction frameworks: \emph{probabilistic simulation relations} \citep{haesaert2017verification}, inducing a total-variation ambiguity set \smash{$\amb^{\text{SSR}}_{\hat N}$} (Theorem~\ref{thm:ssr}), and \emph{interval Markov processes} \citep{suilen2025robust}, inducing an interval ambiguity set \smash{$\amb^{\text{IMP}}_{\hat N}$} (Proposition~\ref{prop:imp}); derivations are deferred to Appendix~\ref{app:sim_rel_imp}.

Both barrier- and abstraction-based approaches implement the same recursive DP structure---one in continuous function space, the other in finite symbolic space. In both cases, uncertainty enters at every recursion step, causing conservatism to compound with the horizon length. 
This shared limitation exposes a deeper question: can we avoid recursion altogether and estimate the multi-step safety probability directly?

\section{Breaking Free from Dynamic Programming: Beyond Markovianity}\label{sec:beyond_dp}

The previous sections highlighted that DP---and thus its derivative methods based on barriers or finite abstractions---is structurally ill-suited for learning-based certification: the recursive invocation of the empirical CME compounds uncertainty across time, so that even small estimation errors are repeatedly amplified (as will be demonstrated in Section~\ref{sec:experiments}).
To break this recursive dependency, we instead estimate the safety probability in \eqref{eq:safety_probability} as a single-step inference task directly from data.

\textbf{Multi-Step Embedding.}\;
To estimate \eqref{eq:safety_probability} directly, we embed the $T$-step dynamics $\lawseq$ via a product kernel $k^{T+1}\colon\Xdomain^{T+1}\!\times\Xdomain^{T+1}\!\to\R$, defined by $k^{T+1}(\x_{0:T},\y_{0:T}):={\prod_{t=0}^T k(\x_t,\y_t)}$:
\begin{equation}
    \me_{\Hilbert|\Hilbert^{T+1}}^{\lawseq}(x_0) := \int_{\Xdomain^{T}} \phi^{T+1}(\x_{0:T}) \,\d\lawseq(\x_{1:T}\mid x_0),\label{eq:CME_Tstep}
\end{equation}
where $\phi^{T+1}(\x_{0:T}) := \otimes_{t=0}^T \phi(\x_t) \in \Hilbert^{T+1}$ maps the full trajectory $\x_{0:T}:=(x_0,\x_{1:T})\in\Xdomain^{T+1}$ into RKHS.
Given trajectory data $\data_N=\smash{\{\x_{0:T}^{(i)}\}_{i=1}^N}$, the empirical estimator is
\begin{equation}
    \hat\me_{\Hilbert|\Hilbert^{T+1}}^{\data_N}(x) := \sum_{i=1}^N w_i(x)\, \phi^{T+1}(\x_{0:T}^{(i)}),\quad\text{with}\quad w(x) := k_N(x)\T[K_N + N\lambda I_N]\inv, \label{eq:empirical_CME_Tstep}
\end{equation}
where $k_N(x):=[k(x,\x_{0}^{(i)})]_{i=1}^N$ and $K_N:=[k(\x_{0}^{(i)},\x_{0}^{(j)})]_{i,j=1}^N$.

\textbf{Robustness.}\;
Based on the $T$-step CME estimator in \eqref{eq:CME_Tstep}, we can construct an ambiguity set in the vector-valued RKHS $\mathcal{G}^{T+1}$ of functions $\Xdomain\to\smash{\Hilbert^{T+1}}$ centered at the empirical CME in \eqref{eq:empirical_CME_Tstep}:
\begin{equation}
        \amb_{N}^{T} := \big\lbrace \bar\mu\colon\Xdomain\times\borel{\Xdomain^{T}}\to[0,1] \,\big\vert\, \big\Vert \hat\me_{\Hilbert|\Hilbert^{T+1}}^{\data_{N}} - \me_{\Hilbert|\Hilbert^{T+1}}^{\bar\mu} \big\Vert_{\mathcal{G}^{T+1}}\leq\varepsilon
        \big\rbrace.\label{eq:amb_set_nonMarkov}
\end{equation}
However, since $\1_S$ does not generally lie in the RKHS of common kernels such as the Gaussian kernel, \eqref{eq:innerProdCond} cannot be applied directly; following~\cite{kanagawa2014recovering}, who extend Gaussian embeddings to functions with fractional smoothness $s>0$, i.e, living in the greater Besov space $B_{2,\infty}^s(\Xdomain)$, we arrive at a direct safety estimator valid beyond Markovian systems (Assumption~\ref{asm:markovianity}).

\begin{theorem}[Direct Approach]\label{thm:empirical_safety_nonMarkov}
    Define the safety functional $\smash{\rho^{S}}:=\smash{\min_{t\in\{0,\ldots,T\}} \1_{S}(\cdotx_t)}\in \smash{B_{2,\infty}^s(\Xdomain)}$ for some $s>0$.
    Consider the Gaussian RKHS $\smash{\Hilbert\equiv\Hilbert_\gamma}$ with lengthscale $\sigma_l = \gamma$ and let $\Hilbert_{\gamma_N}$ be the Gaussian RKHS of the kernel $k_{\gamma_N}$ with $\gamma_N:=N^{-\beta}\gamma$, for some $\beta>0$. Define the smoothed safety functional $\tilde\rho^S_N(\x_{0:T}) \!:=\! \int_{\Xdomain^{T+1}}\! K_{\gamma_N}(\x_{0:T}\!-\!\y)\rho^S(\y)d\y$ via convolution with
    \begin{equation*}
        K_{\gamma_N}(\x_{0:T}):=\sum_{j=1}^r\Big(\begin{smallmatrix}r\\j\end{smallmatrix}\Big)(-1)^{1-j}\frac{1}{j^d}\left(\frac{2}{\gamma_N^2\pi}\right)^{\frac{d(T+1)}{2}}k_{j\gamma_N/\sqrt{2}}(\x_{0:T}), \quad r:= \lfloor s\rfloor + 1.
    \end{equation*}
    Given data $\data_N$ and under Asm.~\ref{asm:ambiguity_set} with ambiguity set $\amb_N^{T}$ in \eqref{eq:amb_set_nonMarkov}, we have
    \begin{equation}
         P^S_{\lawseq}(x) 
        \geq \sum_{i\colon \x^{(i)}_t\in S,\, \forall t\in\{0,\ldots,T\}}  w_i(x) - \big(\varepsilon_1(x) + \varepsilon_2 + \varepsilon_3\big),\label{eq:safety_estimate}
    \end{equation}
    with probability at least $1-\conf$, with the error terms
    $\varepsilon_1(x) := \vert \smash{\sum_{i=1}^N w_i(x) (\rho^S(\x_{0:T}^{(i)})} -  \tilde\rho^S_N(\x_{0:T}^{(i)})) \vert$, $\varepsilon_2 := \varepsilon\kappa\Vert\tilde\rho^S_N\Vert_{\Hilbert_{\gamma_N}}(\gamma/\gamma_N)^{\frac{d(T+1)}{2}}$, and $\varepsilon_3 := \Vert\tilde\rho^S_N - \rho^S\Vert_{L_2(\lawseq)}$.
\end{theorem}
Theorem~\ref{thm:empirical_safety_nonMarkov} yields a non-recursive, direct learning-based lower bound on the safety probability that remains valid even for non-Markovian dynamics. Unlike barrier or abstraction-based approaches, it does not rely on iterative composition of empirical operators, thereby avoiding the cumulative conservatism intrinsic to DP.
The proof of Theorem~\ref{thm:empirical_safety_nonMarkov} is given in Appendix~\ref{app:proof_empirical_safety_nonMarkov}.

Here, the error terms $\varepsilon_1,\varepsilon_2,\varepsilon_3$ bound the empirical convolution error, the statistical error, and the approximation error inflicted by the convolution with \smash{$K_{\gamma_N}$}, respectively.
It is apparent that the discontinuity of the strict safety functional $\rho^S\colon\Xdomain^{T+1}\to\{0,1\}$ is leading to a more complex---and in fact slower---convergence than if $\rho^S$ was living in the RKHS $\Hilbert^{T+1}$ of the CME.
When $\rho^S\in\Hilbert^{T+1}$, the bound reduces to a single error term, with analogies to quantitative semantics~\citep{maler2004monitoring} (Appendix~\ref{app:proof_empirical_safety_nonMarkov}, Corollary~\ref{cor:quantitative}).

\begin{rem_}[Error Structure]\label{rem:error_structure}
In indirect methods, statistical and abstraction errors enter at every DP recursion step and compound over $T$ (see Appendix~\ref{app:additional_material_abstractions}).
By contrast, $\varepsilon_1$, $\varepsilon_2$, and $\varepsilon_3$ each arise from a single $T$-step estimation, not from $T$ recursive applications of an uncertain operator. While $\varepsilon_2$ grows with the trajectory-space dimension $d(T+1)$, this is a smoothing artifact controlled by the bandwidth ratio and $\beta$, not a consequence of compounding estimation error. In particular, $\varepsilon_1$ and $\varepsilon_3$ are horizon-independent, and the bound remains stable as $T$ grows for fixed bandwidth parameters.
\end{rem_}

\section{Experiments}\label{sec:experiments}
We support our theoretical analysis with empirical results comparing the direct safety estimator from Theorem~\ref{thm:empirical_safety_nonMarkov} with the indirect, DP-based estimator from Proposition~\ref{prop:empirical_value_update} ($\varepsilon=0)$.
Throughout, predictive quality is assessed via \emph{root mean square error} (RMSE$\downarrow$), excess RMSE ($\downarrow$, like RMSE but only counting overestimates), and the Brier score decomposition~\citep{murphy1973new} into \emph{reliability} (REL$\downarrow$) and \emph{resolution} (RES$\uparrow$); as a strictly proper scoring rule~\citep{gneiting2007strictly}, the Brier score is uniquely minimised by the true conditional probability $\smash{P^S_\lawseq}(x_0)$ in \eqref{eq:safety_probability}, making binary safety outcomes $\smash{\min_{t\in\{0,\ldots,T\}} \1_{S}(\x_{t})}\in\{0,1\}$ a principled evaluation target even in the absence of Monte Carlo (MC) ground truth $P^S_{\text{MC}}(x_0)\in[0,1]$ (see App.~\ref{app:real_data_quadrotor}).
Where MC is applicable, we report RMSE against $P^S_{\text{MC}}$; where it is not, we evaluate via Brier score against observed binary outcomes.
Throughout, we additionally report certified lower bounds $\hat{p}^{\mathrm{lb}}(x_0) \leq P^S_{\lawseq}(x_0)$ (Problem~\ref{prob:data_driven_safety_verification}) via the distribution-free histogram binning procedure of \citet{gupta2020distribution}, which gives lower confidence bounds on $P^S_{\lawseq}(x_0^{\mathrm{new}})$ for new test states from the evaluation distribution, without Markov or parametric assumptions; see Appendix~\ref{app:conformal} for details.

\subsection{Beyond Markovianity}
To demonstrate the effect of non-Markovianity and horizon length on the two approaches, we consider a synthetic nonlinear system that allows tuning the degree of non-Markovianity.
To this end, we consider the nonlinear discrete-time stochastic process
\[
    \begin{bmatrix}x_{1,t+1}\\x_{2,t+1}\end{bmatrix}
    = \begin{bmatrix}x_{1,t}\\x_{2,t}\end{bmatrix}
      + h\begin{bmatrix}x_{2,t}\\\tfrac{1}{3}x_{1,t}^3-x_{1,t}-x_{2,t}\end{bmatrix} + z_t,\quad
    z_{t+1} = \alpha\big(z_t+\beta\tanh(\gamma x_{1,t})\big) + w_t,
\]
with time discretization $h=0.1$, coupling parameters $\beta_c=0.12$, $\gamma_c=1.0$,
i.i.d.\ Gaussian noise $w_t\sim\mathcal{N}(\cdotx\mid0,\sigma^2{(1 - \alpha^2)}I_2)$ with variance $\sigma^2 = 0.15^2$, and latent initial distribution $z_0\sim\mathcal{N}(\cdotx\mid0,\sigma^2I_2)$.
The safe set is $S := \{x\in[-3,2.5]\times[-2,1]\} \setminus \mathcal{O}$, where $\mathcal{O}$ comprises three small rectangular obstacles (Table~\ref{tab:ar1_obstacles}, App.~\ref{app:additional_case_study_details}).
The parameter \(\alpha\in[0,1)\) interpolates between Markovian (\(\alpha=0\)) and strongly non-Markovian dynamics (\(\alpha\!\to\!1\)).

\textbf{Datasets.}\;
For varying time horizons $T\in\{5,10,15\}$, we draw $\hat{N}=1000T$ i.i.d.\ sample sequences $\smash{\{\x^{(i)}_{0:T}\}_{i=1}^{\hat{N}}}$ and compare the two different approaches:
\begin{itemize}
    \item \textbf{Empirical DP (Proposition~\ref{prop:empirical_value_update}):} we extract $\hat{N}$ i.i.d.\ one-step sample transitions $\data_{\hat{N}}:=\smash{\{(\x^{(i)}_0,\x^{(i)}_1)\}_{i=1}^{\hat{N}}}$ to train
    the one-step CME in \eqref{eq:CME}.
    \item \textbf{Direct (Theorem~\ref{thm:empirical_safety_nonMarkov}):} we take a subset of $N=\smash{\hat{N}}/T=1000$ i.i.d.\ sample sequences $\data_N:=\smash{\{\x^{(i)}_{0:T}\}_{i=1}^N}$ to train the $T$-step CME in \eqref{eq:CME_Tstep}.
\end{itemize}
Note that selecting $\hat{N}=NT$ gives the DP-based approach the most favorable data basis.

\textbf{Results.}\;
We evaluate both methods on a \(40\times40\) grid against a Monte-Carlo baseline (\(N_{\mathrm{MC}}=1000\) per grid point); Figure~\ref{fig:case_study_nonmarkov_effect} summarises the results across \(T\) and \(\alpha\).
The indirect estimator accumulates substantial error with growing $\alpha$ and $T$ whilst the direct estimator remains essentially stable; crucially, this error can be almost exclusively attributed to \emph{safety overestimates} (Figure~\ref{fig:case_study_nonmarkov_effect}b): high certified safety assigned to states that are in fact unsafe.
The reliability score (Figure~\ref{fig:case_study_nonmarkov_effect}c) corroborates this: the direct estimator's reliability remains near zero across all $\alpha$ and $T$, confirming consistent calibration, whilst the indirect estimator's reliability deteriorates with both---a signature of systematic miscalibration induced by the violated Markov assumption.
We provide additional details and results in
Appendix~\ref{app:additional_case_study_details}.

\begin{figure}
    \centering
    \begin{subfigure}{.32\linewidth}
        \centering
        \includegraphics[width=\linewidth]{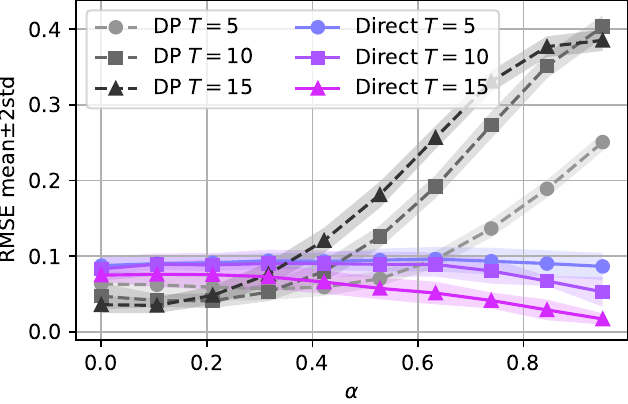}
        \caption{RMSE (lower is better)}
    \end{subfigure}
    \begin{subfigure}{.32\linewidth}
        \centering
        \includegraphics[width=\linewidth]{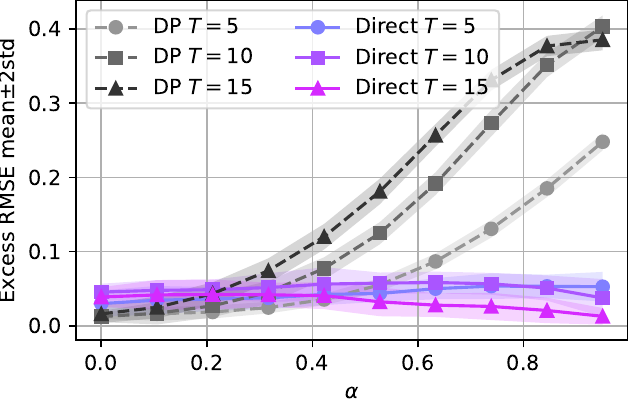}
        \caption{Excess RMSE (lower is better)}
    \end{subfigure}
    \begin{subfigure}{.32\linewidth}
        \centering
        \includegraphics[width=\linewidth]{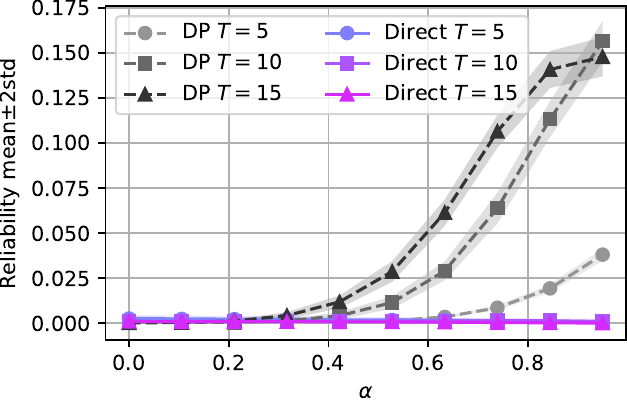}
        \caption{Reliability (lower is better)}
    \end{subfigure}
    \caption{The effect of increasingly non-Markovian dynamics ($\alpha\to1$) on the DP and direct methods.}
    \label{fig:case_study_nonmarkov_effect}
\end{figure}

\subsection{Synthetic Neural-Controlled Quadrotor}\label{sec:synthetic_quadrotor}

Next, we move to a more realistic system. We consider a simulated quadrotor system equipped with a neural adaptive controller \citep{zhang2025learning}\footnote{Code available at \url{https://github.com/muellerlab/xadapt\_ctrl}} which we perturb with i.i.d.\ Gaussian velocity disturbances ($\sigma{=}0.035\,\text{m/s}$).
The controller's internal adaptive state renders the closed-loop dynamics non-Markovian.
Starting at random initial positions $(p_x, p_y) \in [-2.5, 2.5]^2\,\text{m}$ and an initial altitude $p_z{=}0.8\,\text{m}$, the system must maintain $p_z > 0.65\,\text{m}$ and $|p_{x,y}| < 4\,\text{m}$ over a $T{=}2500$-step ($5\,\text{s}$) horizon.
We train on $N{=}1000$ trajectories (plus $n_{\mathrm{cal}}{=}1000$ held out for calibration) and evaluate on a $10{\times}10$ MC grid with $200$ rollouts per point.
Both methods use $\smash{x_0}{=}\smash{(p_x,p_y)}$ initial position as features.
The DP baseline subsamples $\smash{\hat{N}}{=}20{,}000$ one-step pairs from the $N{\cdot}T{=}2.5\,\text{M}$ available transitions.

\textbf{Results.}\;
Figure~\ref{fig:quadrotor_sim_heatmap} reveals that the indirect estimator fails to produce non-zero safety probabilities, while the direct estimator produces spatially varying, non-trivial estimates.
Table~\ref{tab:benchmark_metrics} confirms this quantitatively: the indirect estimator's RMSE${=}0.798$ equals that of the constant-zero predictor, with zero discriminative power (RES${=}0.000$).
In contrast, the direct estimator achieves a $13{\times}$ improvement in RMSE and a $300{\times}$ improvement in reliability (REL).
The collapse is explained by a spectral analysis of the DP backward operator: its spectral radius $\varrho{\approx}0.99976$ gives $\varrho^{2500}{\approx}0.55$, so 2500 backward recursions attenuate the estimate until all spatial structure is lost (RES${=}0.000$).

\textbf{Conformalization cannot compensate for a degenerate predictor:}\;
Figure~\ref{fig:quadrotor_sim_error} shows the error $\hat{p}^{\mathrm{lb}} - P^S_{\mathrm{MC}}$ in the certified lower bounds $\hat{p}^{\mathrm{lb}}$ ($1-\conf{=}90\%$ confidence, App.~\ref{app:conformal}): \textcolor{red!70!black}{red} means the certificate \emph{exceeds} the true safety probability (unsound); \textcolor{blue!70!black}{blue} means it is a sound lower bound.
The direct estimator's certificate is sound at $93\%$ of grid points and spatially discriminative (std $0.085$).
The indirect estimator's certificate appears tighter in terms of $\varepsilon$ ($0.034$ vs.\ $0.152$), yet is unsound at $44\%$ of grid points.
The structural reason: all indirect predictions collapse to a single numerical value, so the conformal procedure pools the entire calibration set into one bin and returns the calibration safety rate (${\approx}0.825$) minus a Hoeffding correction ($\varepsilon{\approx}0.034$), giving $\hat{p}^{\mathrm{lb}}{\approx}0.791$ everywhere.
This constant exceeds the true safety probability in the dangerous regions (red corner regions)---precisely where a sound certificate is needed most.

\begin{figure}[t]
    \centering
    \begin{subfigure}{.59\linewidth}
        \centering
        \includegraphics[width=1\linewidth]{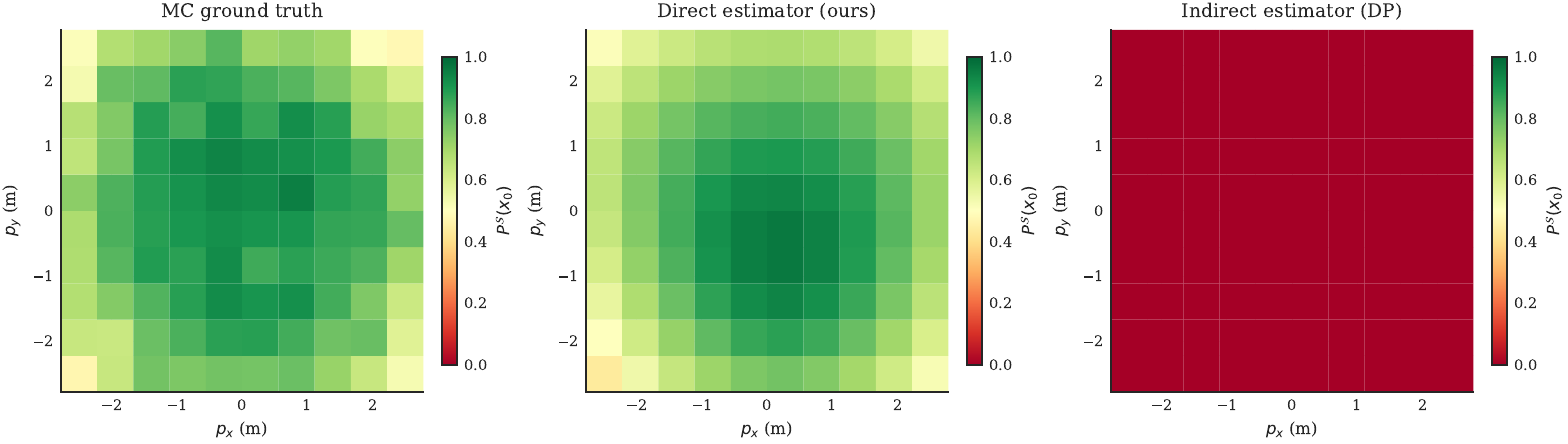}
        \caption{Safety probability (estimates) $P^S(x_0)$}
        \label{fig:quadrotor_sim_heatmap}
    \end{subfigure}
    \begin{subfigure}{.4\linewidth}
        \centering
        \includegraphics[width=1\linewidth]{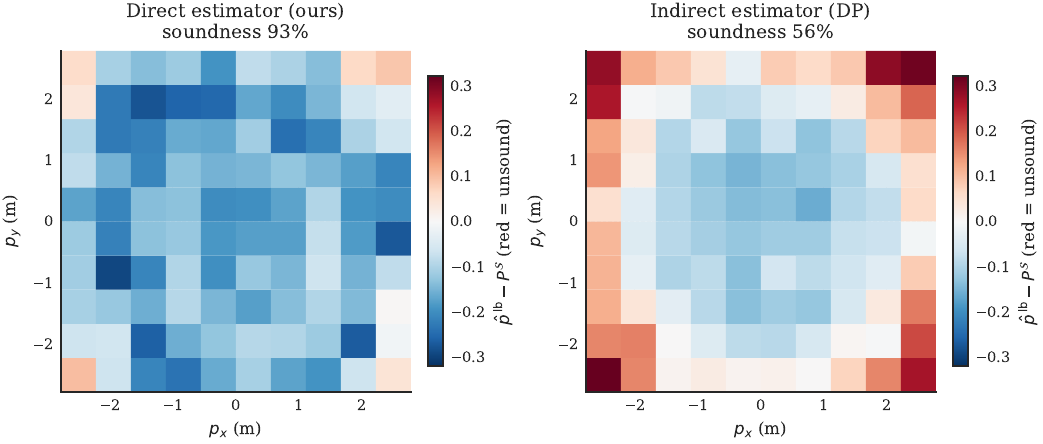}
        \caption{Error in conformalized estimates}
        \label{fig:quadrotor_sim_error}
    \end{subfigure}
    \caption{Synthetic neural-controlled quadrotor (Sec.~\ref{sec:synthetic_quadrotor}, $T{=}2500$).
    \textbf{(a)} Point estimates of $P^S(x_0)$: the indirect estimator collapses to zero everywhere (zero level set in red); the direct estimator recovers the spatial safety structure.
    \textbf{(b)} Certified lower bound error $\hat{p}^{\mathrm{lb}} - P^S_{\mathrm{MC}}$ ($\alpha{=}0.1$, App.~\ref{app:conformal}): \textcolor{red!70!black}{red} $=$ unsound ($\hat{p}^{\mathrm{lb}} > P^S_{\mathrm{MC}}$), \textcolor{blue!70!black}{blue} $=$ sound.
    The indirect certificate is a constant ${\approx}0.791$ (calibration safety rate ${\approx}0.825$ minus Hoeffding correction $\varepsilon{\approx}0.034$), unsound wherever the true safety probability falls below $0.791$.
    The direct certificate is spatially adaptive and sound at $93\%$ of grid points.
    }
\end{figure}

\begin{table}[t]
\centering
\caption{Quantitative results for Section~\ref{sec:synthetic_quadrotor} ($N{=}1000$ train).
\textbf{Top:} point estimates vs.\ MC ground truth; Brier decomposition into REL$\downarrow$ (calibration error) and RES$\uparrow$ (discriminativeness); $\varrho^T$ spectral decay of DP operator.
\textbf{Bottom:} certified lower bounds ($1-\conf{=}90\%$ confidence, App.~\ref{app:conformal}); Soundness = fraction of grid where $\hat{p}^{\mathrm{lb}} \leq P^S_{\mathrm{MC}}$; Disc.\ = std of $\hat{p}^{\mathrm{lb}}$ across grid (spatial variability).}
\label{tab:benchmark_metrics}
\setlength{\tabcolsep}{5pt}
\begin{tabular}{llccccc}
    \toprule
    Method & RMSE$\downarrow$ & REL$\downarrow$ & RES$\uparrow$ & $\varrho^T$ \\
    \midrule
    \textbf{Direct (ours)} & $\color{green!70!black}\mathbf{0.059}$ & $\color{green!70!black}\mathbf{0.002}$ & $\color{green!70!black}\mathbf{0.011}$ & $-$ \\
    Indirect               & $0.798$ & $0.623$ & $0.000$ & $5.5{\times}10^{-1}$ \\
    \midrule
    Method & Soundness$\uparrow$ & Disc.$\uparrow$ & $\varepsilon$ & \\
    \midrule
    \textbf{Direct + UQ (ours)} & $\color{green!70!black}\mathbf{93\%}$ & $\color{green!70!black}\mathbf{0.085}$ & $0.152$ & \\
    Indirect + UQ               & $56\%$ & $0.000$ & $0.034$ & \\
    \bottomrule
\end{tabular}
\end{table}

\section{Conclusion}
\label{sec:related_work_and_conclusion}
Our analysis shows that, in the learning setting, the DP paradigm underlying barriers and abstractions is no longer an advantage: it amplifies statistical and discretization error, and breaks down in the non-Markovian case.
Direct multi-step predictors avoid this recursion, operate naturally on sequenced data, and extend to non-Markovian systems.

We highlight two further directions in which the unified embedding view extends existing ideas.
Theorem~\ref{thm:empirical_safety_nonMarkov} shows that strict safety indicators introduce smoothing effects and slow convergence.
This connects directly to \textbf{quantitative semantics}~\citep{maler2004monitoring}, where smooth safety functionals $\tilde\rho^S$ quantify the margin of satisfaction, in which case the bound reduces to a single error term (see Corollary~\ref{cor:quantitative} in App.~\ref{app:quantitative_semantics}).
Notably, general quantitative semantics do not admit an equivalent DP recursion, highlighting that the direct inference approach is inherently more expressive.
Also, whilst discrete- and continuous-time techniques are usually treated separately, choosing a kernel on path space allows our framework to \textbf{extend directly to continuous-time systems}~\citep{chevyrev2022signature}, suggesting a potential route to subsuming Hamilton--Jacobi reachability within the same unified embedding view.
Kernel methods scale as $O(N^3)$, motivating the natural extension to calibrated neural classifiers, out-of-distribution conformal techniques, and deep learning-based verification pipelines~\citep[see also][]{thorpe2022socks,sturt2023nonparametric}.

\textbf{Impact.}\;
Although this work focuses on safety verification, we believe that its implications suggest a structural turning point in data-driven verification and control synthesis in general, yielding methods that combine statistical learning with formal reasoning without relying on Markovian structure.
We expect these insights to advance the individual methods discussed.

\bibliographystyle{plainnat}
\bibliography{references}


\appendix

\appendix

\section{Additional Case Study Details and Experimental Results}\label{app:additional_case_study_details}
This appendix provides more comprehensive numerical results and additional details on the case study in Section~\ref{sec:experiments}.

\paragraph{Safe set for the synthetic benchmark.}
The safe set is $S = [-3,2.5]\times[-2,1]\setminus\mathcal{O}$, where $\mathcal{O}$ is the union of three rectangular obstacles listed in Table~\ref{tab:ar1_obstacles}.

\begin{table}[h]
\centering
\caption{Obstacle coordinates for the non-Markovian AR(1) benchmark.}
\label{tab:ar1_obstacles}
\begin{tabular}{ccc}
\toprule
Obstacle & $x_1$ range & $x_2$ range \\
\midrule
1 & $[0.4,\ 0.6]$ & $[0.2,\ 0.6]$ \\
2 & $[0.6,\ 0.7]$ & $[0.2,\ 0.4]$ \\
3 & $[-1.5,\ -0.5]$ & $[-1.5,\ -1.0]$ \\
\bottomrule
\end{tabular}
\end{table}

\paragraph{Hyperparameters.}
For all presented experiments, we use the hyperparameters reported in Table~\ref{tab:HPs}, which have been obtained through optimization using \textsc{Optuna}~\citep{akiba2019optuna}.
The robustness parameter $\varepsilon$ was set to zero for all experiments to allow for comparability.  
Apart from the time horizon $T\in\{5,10,15\}$, we differentiate between two types of data used to train the one-step CME in~\eqref{eq:CME} for the DP method:
\begin{itemize}
    \item \textbf{I.i.d.\ Data:} as used in Section~\ref{sec:experiments}, where we extract $\hat{N}$ uniformly sampled one-step transitions; or
    \item \textbf{Non-I.i.d.\ Data:} obtained by breaking up the $N=1000$ full sequences $\data_N$ into $\hat{N}=NT$ dependent one-step transitions.
\end{itemize}
We note that the DP-based estimator is substantially more sensitive to hyperparameter changes than the direct approach.  
In addition, evaluating the direct estimator is orders of magnitude faster than the iterative DP scheme, with the gap growing linearly in the horizon~$T$.  
Because CME evaluation scales cubically in the number of samples, and the direct approach requires only a single CME evaluation --- compared to $T$ evaluations for the DP approach --- this difference becomes significant.  
For completeness, we also provide results for a matched sample amount $\hat{N}=N=1000$.

\medskip
For all reported results we repeat the experiments ---including the expensive MC evaluation --- $10$ times and report the mean and double standard deviation.  
An overview of all additional experimental results can be found in Table~\ref{tab:additional_results}.  
Apart from the \emph{root mean square error} (RMSE), we report the following additional metrics and the Brier decomposition:
\begin{itemize}
    \item \textbf{Excess RMSE:} RMSE computed only on grid points where the predicted safety probability exceeds the MC baseline.
    \item \textbf{Reliability:} deviation between predicted probabilities and empirical event frequencies across forecast bins (calibration error).
    \item \textbf{Resolution:} spread of empirical event frequencies across bins, quantifying how much the predictions separate safe from unsafe regions. We normalize the resolution by the uncertainty score to obtain a value between zero and one.
    \item \textbf{Uncertainty:} intrinsic difficulty of the prediction task determined by the marginal event frequency.
    \item \textbf{Brier score:} the \emph{mean squared error} (MSE), included for completeness.
\end{itemize}

Across all settings, the direct approach yields consistently stable reliability, with RMSE and excess RMSE remaining essentially unchanged for varying degrees of non-Markovianity ($\alpha\in[0,1)$).  
By contrast, the DP estimator shows strong sensitivity to hyperparameters and frequently over-predicts the safety probability, particularly for increasing $\alpha$ and larger horizons~$T$, where the system becomes largely unsafe.  
For the matched-sample case $\hat{N}=N=1000$, the initial DP advantage vanishes entirely.

\textbf{Hardware.}\;
All experiments were run on a MacBook Pro with an Apple M3 chip and 32\,GB of RAM.
Expected runtimes for each experiment configuration are provided in the repository.
The full sweep in Section~\ref{sec:experiments} (all $\alpha$, $T$, and seed combinations) takes on the order of several hours on this hardware.

\textbf{Computational Cost of the DP Baseline.}\;
The standard CME-DP baseline extracts all $N{\cdot}T$ one-step transition pairs from the training trajectories and builds an $\hat{N}{\times}\hat{N}$ Gram matrix over source states, incurring $O(\hat{N}^3 + \hat{N}^2 T)$ cost.
For $N{=}1000$ and $T{=}2500$ the full $\hat{N}{=}N{\cdot}T{=}2.5\,\text{M}$ matrix is infeasible (${\approx}50\,\text{TB}$); we therefore randomly subsample $\hat{N}{=}20{,}000$ pairs ($\hat{N}^2{\times}8\,\text{bytes}\approx3.2\,\text{GB}$ Gram matrix) and note that this already gives DP $20{\times}$ more data points than the direct method uses ($N{=}1000$ trajectories).
Even with this advantage, the spectral contraction over $T{=}2500$ steps causes the DP to collapse to near-zero (Section~\ref{sec:synthetic_quadrotor}).
By contrast, the direct estimator constructs a single $N{\times}N{=}1{,}000{\times}1{,}000$ Gram matrix and requires no backward recursion, giving an $O(N^3)$ total cost and a $(\hat{N}/N)^3{=}8{,}000{\times}$ reduction in Gram matrix cost relative to the subsampled DP.

\begin{table}[h]
    \centering
    \aboverulesep=0ex 
    \belowrulesep=0ex 
    \caption{Additional experimental results}
    \begin{tabular}{l|cc}
        \toprule
        \rule{0pt}{1.1EM}
         & $\hat N=NT$ & $\hat N=N$ \\
        \midrule
        \rule{0pt}{1.1EM}
        Uniform DP data & Figs.~\ref{fig:results_by_T_diffN_indeptrain}--\ref{fig:safetyprob_a0_00_diffN_indeptrain}--\ref{fig:safetyprob_a0_95_diffN_indeptrain} & Figs.~\ref{fig:results_by_T_sameN_indeptrain}--\ref{fig:safetyprob_a0_00_sameN_indeptrain}--\ref{fig:safetyprob_a0_95_sameN_indeptrain} \\
        Dependent DP data & Figs.~\ref{fig:results_by_T_diffN_deptrain}--\ref{fig:safetyprob_a0_00_diffN_deptrain}--\ref{fig:safetyprob_a0_95_diffN_deptrain} & Figs.~\ref{fig:results_by_T_sameN_deptrain}--\ref{fig:safetyprob_a0_00_sameN_deptrain}--\ref{fig:safetyprob_a0_95_sameN_deptrain} \\
        \bottomrule
    \end{tabular}
    \label{tab:additional_results}
\end{table}
\begin{table*}
    \centering
    \caption{Hyperparameters used in the experiments.}
    \resizebox{\textwidth}{!}{\begin{tabular}{lcccc}
        \toprule
        Configuration & \multicolumn{2}{c}{Direct Approach} & \multicolumn{2}{c}{DP Approach} \\
         & $\sigma_l^2$ & $\lambda$ & $\sigma_l^2$ & $\lambda$ \\
        \midrule
        i.i.d.\ DP, $T=5$ & $(0.772, 1.572)$ & $3.004\cdot 10^{-8}$ & $(0.596, 0.361)$ & $1.456\cdot 10^{-6}$ \\
        i.i.d.\ DP, $T=10$ & $(0.986, 0.914)$ & $4.615\cdot 10^{-8}$ & $(0.556, 0.652)$ & $2.038\cdot 10^{-6}$ \\
        i.i.d.\ DP, $T=15$ & $(1.282, 1.416)$ & $2.791\cdot 10^{-7}$ & $(0.472, 0.290)$ & $2.294\cdot 10^{-7}$ \\
        non-i.i.d.\ DP, $T=5$ & $(0.917, 1.187)$ & $1.645\cdot 10^{-8}$ & $(0.477, 0.444)$ & $9.892\cdot 10^{-6}$ \\
        non-i.i.d.\ DP, $T=10$ & $(1.189, 0.981)$ & $1.749\cdot 10^{-7}$ & $(0.408, 0.359)$ & $5.239\cdot 10^{-7}$ \\
        non-i.i.d.\ DP, $T=15$ & $(0.599, 0.401)$ & $0.001$ & $(0.638, 0.784)$ & $5.162\cdot 10^{-7}$ \\
        \midrule
        Synth.\ quadrotor (Sec.~\ref{sec:synthetic_quadrotor}) & $\sigma_l{=}2.0$ & $0.01$ & $\sigma_l{=}0.3$, $\hat{N}{=}20{,}000$ & $10^{-4}$ \\
        Real-data (App.~\ref{app:real_data_quadrotor}) & \multicolumn{4}{c}{5-fold CV: $\sigma_l\in\{2,4,8,12,20\}$, $\lambda\in\{10^{-4},\ldots,10^{-1}\}$, per pilot} \\
        \bottomrule
    \end{tabular}}
    \label{tab:HPs}
\end{table*}


\begin{figure*}[!ht]
    \centering
    \begin{subfigure}[t]{.32\linewidth}
        \includegraphics[width=\linewidth]{figures/newresults_by_T_5_10_15_diffN_indeptrain_rmse.pdf}
        \caption{RMSE (lower is better)}
    \end{subfigure}
    \begin{subfigure}[t]{.32\linewidth}
        \includegraphics[width=\linewidth]{figures/newresults_by_T_5_10_15_diffN_indeptrain_excess.pdf}
        \caption{Excess RMSE (lower is better)}
    \end{subfigure}
    \begin{subfigure}[t]{.32\linewidth}
        \includegraphics[width=\linewidth]{figures/newresults_by_T_5_10_15_diffN_indeptrain_rel.pdf}
        \caption{Reliability (lower is better)}
    \end{subfigure}
    \\[.6em]
    \begin{subfigure}[t]{.32\linewidth}
        \includegraphics[width=\linewidth]{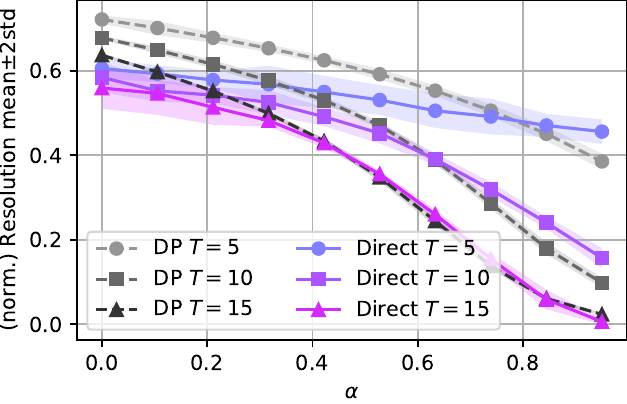}
        \caption{Resolution (normalized) (1 best/0 worst)}
    \end{subfigure}
    \begin{subfigure}[t]{.32\linewidth}
        \includegraphics[width=\linewidth]{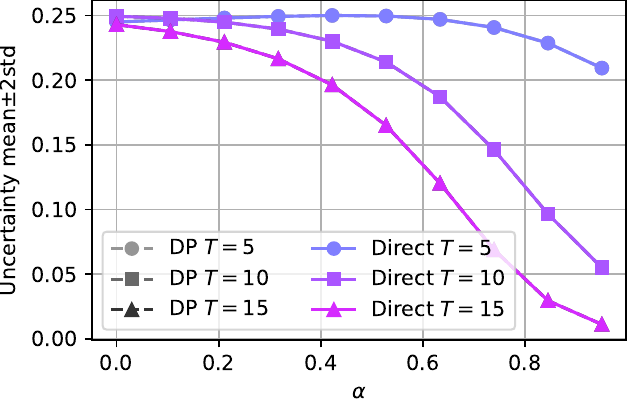}
        \caption{Uncertainty (DP behind Direct)}
    \end{subfigure}
    \begin{subfigure}[t]{.32\linewidth}
        \includegraphics[width=\linewidth]{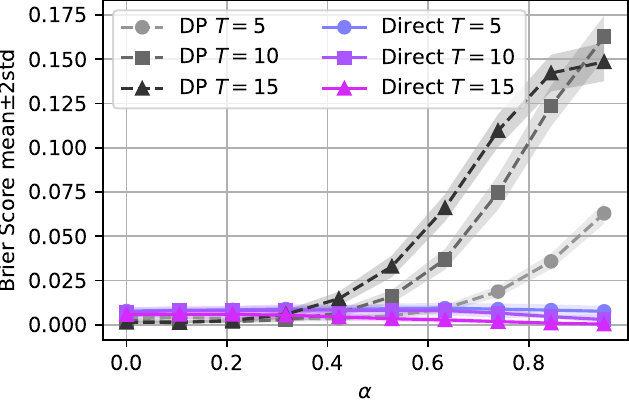}
        \caption{Brier (lower is better)}
    \end{subfigure}
    \caption{For uniform DP data and $\hat N=NT$: Results comparing DP and the direct approach for various $\alpha\in[0,1)$ and $T=5,10,15$ in terms of different metrics.}
    \label{fig:results_by_T_diffN_indeptrain}
\end{figure*}

\begin{figure*}[!ht]
    \centering
    \begin{subfigure}[t]{.32\linewidth}
        \includegraphics[width=\linewidth]{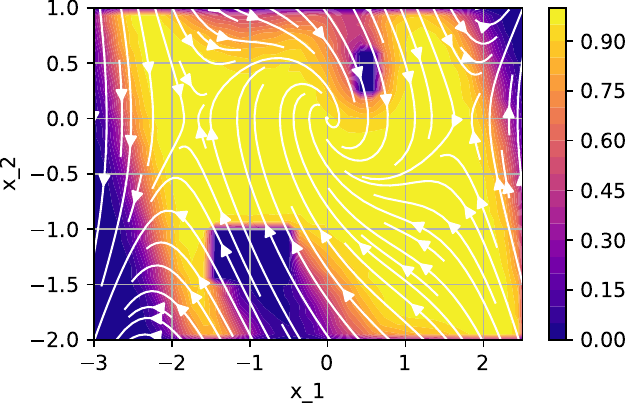}
        \caption{MC Ground Truth ($T=5$)}
    \end{subfigure}
    \begin{subfigure}[t]{.32\linewidth}
        \includegraphics[width=\linewidth]{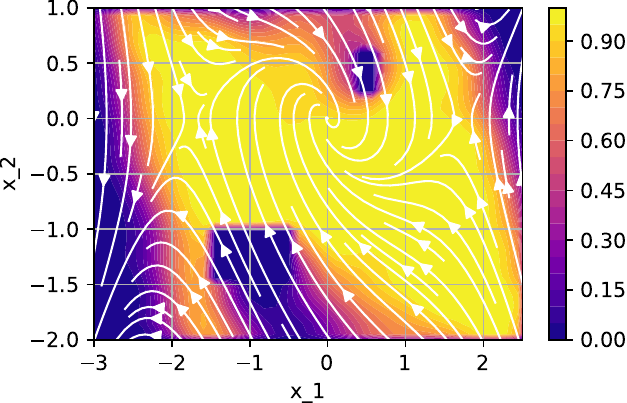}
        \caption{DP ($T=5$)}
    \end{subfigure}
    \begin{subfigure}[t]{.32\linewidth}
        \includegraphics[width=\linewidth]{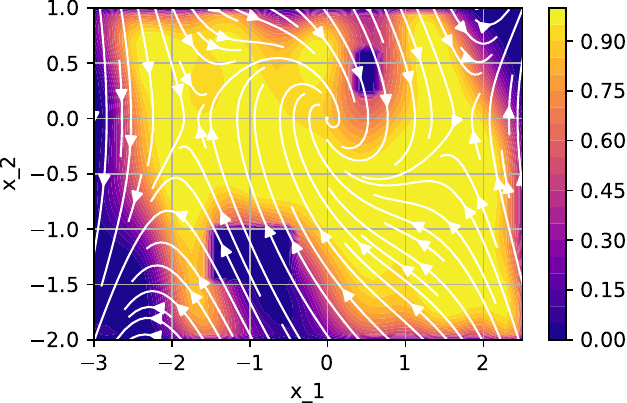}
        \caption{Direct ($T=5$)}
    \end{subfigure}
    \\[.6em]
    \begin{subfigure}[t]{.32\linewidth}
        \includegraphics[width=\linewidth]{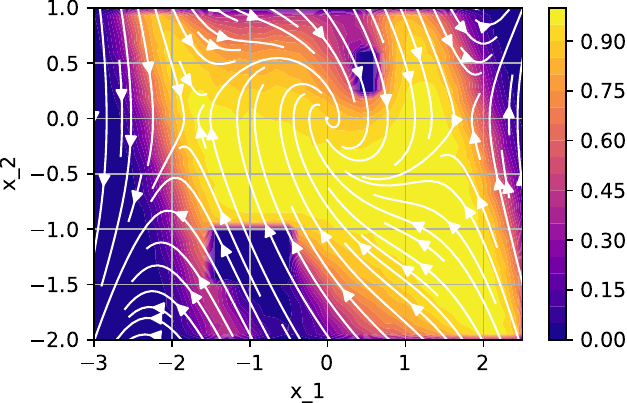}
        \caption{MC Ground Truth ($T=10$)}
    \end{subfigure}
    \begin{subfigure}[t]{.32\linewidth}
        \includegraphics[width=\linewidth]{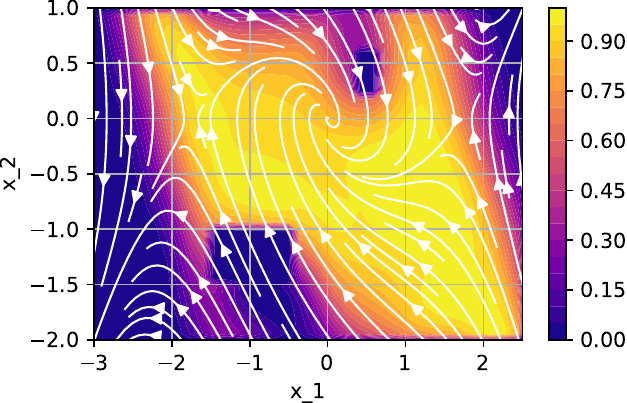}
        \caption{DP ($T=10$)}
    \end{subfigure}
    \begin{subfigure}[t]{.32\linewidth}
        \includegraphics[width=\linewidth]{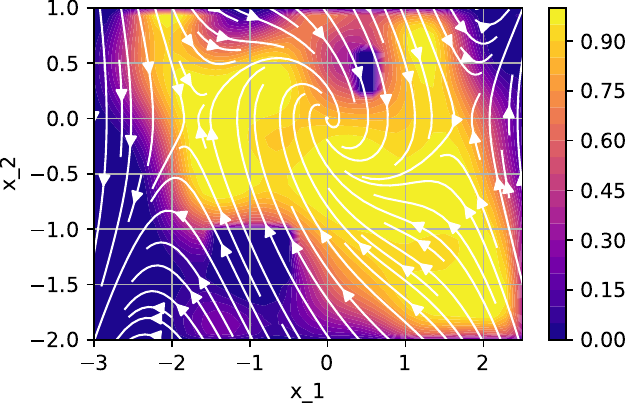}
        \caption{Direct ($T=10$)}
    \end{subfigure}
    \\[.6em]
    \begin{subfigure}[t]{.32\linewidth}
        \includegraphics[width=\linewidth]{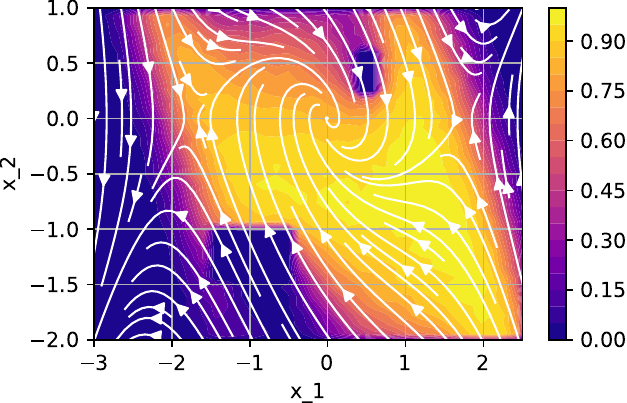}
        \caption{MC Ground Truth ($T=15$)}
    \end{subfigure}
    \begin{subfigure}[t]{.32\linewidth}
        \includegraphics[width=\linewidth]{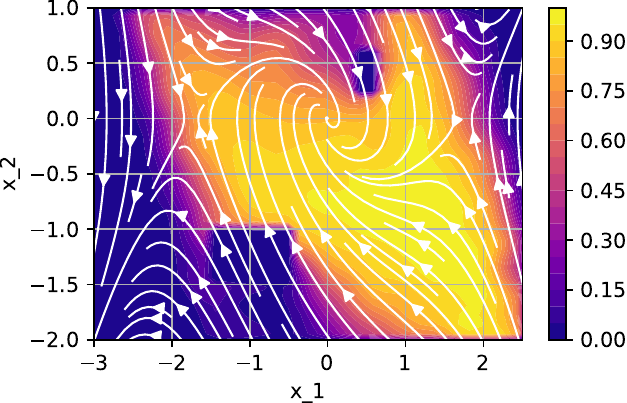}
        \caption{DP ($T=15$)}
    \end{subfigure}
    \begin{subfigure}[t]{.32\linewidth}
        \includegraphics[width=\linewidth]{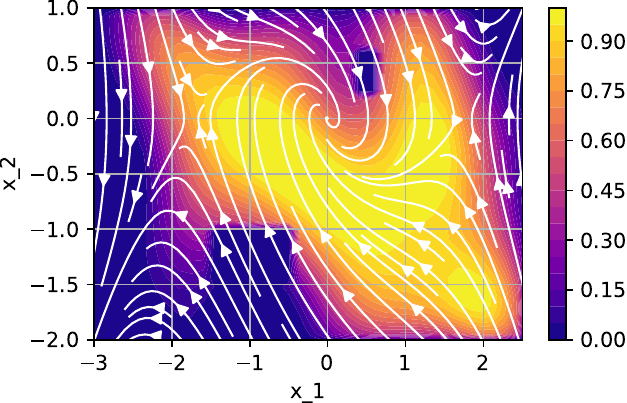}
        \caption{Direct ($T=15$)}
    \end{subfigure}
    \caption{For uniform DP data and $\hat N=NT$: Safety probability estimates (in color) and unknown true dynamics (mean) vector field (white arrows; all identical) for the fully Markovian dynamics ($\alpha=0$).}
    \label{fig:safetyprob_a0_00_diffN_indeptrain}
\end{figure*}

\begin{figure*}[!ht]
    \centering
    \begin{subfigure}[t]{.32\linewidth}
        \includegraphics[width=\linewidth]{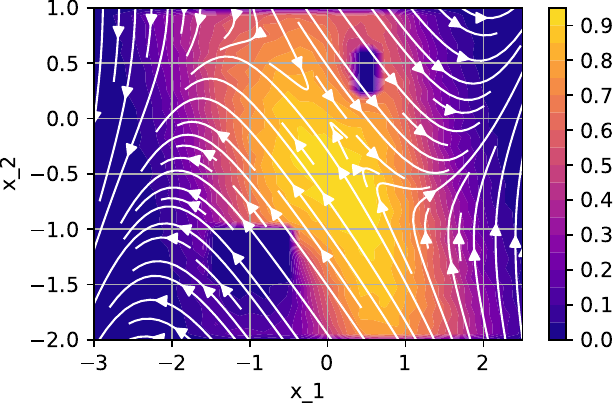}
        \caption{MC Ground Truth ($T=5$)}
    \end{subfigure}
    \begin{subfigure}[t]{.32\linewidth}
        \includegraphics[width=\linewidth]{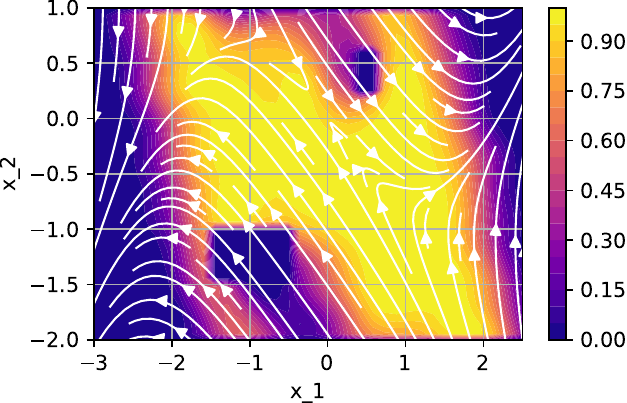}
        \caption{DP ($T=5$)}
    \end{subfigure}
    \begin{subfigure}[t]{.32\linewidth}
        \includegraphics[width=\linewidth]{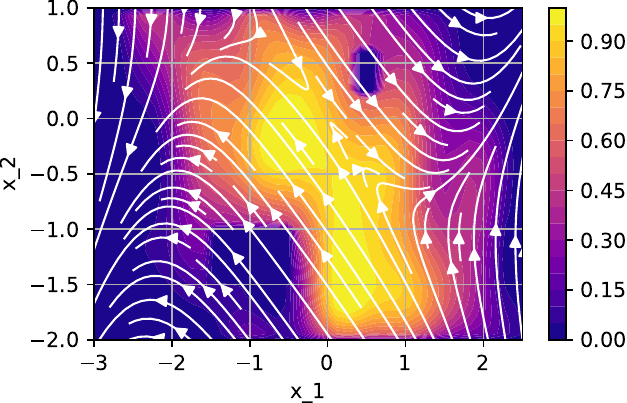}
        \caption{Direct ($T=5$)}
    \end{subfigure}
    \\[.6em]
    \begin{subfigure}[t]{.32\linewidth}
        \includegraphics[width=\linewidth]{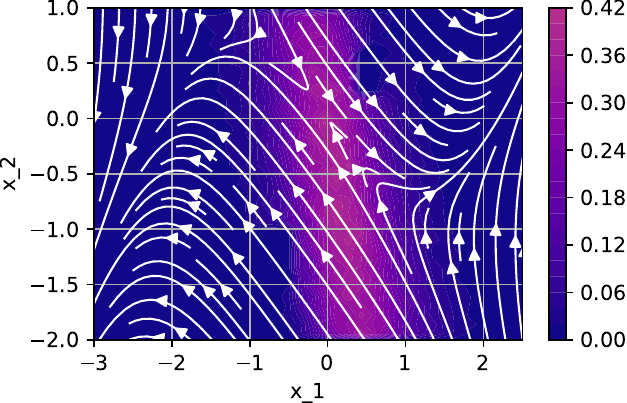}
        \caption{MC Ground Truth ($T=10$)}
    \end{subfigure}
    \begin{subfigure}[t]{.32\linewidth}
        \includegraphics[width=\linewidth]{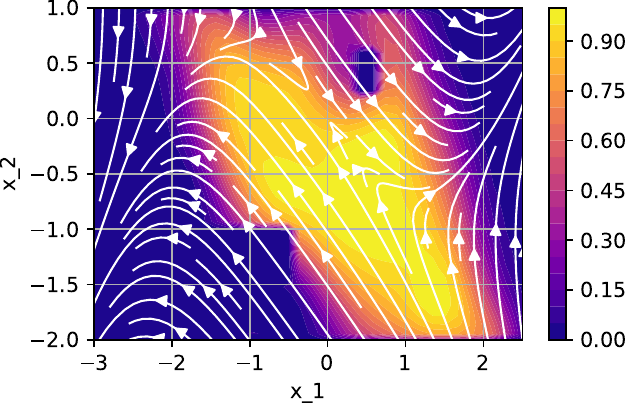}
        \caption{DP ($T=10$)}
    \end{subfigure}
    \begin{subfigure}[t]{.32\linewidth}
        \includegraphics[width=\linewidth]{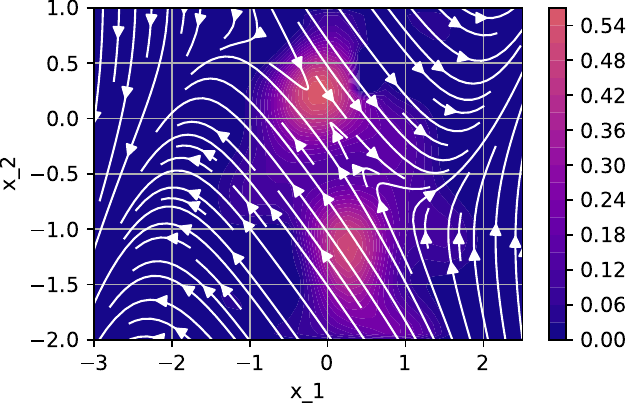}
        \caption{Direct ($T=10$)}
    \end{subfigure}
    \\[.6em]
    \begin{subfigure}[t]{.32\linewidth}
        \includegraphics[width=\linewidth]{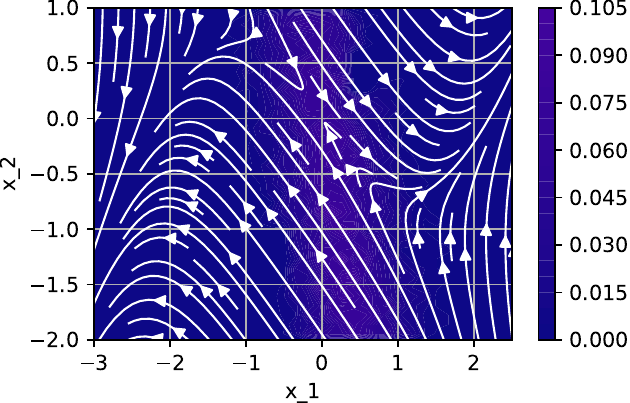}
        \caption{MC Ground Truth ($T=15$)}
    \end{subfigure}
    \begin{subfigure}[t]{.32\linewidth}
        \includegraphics[width=\linewidth]{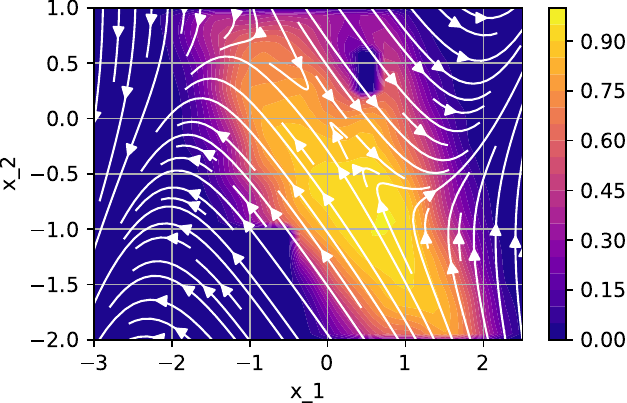}
        \caption{DP ($T=15$)}
    \end{subfigure}
    \begin{subfigure}[t]{.32\linewidth}
        \includegraphics[width=\linewidth]{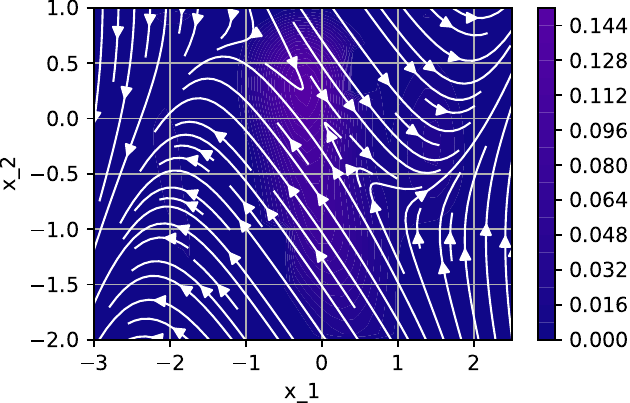}
        \caption{Direct ($T=15$)}
    \end{subfigure}
    \caption{For uniform DP data and $\hat N=NT$: Safety probability estimates (in color) and unknown true dynamics (mean) vector field (white arrows; all identical) for the fully Markovian dynamics ({\color{red}$\alpha=0.95$}).}
    \label{fig:safetyprob_a0_95_diffN_indeptrain}
\end{figure*}

\begin{figure*}[!ht]
    \centering
    \begin{subfigure}[t]{.32\linewidth}
        \includegraphics[width=\linewidth]{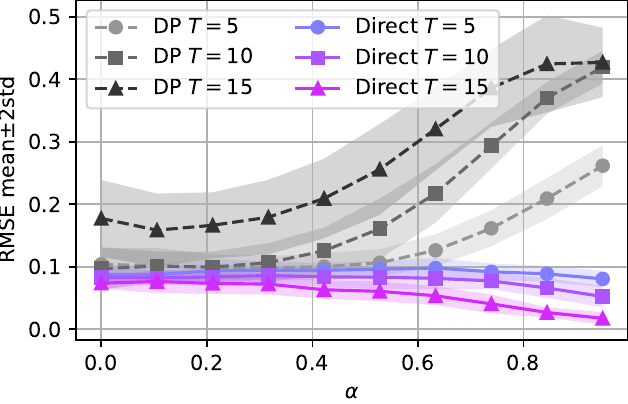}
        \caption{RMSE (lower is better)}
    \end{subfigure}
    \begin{subfigure}[t]{.32\linewidth}
        \includegraphics[width=\linewidth]{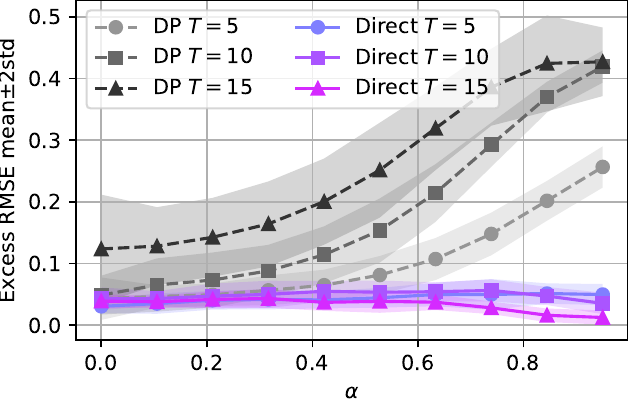}
        \caption{Excess RMSE (lower is better)}
    \end{subfigure}
    \begin{subfigure}[t]{.32\linewidth}
        \includegraphics[width=\linewidth]{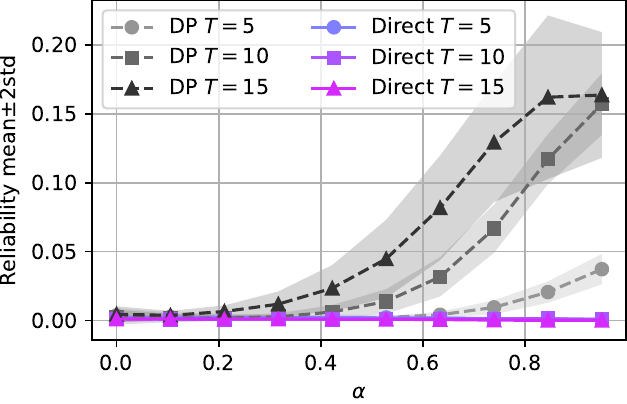}
        \caption{Reliability (lower is better)}
    \end{subfigure}
    \\[.6em]
    \begin{subfigure}[t]{.32\linewidth}
        \includegraphics[width=\linewidth]{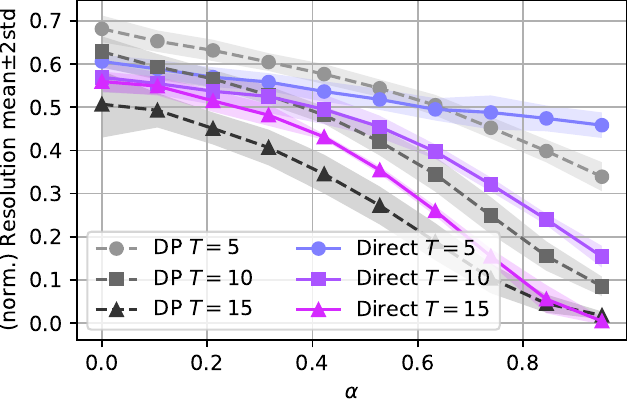}
        \caption{Resolution (normalized) (1 best/0 worst)}
    \end{subfigure}
    \begin{subfigure}[t]{.32\linewidth}
        \includegraphics[width=\linewidth]{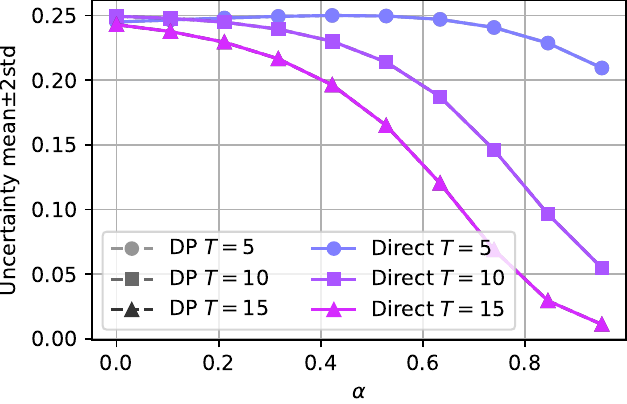}
        \caption{Uncertainty (DP behind Direct)}
    \end{subfigure}
    \begin{subfigure}[t]{.32\linewidth}
        \includegraphics[width=\linewidth]{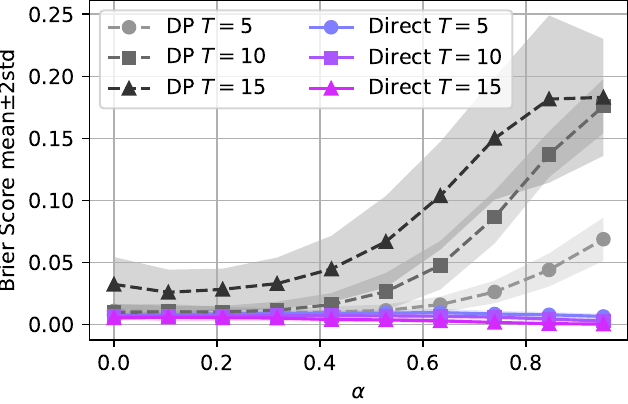}
        \caption{Brier (lower is better)}
    \end{subfigure}
    \caption{For uniform DP data and {\color{red}$\hat N=N$}: Results comparing DP and the direct approach for various $\alpha\in[0,1)$ and $T=5,10,15$ in terms of different metrics. Using same hyperparameters as for $\hat N=NT$.} 
    \label{fig:results_by_T_sameN_indeptrain}
\end{figure*}

\begin{figure*}[!ht]
    \centering
    \begin{subfigure}[t]{.32\linewidth}
        \includegraphics[width=\linewidth]{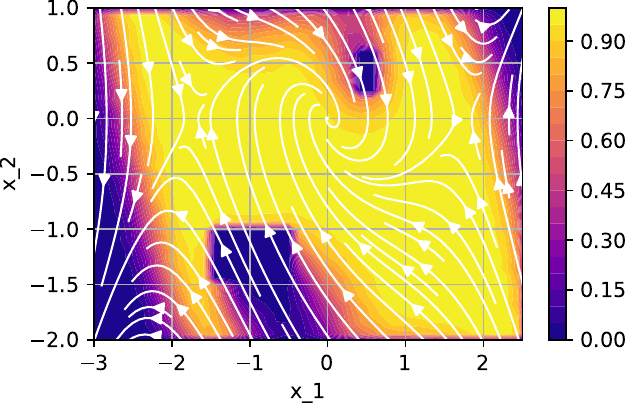}
        \caption{MC Ground Truth ($T=5$)}
    \end{subfigure}
    \begin{subfigure}[t]{.32\linewidth}
        \includegraphics[width=\linewidth]{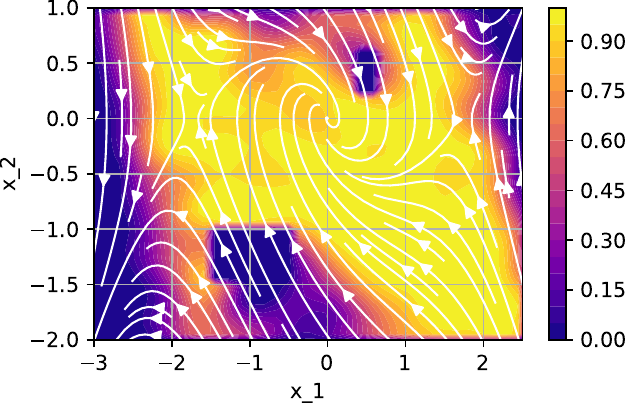}
        \caption{DP ($T=5$)}
    \end{subfigure}
    \begin{subfigure}[t]{.32\linewidth}
        \includegraphics[width=\linewidth]{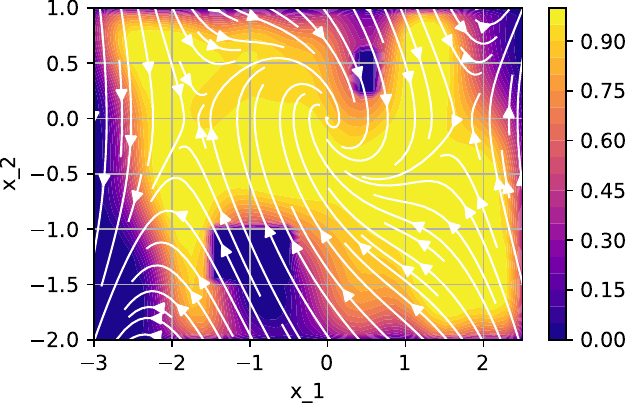}
        \caption{Direct ($T=5$)}
    \end{subfigure}
    \\[.6em]
    \begin{subfigure}[t]{.32\linewidth}
        \includegraphics[width=\linewidth]{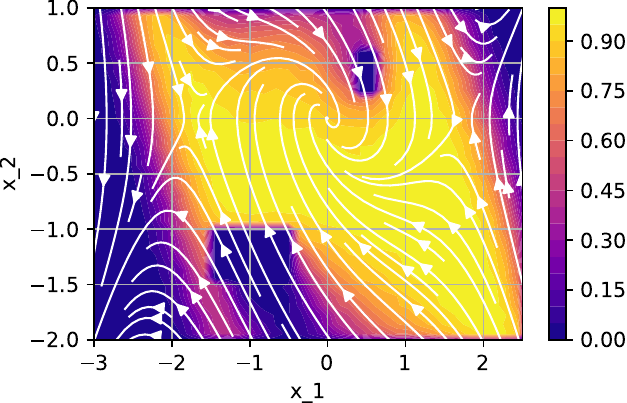}
        \caption{MC Ground Truth ($T=10$)}
    \end{subfigure}
    \begin{subfigure}[t]{.32\linewidth}
        \includegraphics[width=\linewidth]{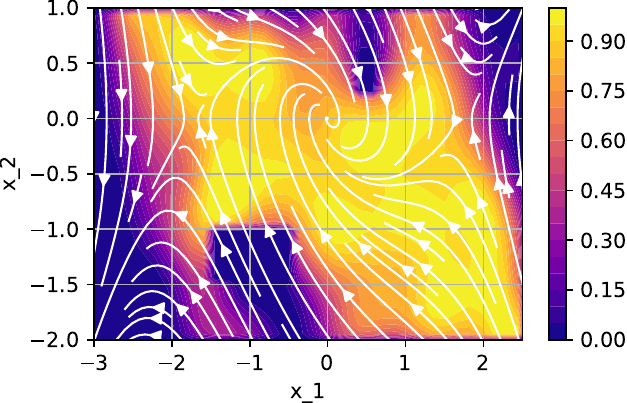}
        \caption{DP ($T=10$)}
    \end{subfigure}
    \begin{subfigure}[t]{.32\linewidth}
        \includegraphics[width=\linewidth]{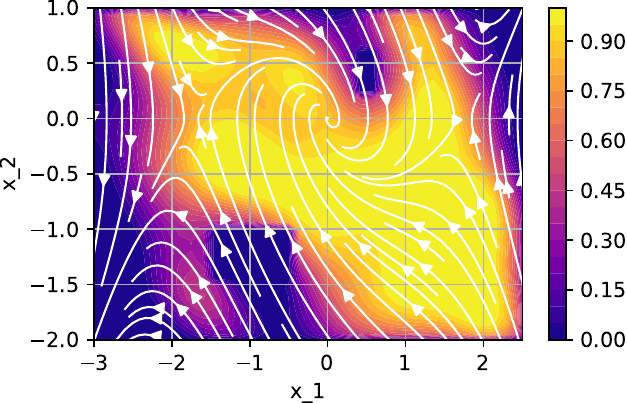}
        \caption{Direct ($T=10$)}
    \end{subfigure}
    \\[.6em]
    \begin{subfigure}[t]{.32\linewidth}
        \includegraphics[width=\linewidth]{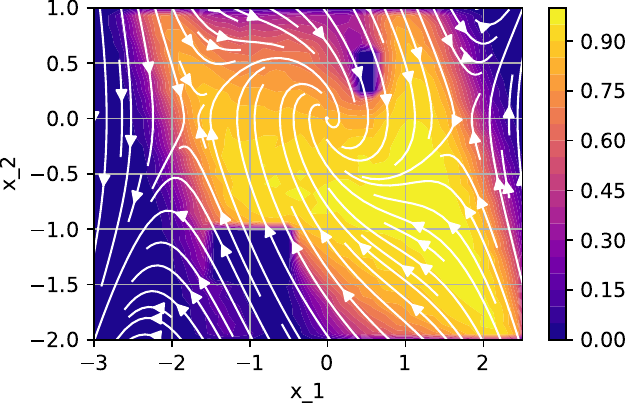}
        \caption{MC Ground Truth ($T=15$)}
    \end{subfigure}
    \begin{subfigure}[t]{.32\linewidth}
        \includegraphics[width=\linewidth]{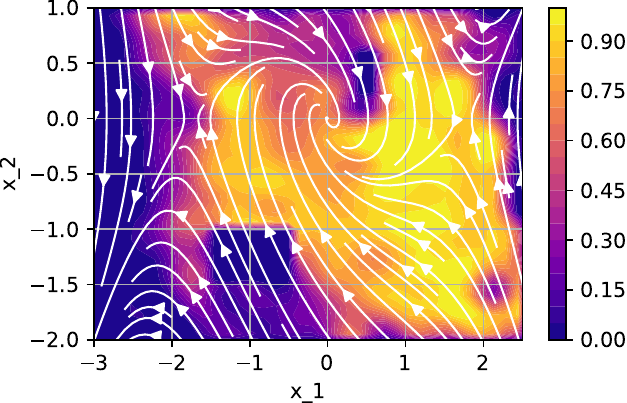}
        \caption{DP ($T=15$)}
    \end{subfigure}
    \begin{subfigure}[t]{.32\linewidth}
        \includegraphics[width=\linewidth]{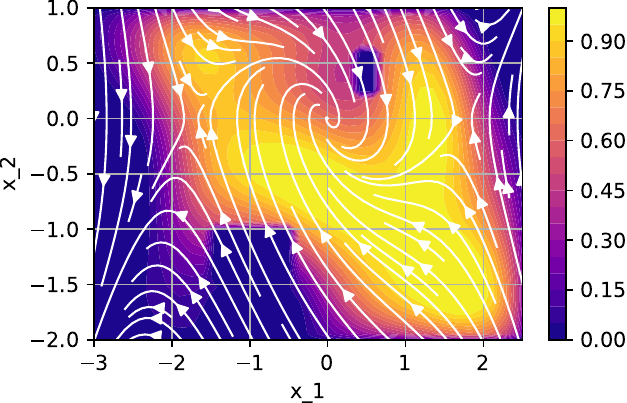}
        \caption{Direct ($T=15$)}
    \end{subfigure}
    \caption{For uniform DP data and {\color{red}$\hat N=N$}: Safety probability estimates (in color) and unknown true dynamics (mean) vector field (white arrows; all identical) for the fully Markovian dynamics ($\alpha=0$). Using same hyperparameters as for $\hat N=NT$.}
    \label{fig:safetyprob_a0_00_sameN_indeptrain}
\end{figure*}

\begin{figure*}[!ht]
    \centering
    \begin{subfigure}[t]{.32\linewidth}
        \includegraphics[width=\linewidth]{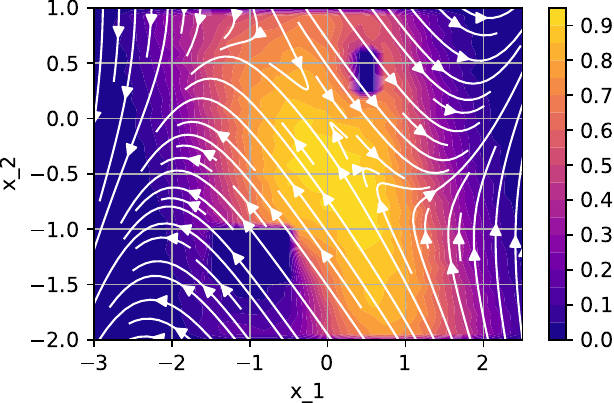}
        \caption{MC Ground Truth ($T=5$)}
    \end{subfigure}
    \begin{subfigure}[t]{.32\linewidth}
        \includegraphics[width=\linewidth]{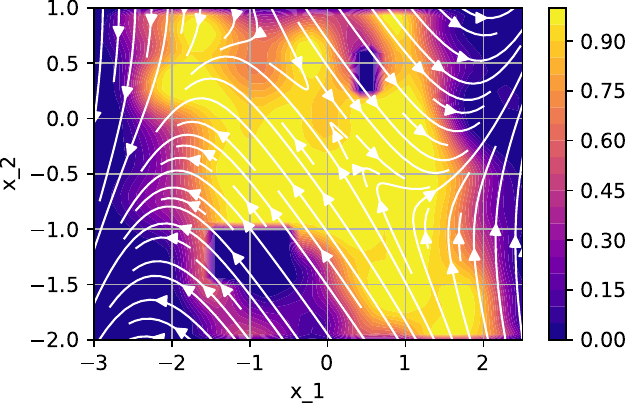}
        \caption{DP ($T=5$)}
    \end{subfigure}
    \begin{subfigure}[t]{.32\linewidth}
        \includegraphics[width=\linewidth]{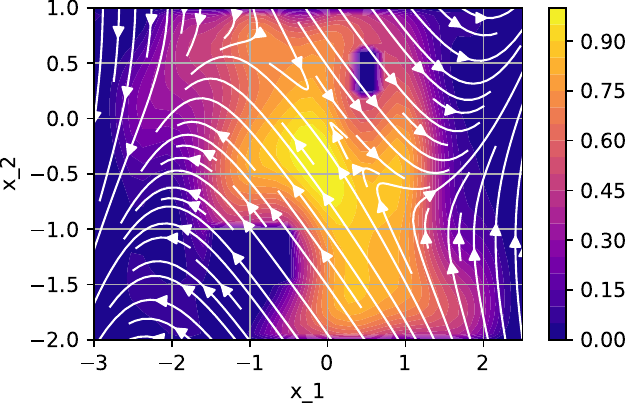}
        \caption{Direct ($T=5$)}
    \end{subfigure}
    \\[.6em]
    \begin{subfigure}[t]{.32\linewidth}
        \includegraphics[width=\linewidth]{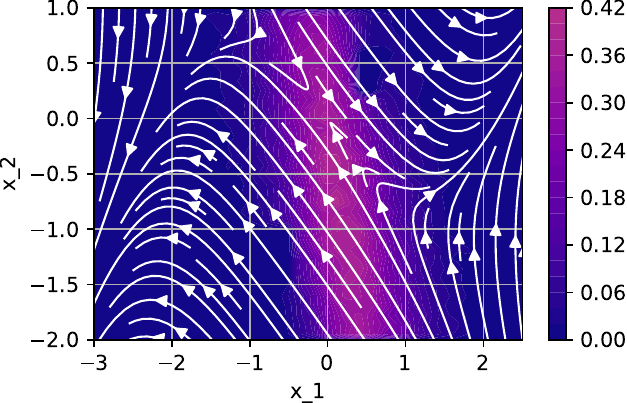}
        \caption{MC Ground Truth ($T=10$)}
    \end{subfigure}
    \begin{subfigure}[t]{.32\linewidth}
        \includegraphics[width=\linewidth]{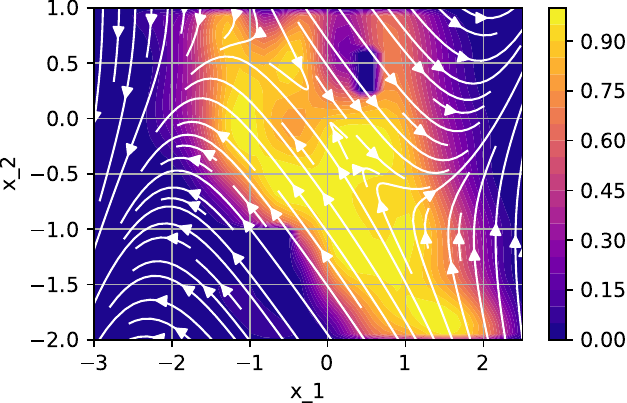}
        \caption{DP ($T=10$)}
    \end{subfigure}
    \begin{subfigure}[t]{.32\linewidth}
        \includegraphics[width=\linewidth]{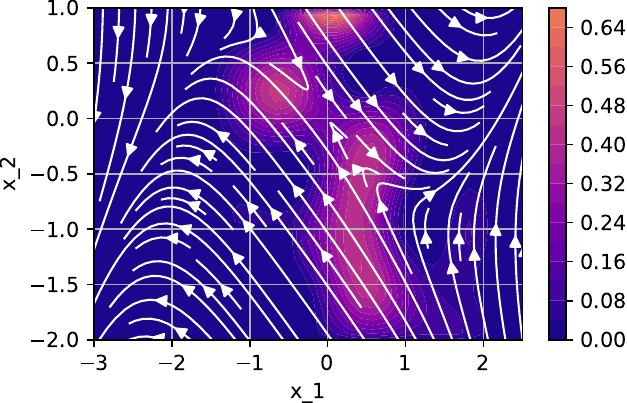}
        \caption{Direct ($T=10$)}
    \end{subfigure}
    \\[.6em]
    \begin{subfigure}[t]{.32\linewidth}
        \includegraphics[width=\linewidth]{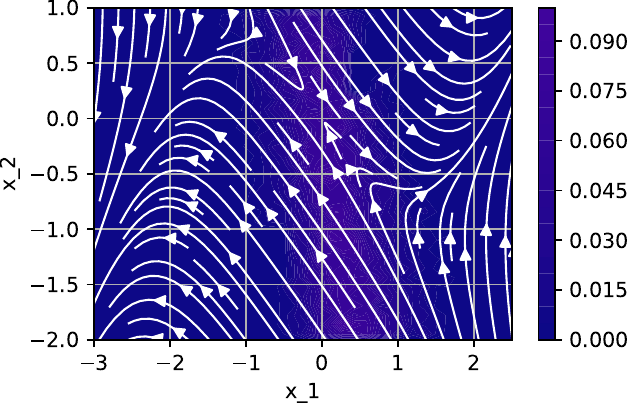}
        \caption{MC Ground Truth ($T=15$)}
    \end{subfigure}
    \begin{subfigure}[t]{.32\linewidth}
        \includegraphics[width=\linewidth]{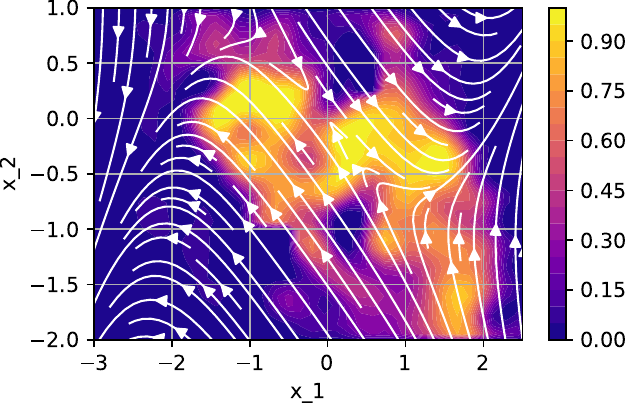}
        \caption{DP ($T=15$)}
    \end{subfigure}
    \begin{subfigure}[t]{.32\linewidth}
        \includegraphics[width=\linewidth]{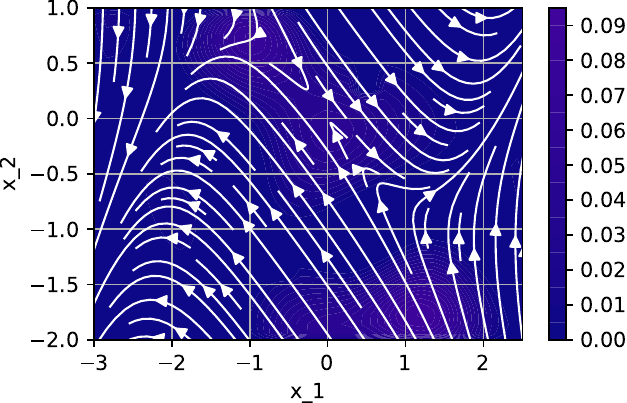}
        \caption{Direct ($T=15$)}
    \end{subfigure}
    \caption{For uniform DP data and {\color{red}$\hat N=N$}: Safety probability estimates (in color) and unknown true dynamics (mean) vector field (white arrows; all identical) for the fully Markovian dynamics ({\color{red}$\alpha=0.95$}). Using same hyperparameters as for $\hat N=NT$.}
    \label{fig:safetyprob_a0_95_sameN_indeptrain}
\end{figure*}


\begin{figure*}[!ht]
    \centering
    \begin{subfigure}[t]{.32\linewidth}
        \includegraphics[width=\linewidth]{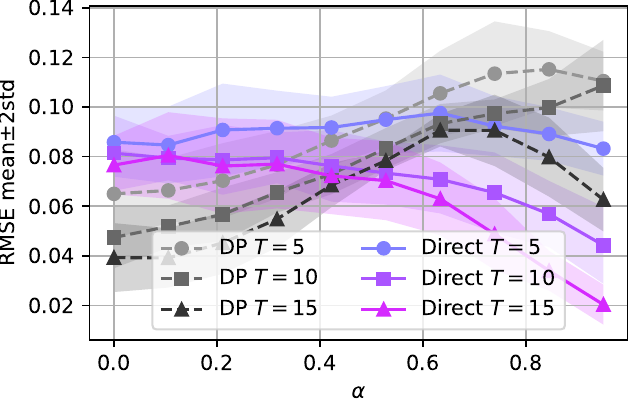}
        \caption{RMSE (lower is better)}
    \end{subfigure}
    \begin{subfigure}[t]{.32\linewidth}
        \includegraphics[width=\linewidth]{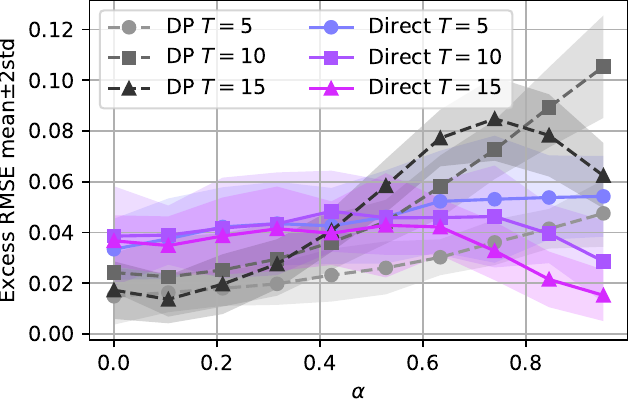}
        \caption{Excess RMSE (lower is better)}
    \end{subfigure}
    \begin{subfigure}[t]{.32\linewidth}
        \includegraphics[width=\linewidth]{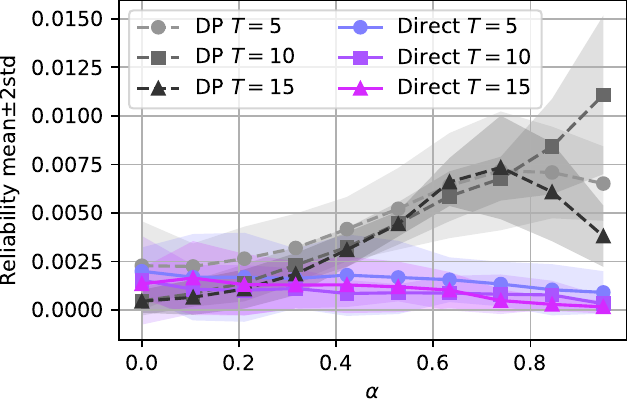}
        \caption{Reliability (lower is better)}
    \end{subfigure}
    \\[.6em]
    \begin{subfigure}[t]{.32\linewidth}
        \includegraphics[width=\linewidth]{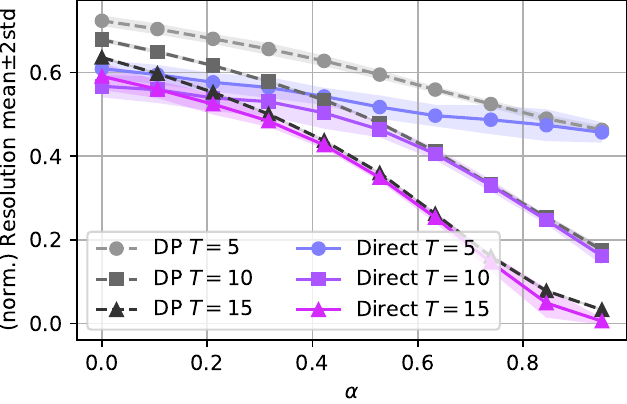}
        \caption{Resolution (normalized) (1 best/0 worst)}
    \end{subfigure}
    \begin{subfigure}[t]{.32\linewidth}
        \includegraphics[width=\linewidth]{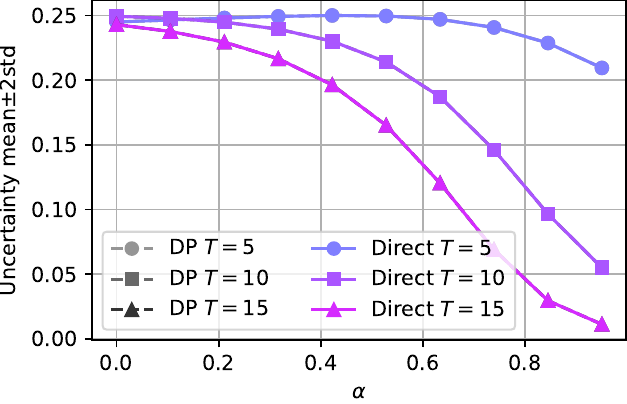}
        \caption{Uncertainty (DP behind Direct)}
    \end{subfigure}
    \begin{subfigure}[t]{.32\linewidth}
        \includegraphics[width=\linewidth]{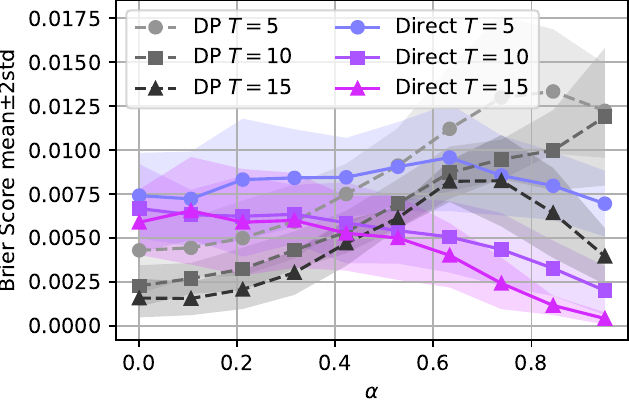}
        \caption{Brier (lower is better)}
    \end{subfigure}
    \caption{For {\color{red}dependent} DP data and $\hat N=NT$: Results comparing DP and the direct approach for various $\alpha\in[0,1)$ and $T=5,10,15$ in terms of different metrics.}
    \label{fig:results_by_T_diffN_deptrain}
\end{figure*}

\begin{figure*}[!ht]
    \centering
    \begin{subfigure}[t]{.32\linewidth}
        \includegraphics[width=\linewidth]{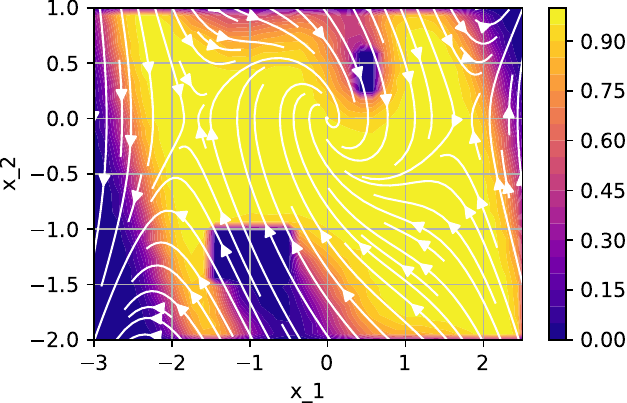}
        \caption{MC Ground Truth ($T=5$)}
    \end{subfigure}
    \begin{subfigure}[t]{.32\linewidth}
        \includegraphics[width=\linewidth]{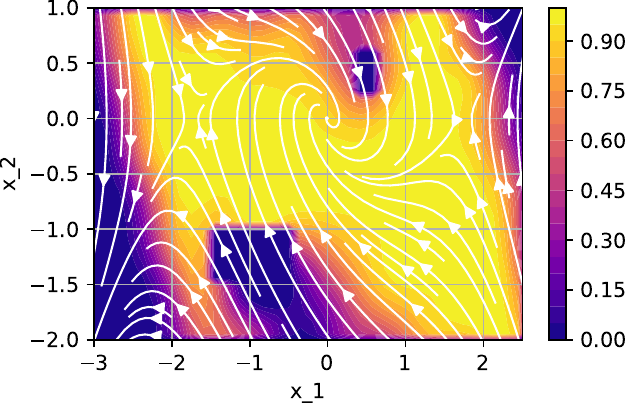}
        \caption{DP ($T=5$)}
    \end{subfigure}
    \begin{subfigure}[t]{.32\linewidth}
        \includegraphics[width=\linewidth]{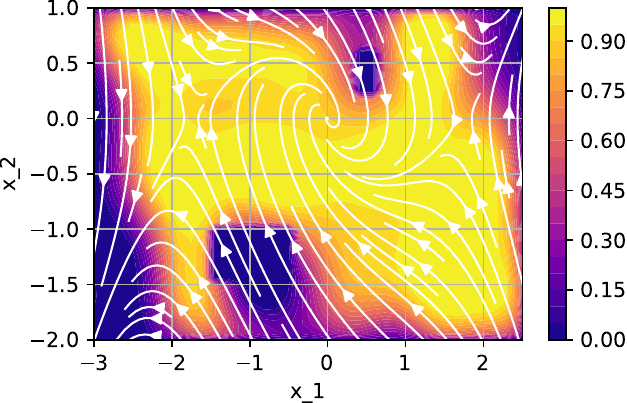}
        \caption{Direct ($T=5$)}
    \end{subfigure}
    \\[.6em]
    \begin{subfigure}[t]{.32\linewidth}
        \includegraphics[width=\linewidth]{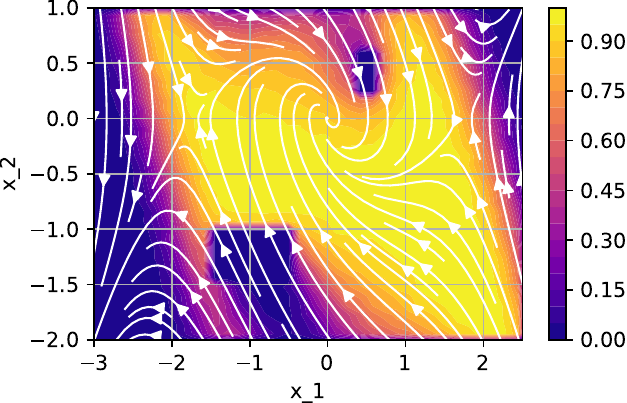}
        \caption{MC Ground Truth ($T=10$)}
    \end{subfigure}
    \begin{subfigure}[t]{.32\linewidth}
        \includegraphics[width=\linewidth]{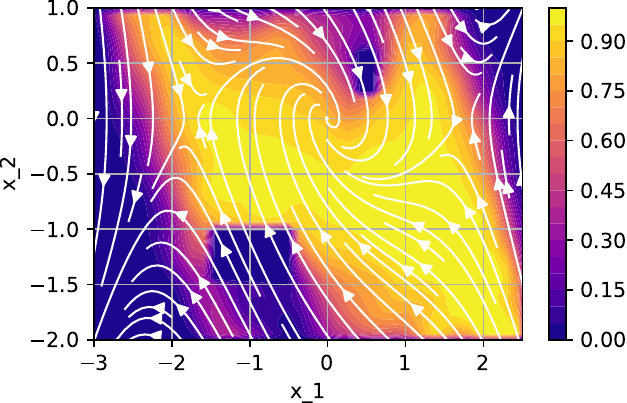}
        \caption{DP ($T=10$)}
    \end{subfigure}
    \begin{subfigure}[t]{.32\linewidth}
        \includegraphics[width=\linewidth]{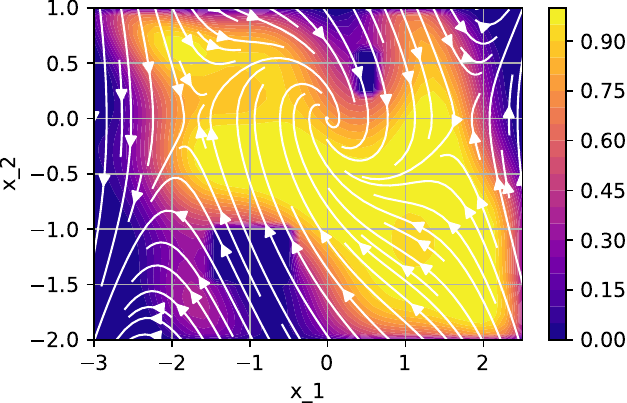}
        \caption{Direct ($T=10$)}
    \end{subfigure}
    \\[.6em]
    \begin{subfigure}[t]{.32\linewidth}
        \includegraphics[width=\linewidth]{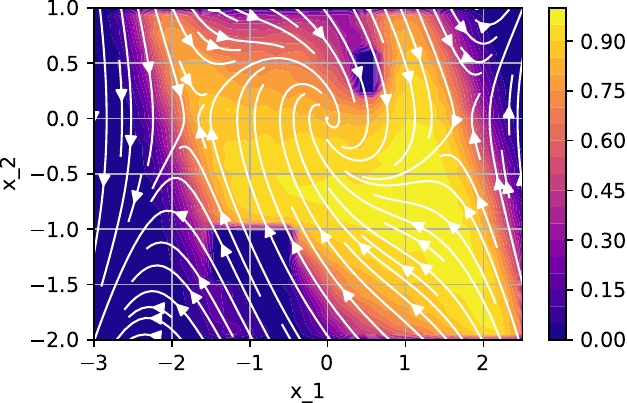}
        \caption{MC Ground Truth ($T=15$)}
    \end{subfigure}
    \begin{subfigure}[t]{.32\linewidth}
        \includegraphics[width=\linewidth]{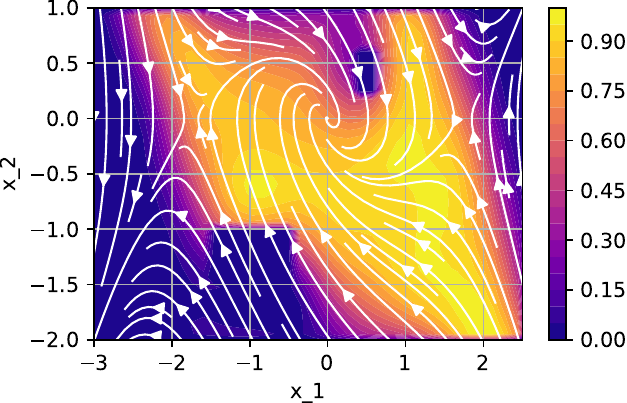}
        \caption{DP ($T=15$)}
    \end{subfigure}
    \begin{subfigure}[t]{.32\linewidth}
        \includegraphics[width=\linewidth]{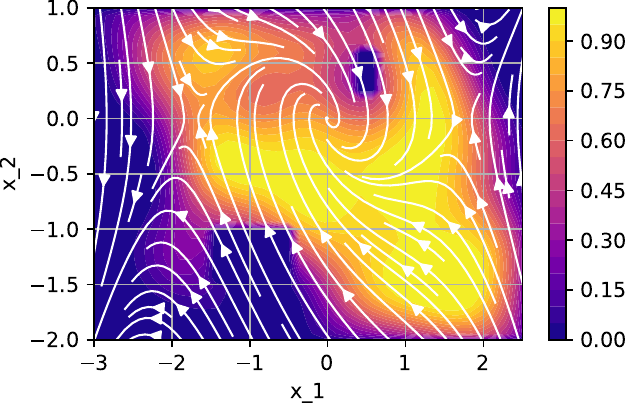}
        \caption{Direct ($T=15$)}
    \end{subfigure}
    \caption{For {\color{red}dependent} DP data and $\hat N=NT$: Safety probability estimates (in color) and unknown true dynamics (mean) vector field (white arrows; all identical) for the fully Markovian dynamics ($\alpha=0$).}
    \label{fig:safetyprob_a0_00_diffN_deptrain}
\end{figure*}

\begin{figure*}[!ht]
    \centering
    \begin{subfigure}[t]{.32\linewidth}
        \includegraphics[width=\linewidth]{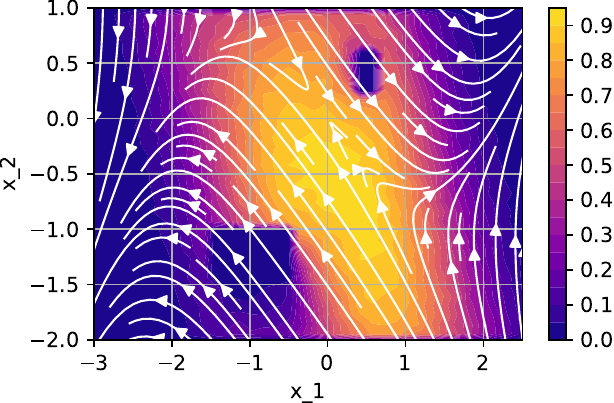}
        \caption{MC Ground Truth ($T=5$)}
    \end{subfigure}
    \begin{subfigure}[t]{.32\linewidth}
        \includegraphics[width=\linewidth]{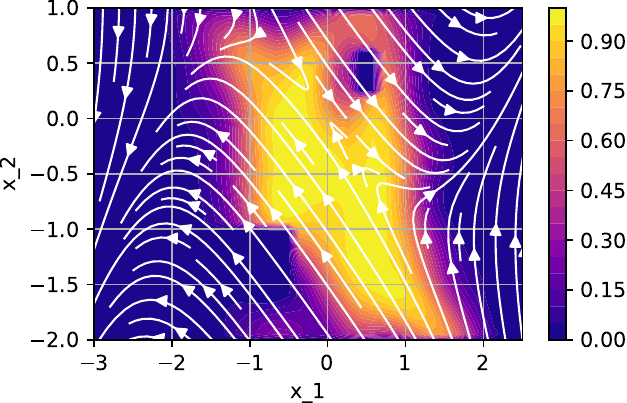}
        \caption{DP ($T=5$)}
    \end{subfigure}
    \begin{subfigure}[t]{.32\linewidth}
        \includegraphics[width=\linewidth]{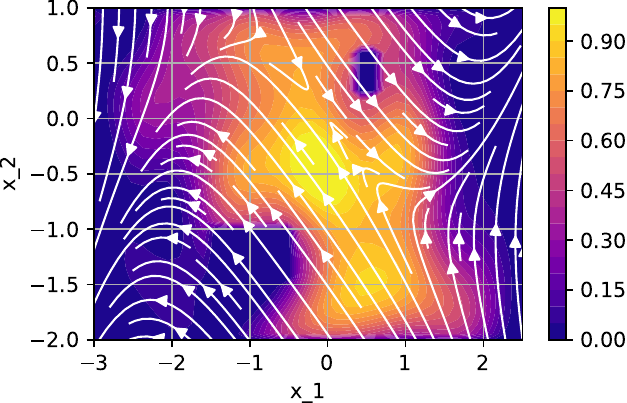}
        \caption{Direct ($T=5$)}
    \end{subfigure}
    \\[.6em]
    \begin{subfigure}[t]{.32\linewidth}
        \includegraphics[width=\linewidth]{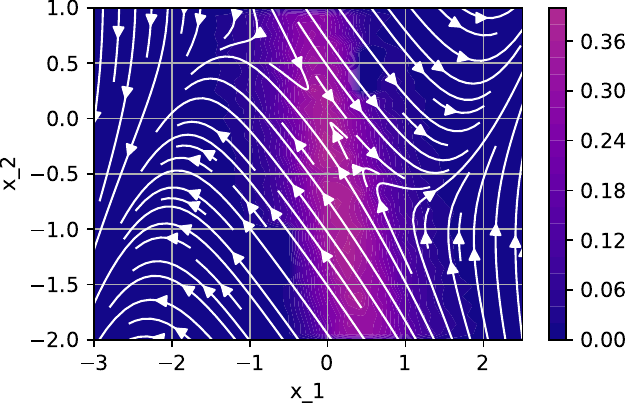}
        \caption{MC Ground Truth ($T=10$)}
    \end{subfigure}
    \begin{subfigure}[t]{.32\linewidth}
        \includegraphics[width=\linewidth]{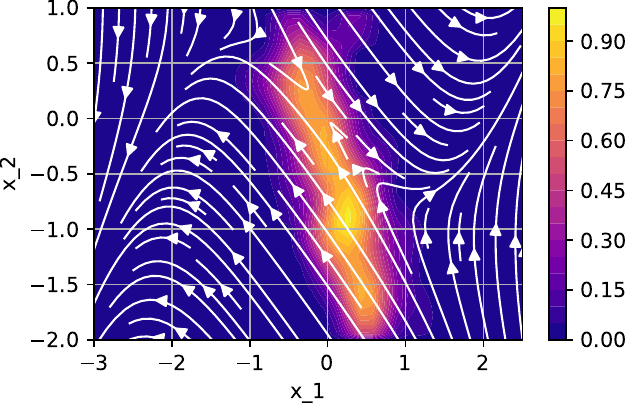}
        \caption{DP ($T=10$)}
    \end{subfigure}
    \begin{subfigure}[t]{.32\linewidth}
        \includegraphics[width=\linewidth]{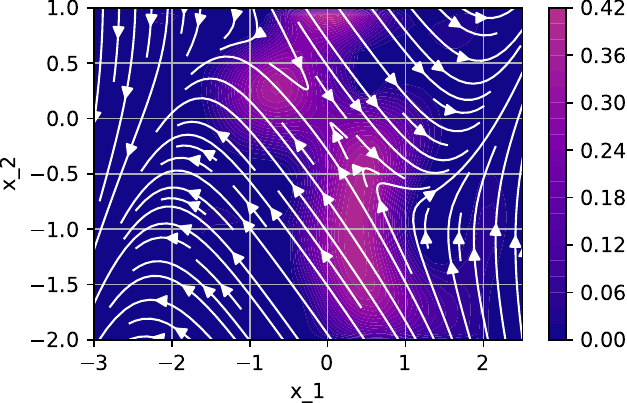}
        \caption{Direct ($T=10$)}
    \end{subfigure}
    \\[.6em]
    \begin{subfigure}[t]{.32\linewidth}
        \includegraphics[width=\linewidth]{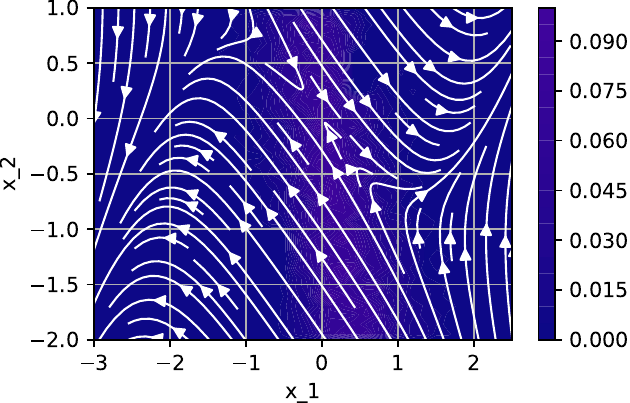}
        \caption{MC Ground Truth ($T=15$)}
    \end{subfigure}
    \begin{subfigure}[t]{.32\linewidth}
        \includegraphics[width=\linewidth]{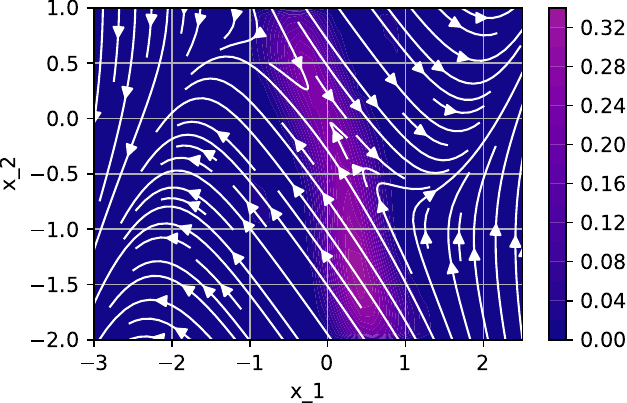}
        \caption{DP ($T=15$)}
    \end{subfigure}
    \begin{subfigure}[t]{.32\linewidth}
        \includegraphics[width=\linewidth]{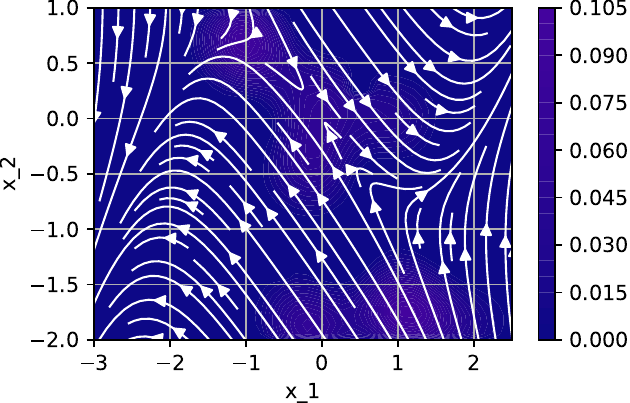}
        \caption{Direct ($T=15$)}
    \end{subfigure}
    \caption{For {\color{red}dependent} DP data and $\hat N=NT$: Safety probability estimates (in color) and unknown true dynamics (mean) vector field (white arrows; all identical) for the fully Markovian dynamics ({\color{red}$\alpha=0.95$}).}
    \label{fig:safetyprob_a0_95_diffN_deptrain}
\end{figure*}

\begin{figure*}[!ht]
    \centering
    \begin{subfigure}[t]{.32\linewidth}
        \includegraphics[width=\linewidth]{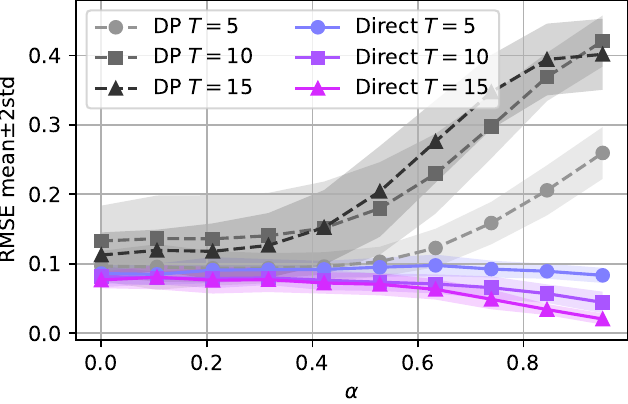}
        \caption{RMSE (lower is better)}
    \end{subfigure}
    \begin{subfigure}[t]{.32\linewidth}
        \includegraphics[width=\linewidth]{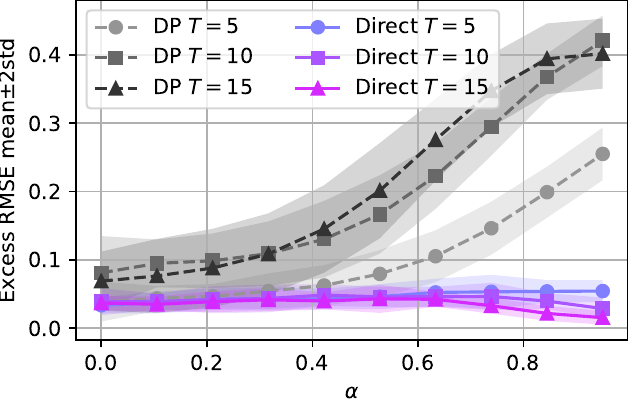}
        \caption{Excess RMSE (lower is better)}
    \end{subfigure}
    \begin{subfigure}[t]{.32\linewidth}
        \includegraphics[width=\linewidth]{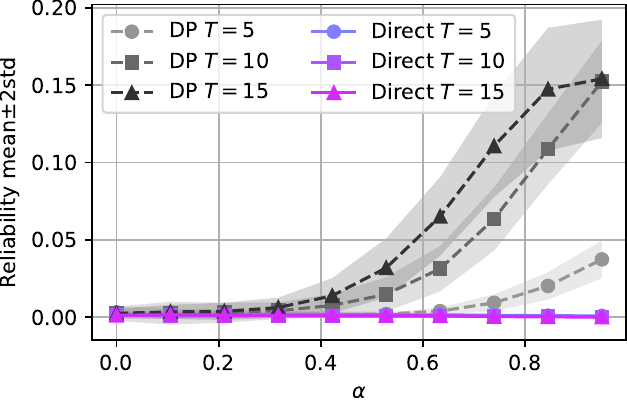}
        \caption{Reliability (lower is better)}
    \end{subfigure}
    \\[.6em]
    \begin{subfigure}[t]{.32\linewidth}
        \includegraphics[width=\linewidth]{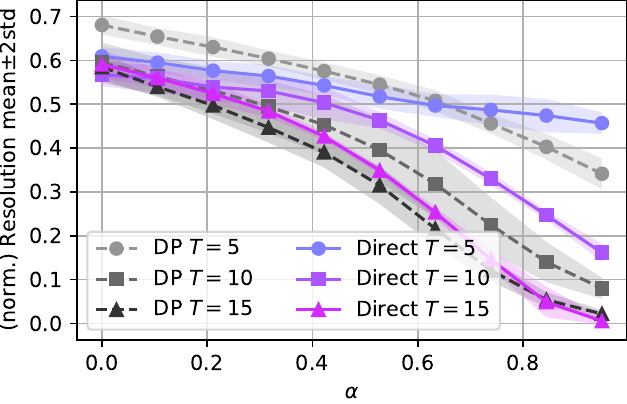}
        \caption{Resolution (normalized) (1 best/0 worst)}
    \end{subfigure}
    \begin{subfigure}[t]{.32\linewidth}
        \includegraphics[width=\linewidth]{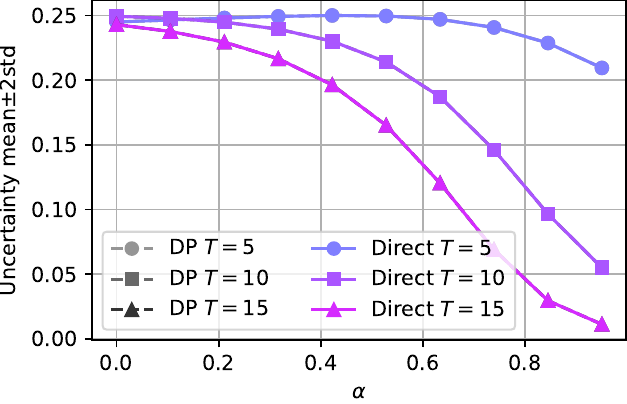}
        \caption{Uncertainty (DP behind Direct)}
    \end{subfigure}
    \begin{subfigure}[t]{.32\linewidth}
        \includegraphics[width=\linewidth]{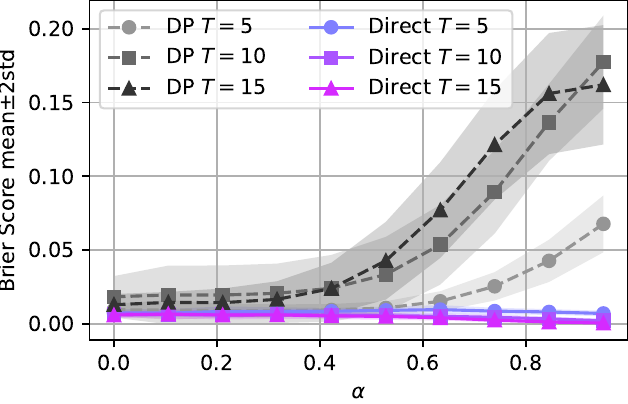}
        \caption{Brier$\downarrow$}
    \end{subfigure}
    \caption{For {\color{red}dependent} DP data and {\color{red}$\hat N=N$}: Results comparing DP and the direct approach for various $\alpha\in[0,1)$ and $T=5,10,15$ in terms of different metrics. Using same hyperparameters as for $\hat N=NT$.} 
    \label{fig:results_by_T_sameN_deptrain}
\end{figure*}

\begin{figure*}[!ht]
    \centering
    \begin{subfigure}[t]{.32\linewidth}
        \includegraphics[width=\linewidth]{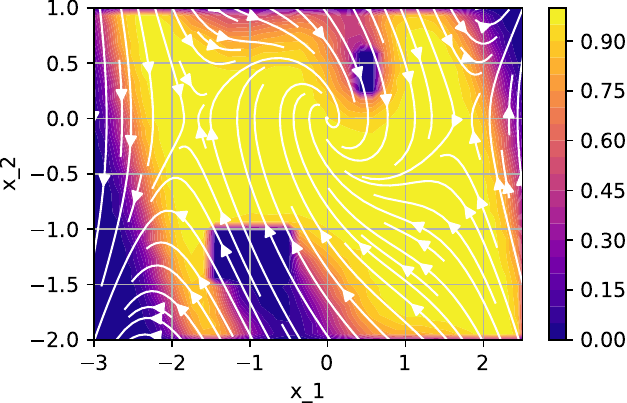}
        \caption{MC Ground Truth ($T=5$)}
    \end{subfigure}
    \begin{subfigure}[t]{.32\linewidth}
        \includegraphics[width=\linewidth]{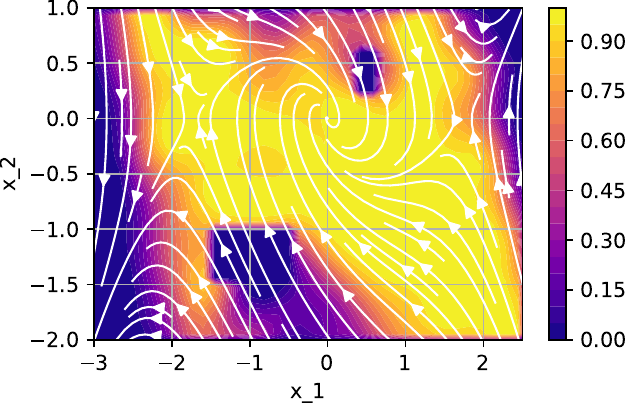}
        \caption{DP ($T=5$)}
    \end{subfigure}
    \begin{subfigure}[t]{.32\linewidth}
        \includegraphics[width=\linewidth]{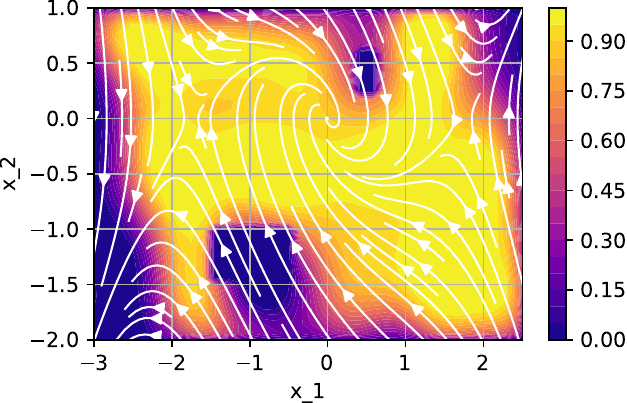}
        \caption{Direct ($T=5$)}
    \end{subfigure}
    \\[.6em]
    \begin{subfigure}[t]{.32\linewidth}
        \includegraphics[width=\linewidth]{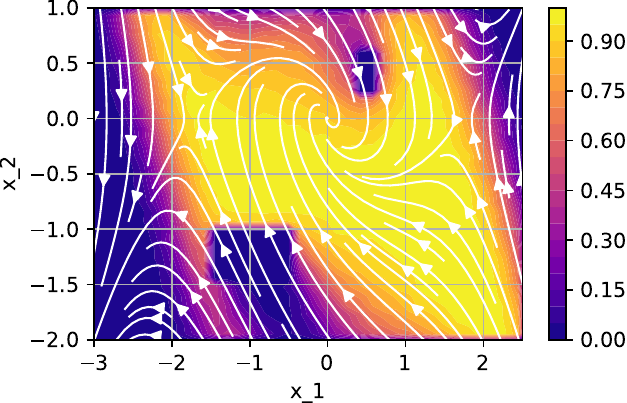}
        \caption{MC Ground Truth ($T=10$)}
    \end{subfigure}
    \begin{subfigure}[t]{.32\linewidth}
        \includegraphics[width=\linewidth]{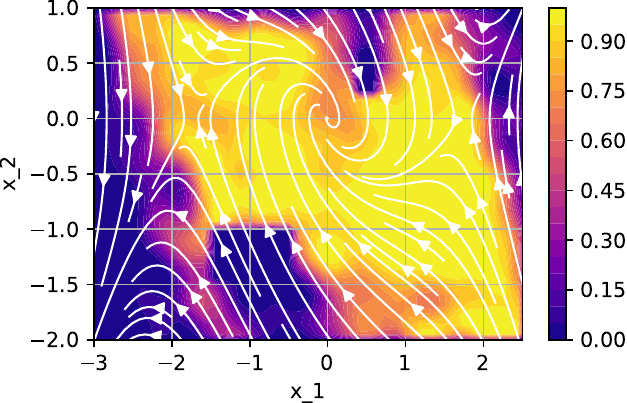}
        \caption{DP ($T=10$)}
    \end{subfigure}
    \begin{subfigure}[t]{.32\linewidth}
        \includegraphics[width=\linewidth]{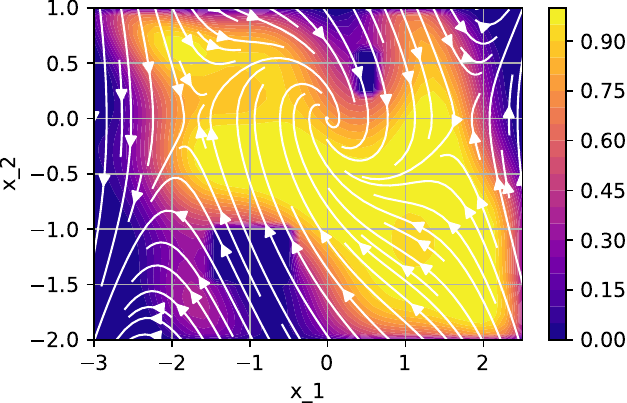}
        \caption{Direct ($T=10$)}
    \end{subfigure}
    \\[.6em]
    \begin{subfigure}[t]{.32\linewidth}
        \includegraphics[width=\linewidth]{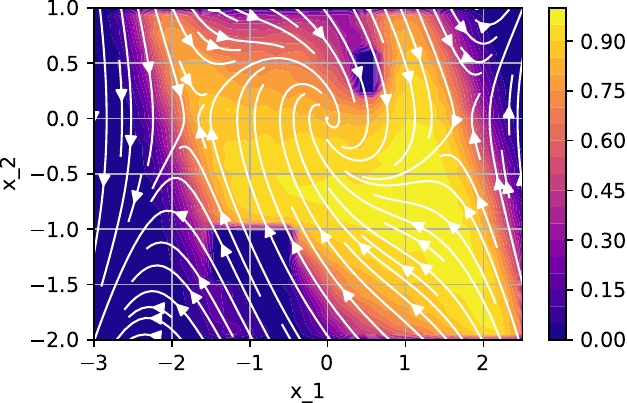}
        \caption{MC Ground Truth ($T=15$)}
    \end{subfigure}
    \begin{subfigure}[t]{.32\linewidth}
        \includegraphics[width=\linewidth]{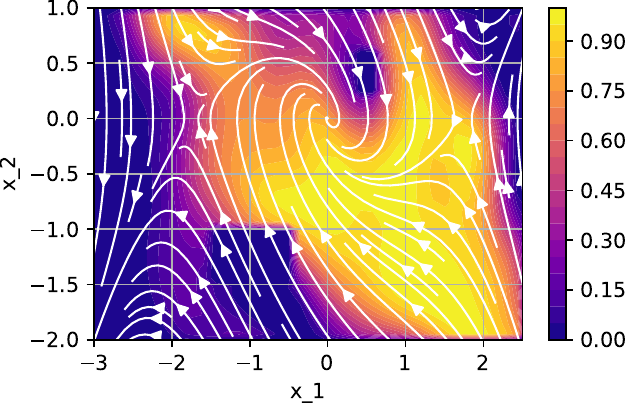}
        \caption{DP ($T=15$)}
    \end{subfigure}
    \begin{subfigure}[t]{.32\linewidth}
        \includegraphics[width=\linewidth]{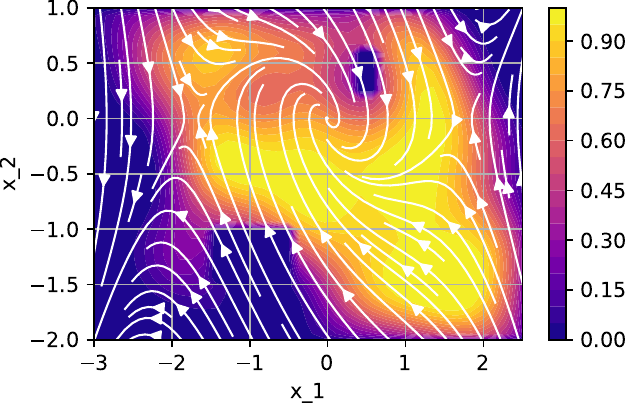}
        \caption{Direct ($T=15$)}
    \end{subfigure}
    \caption{For {\color{red}dependent} DP data and {\color{red}$\hat N=N$}: Safety probability estimates (in color) and unknown true dynamics (mean) vector field (white arrows; all identical) for the fully Markovian dynamics ($\alpha=0$). Using same hyperparameters as for $\hat N=NT$.}
    \label{fig:safetyprob_a0_00_sameN_deptrain}
\end{figure*}

\begin{figure*}[!ht]
    \centering
    \begin{subfigure}[t]{.32\linewidth}
        \includegraphics[width=\linewidth]{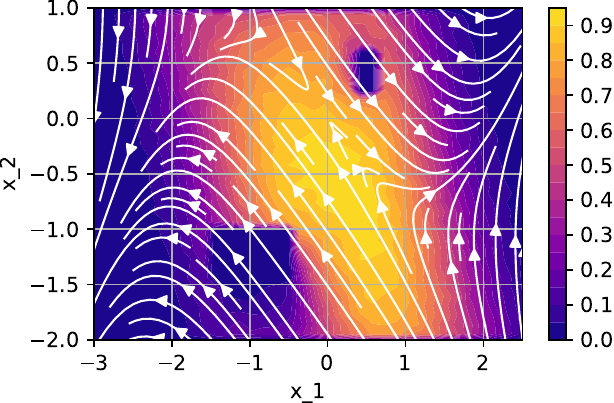}
        \caption{MC Ground Truth ($T=5$)}
    \end{subfigure}
    \begin{subfigure}[t]{.32\linewidth}
        \includegraphics[width=\linewidth]{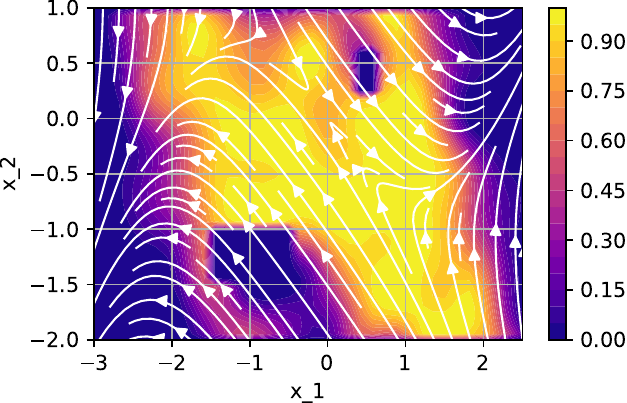}
        \caption{DP ($T=5$)}
    \end{subfigure}
    \begin{subfigure}[t]{.32\linewidth}
        \includegraphics[width=\linewidth]{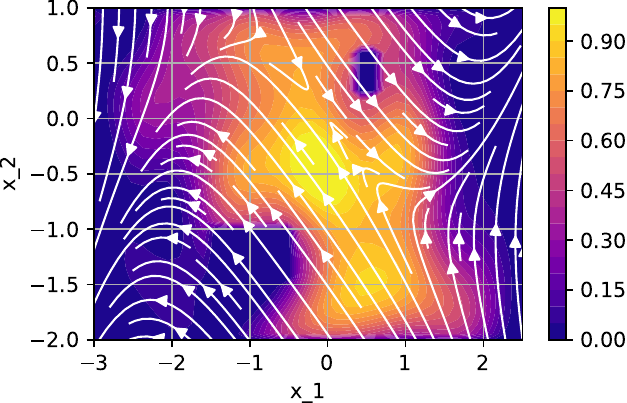}
        \caption{Direct ($T=5$)}
    \end{subfigure}
    \\[.6em]
    \begin{subfigure}[t]{.32\linewidth}
        \includegraphics[width=\linewidth]{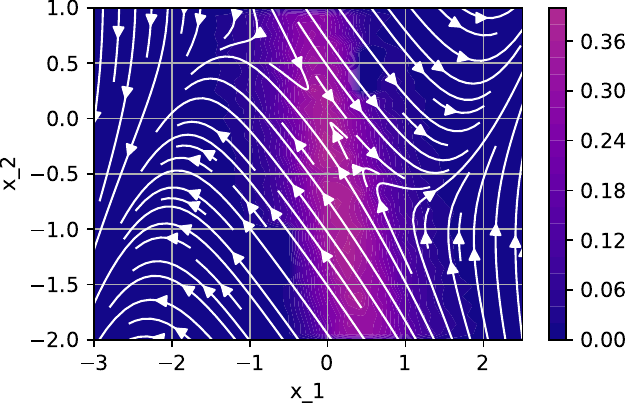}
        \caption{MC Ground Truth ($T=10$)}
    \end{subfigure}
    \begin{subfigure}[t]{.32\linewidth}
        \includegraphics[width=\linewidth]{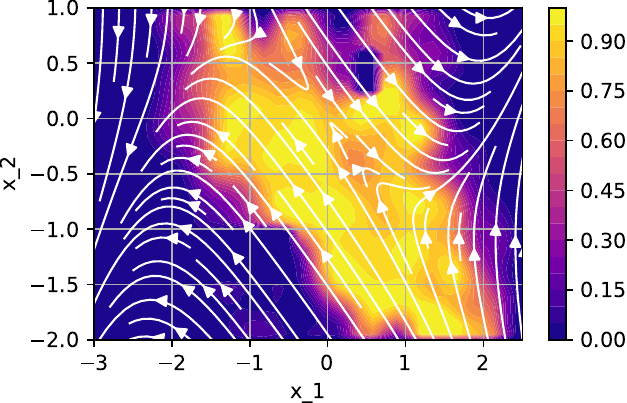}
        \caption{DP ($T=10$)}
    \end{subfigure}
    \begin{subfigure}[t]{.32\linewidth}
        \includegraphics[width=\linewidth]{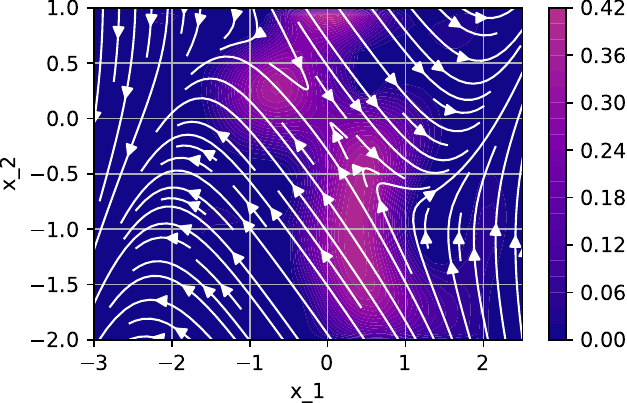}
        \caption{Direct ($T=10$)}
    \end{subfigure}
    \\[.6em]
    \begin{subfigure}[t]{.32\linewidth}
        \includegraphics[width=\linewidth]{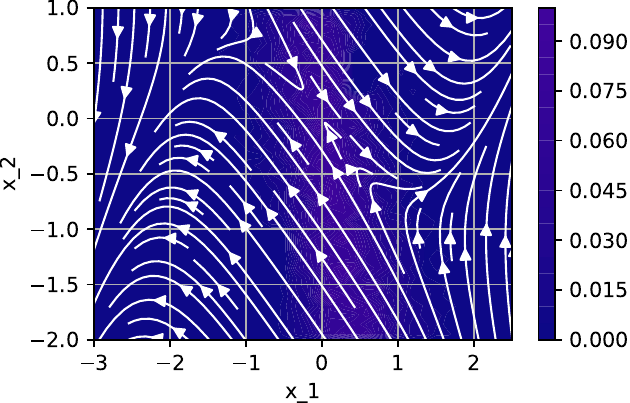}
        \caption{MC Ground Truth ($T=15$)}
    \end{subfigure}
    \begin{subfigure}[t]{.32\linewidth}
        \includegraphics[width=\linewidth]{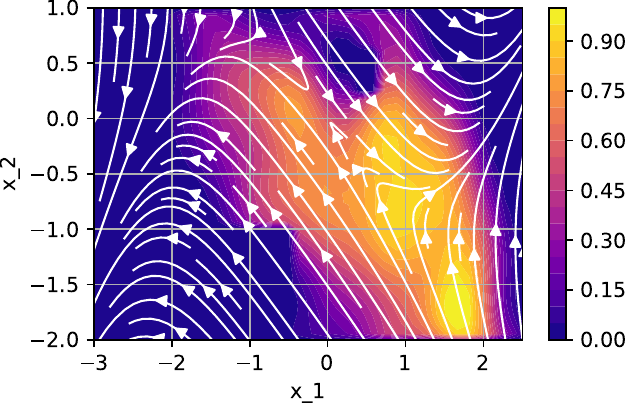}
        \caption{DP ($T=15$)}
    \end{subfigure}
    \begin{subfigure}[t]{.32\linewidth}
        \includegraphics[width=\linewidth]{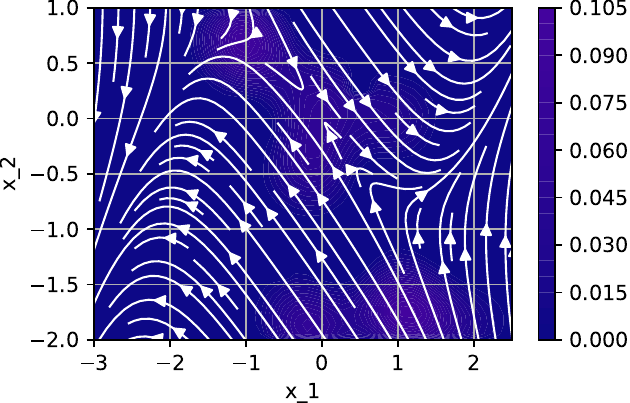}
        \caption{Direct ($T=15$)}
    \end{subfigure}
    \caption{For {\color{red}dependent} DP data and {\color{red}$\hat N=N$}: Safety probability estimates (in color) and unknown true dynamics (mean) vector field (white arrows; all identical) for the fully Markovian dynamics ({\color{red}$\alpha=0.95$}). Using same hyperparameters as for $\hat N=NT$.} 
    \label{fig:safetyprob_a0_95_sameN_deptrain}
\end{figure*}

\subsection{Real-Data Benchmark: Neural-Controlled Racing Quadrotor}\label{app:real_data_quadrotor}

We evaluate on the drone racing flight dataset\footnote{Dataset available at \url{https://doi.org/10.5281/zenodo.7955278}} of \citet{kaufmann2023champion}, recorded during a competition between the AI system \emph{Swift} and human world-champion pilots.
We consider two pilots, Swift and M.\ Schaepper, and predict whether altitude $p_z > 0.2\,\text{m}$ is maintained for $1\,\text{s}$ after each gate passage.
The initial state $x_0$ is the $15$-step (${\approx}0.23\,\text{s}$) position history $(p_x, p_y, p_z)$ immediately preceding each gate crossing; each gate passage constitutes one training sample, yielding $N{=}560$ for Swift and $N{=}536$ for Schaepper (${\approx}5\%$ unsafe).

\textbf{Data splitting.}\;
All methods are evaluated via $5$-fold cross-validation: for each fold, the model is trained on 80\% of gate passages and evaluated on the held-out 20\%.
Out-of-fold predictions are assembled in original sample order for Brier evaluation (no sample is ever evaluated on its own training fold).
Hyperparameters are selected by nested cross-validation within the training folds.
There is no data leakage.

As no MC ground truth is available, evaluation uses the Brier decomposition against observed binary safety outcomes (safe/unsafe gate passage).

With $\varrho^{100}{\approx}0.96$ the DP recursion barely contracts, yet DP is still severely miscalibrated: REL is $18\times$ and $26\times$ higher than Direct for Swift and Schaepper respectively (Table~\ref{tab:benchmark_metrics_app}).
Figure~\ref{fig:quadrotor_real_histogram} confirms the structural nature of the failure: under Direct KRR, safe approaches concentrate near predicted probability 1 while unsafe approaches spread to lower values; under Indirect DP, both distributions are nearly uniform across $[0,1]$, leaving no discriminative information in the certificate.

\begin{table}[h]
\centering
\caption{Real-data benchmark results (Appendix~\ref{app:real_data_quadrotor}).
REL$\downarrow$, RES$\uparrow$; no MC ground truth available.}
\label{tab:benchmark_metrics_app}
\setlength{\tabcolsep}{5pt}
\small
\begin{tabular}{llcccc}
\toprule
Pilot & Method & Brier$\downarrow$ & REL$\downarrow$ & RES$\uparrow$ & $\varrho^{100}$ \\
\midrule
Swift (AI)       & Direct (ours) & $\mathbf{0.041}$ & $\mathbf{0.005}$ & $0.010$ & $-$ \\
                 & Indirect (DP) & $0.131$ & $0.093$ & $0.007$ & $9.6{\times}10^{-1}$ \\
Schaepper (human)& Direct (ours) & $\mathbf{0.036}$ & $\mathbf{0.004}$ & $0.012$ & $-$ \\
                 & Indirect (DP) & $0.137$ & $0.105$ & $0.013$ & $9.6{\times}10^{-1}$ \\
\bottomrule
\end{tabular}
\end{table}

\begin{figure}[!ht]
    \centering
    \begin{subfigure}[t]{\linewidth}
        \centering
        \includegraphics[width=\linewidth]{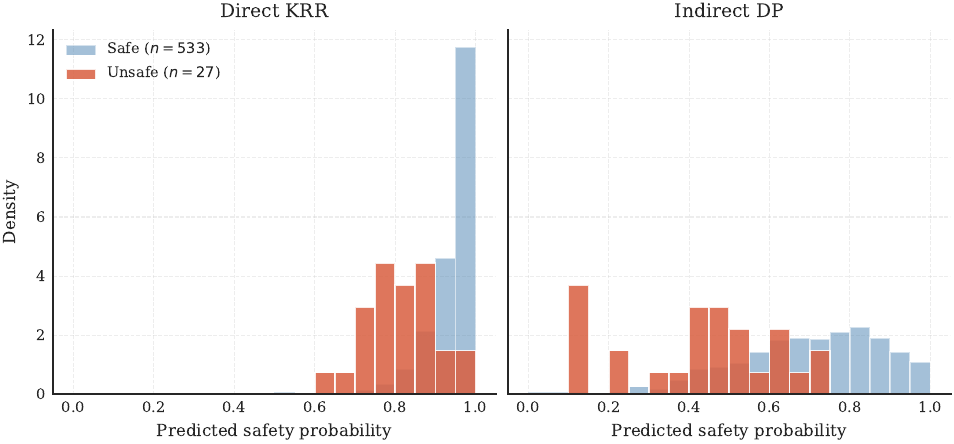}
        \caption{Swift (AI controller)}
    \end{subfigure}
    \begin{subfigure}[t]{\linewidth}
        \centering
        \includegraphics[width=\linewidth]{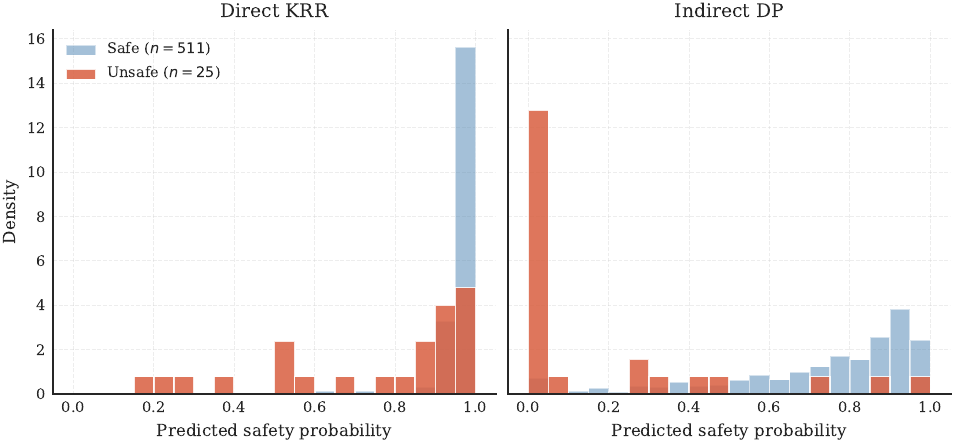}
        \caption{Schaepper (human pilot)}
    \end{subfigure}
    \caption{Distribution of predicted safety probabilities for safe (blue) and unsafe (orange) gate approaches.
    Direct KRR concentrates safe predictions near 1 and pushes unsafe ones lower.
    Indirect DP spreads both classes nearly uniformly across $[0,1]$: no discriminative information survives the violated Markov assumption.}
    \label{fig:quadrotor_real_histogram}
\end{figure}

\clearpage

\section{General Notation}\label{app:notation}
We denote the sets of positive integers and non-negative reals as $\N$ and $\R_{\geq 0}$, respectively.
Consider a Polish sample space $\Xdomain$ \citep{bogachev2007measure}.
Let $(\Xdomain,\borel{\Xdomain},\P)$ be the underlying probability space equipped with a Borel $\sigma$-algebra $\borel{\Xdomain}$ defined over $\Xdomain$, i.e., the smallest $\sigma$-algebra containing open subsets of $\Xdomain$, and a probability measure $\P$. 
For a random variable $X\colon \borel{\Xdomain} \to [0,1]$, let $p_X$ be the pushforward probability measure of $\P$ under $X$ such that $X\sim p_X(\cdotx)$.
The expected value of a function $f(X)$ on $\Xdomain$ is written as $\E_{p_X}[f(X)]$. If it is clear from the context, we abbreviate and write $\E[f(X)]$.
We denote the set of all probability measures for a given measurable space $(\Xdomain,\borel{\Xdomain})$ as $\Probs(\Xdomain)$.
The Dirac delta measure $\delta_a\colon\borel{\Xdomain}\rightarrow [0,1]$ concentrated at a point $a\in\Xdomain$ is defined as $\delta_a(A)=1$ if $a\in A$ and $\delta_a(A)=0$ otherwise, for any measurable set $A\in\borel{\Xdomain}$.
We denote the uniform distribution over $\Xdomain$ as $\mathcal{U}_{\Xdomain}$, with realizations $x\sim\mathcal{U}_{\Xdomain}$.
For two measurable spaces $(\Xdomain,\borel{\Xdomain})$ and $(\Ydomain,\borel{\Ydomain})$, a \emph{probability kernel} is a mapping $\p\colon \Xdomain \times \borel{\Ydomain}\rightarrow  [0,1]$ such that $\p(X=x,\cdotx)\colon\borel{\Ydomain}\rightarrow[0,1]$ is a probability measure for all $x\in\Xdomain$, and $\p(\cdotx, B)\colon \Xdomain\rightarrow [0,1]$ is measurable for all  $B\in\borel{\Ydomain}$ \citep[Chapter~10]{bogachev2007measure}.
A probability kernel associates to each point $x\in\Xdomain$ a measure denoted by $\p(\cdotx|X=x)$.

\section{RKHS Basics}\label{app:RKHS}
A symmetric function $k\colon\Xdomain\times\Xdomain\rightarrow\R$ is called a (positive definite) \emph{kernel} (note the distinction from \emph{probability kernels}) if for all $N\in\N_{>0}$ we have $\smash{\sum_{i=1}^{N}}\smash{\sum_{j=1}^{N}}a_i a_j\allowbreak k(x_i,x_j) \geq 0$ for $x_1,\ldots,x_N\in\Xdomain\subset{\R^n}$ and $a_1,\ldots,a_N\in\R$.
A prominent example is the Gaussian kernel \citep{Rasmussen2005GP,Kanagawa2018GPvsKernel}
$k(x,x') := \smash{\exp\left( -\frac{1}{2} (x-x')\T \Sigma^{-1} (x-x') \right),}$
where $\Sigma:=\smash{\diag(\sigma_l)^{2}}$,
with lengthscale coefficients $\sigma_l\in\R^n$.
For this work, we assume that all kernels are bounded on their domain, i.e., $\E_{}[k(x,x)]<\infty$, $x\in\Xdomain$.
Given a kernel $k$ on a non-empty set $\Xdomain$, there exists a unique corresponding
\emph{reproducing kernel Hilbert space} (RKHS) 
$\Hilbert$
of functions $f\colon\Xdomain\rightarrow\R$ equipped with an inner product $\smash{\innerH{\cdotx}{\cdotx}{\Hilbert}}$
with the celebrated \emph{reproducing property} such that for any function $f\in\Hilbert$ and $x\in\Xdomain$ we have $f(x)=\smash{\innerH{f}{k(\cdotx,x)}{\Hilbert}}$~\citep{Berlinet2004RKHSProbStat}.
Note that $k(\cdotx,x)\colon\Xdomain\rightarrow\Hilbert$ is a real-valued function,
which is also called an implicit \emph{canonical} \emph{embedding} or \emph{feature map} $\phi$
such that $k(x,x')=\smash{\innerH{\phi(x)}{\phi(x')}{\Hilbert}}$ for all $x,x'\in\Xdomain$.
For an RKHS $\Hilbert$, we use the associated feature map $\phi$ and kernel $k$ interchangeably for ease of notation and comprehensibility.
The inner product induces the norm $\smash{\norm{f}_{\Hilbert}}\!\!:=\!\!\sqrt{\smash[b]{\innerH{f}{f}{\Hilbert}}}$ of the RKHS. 
Throughout this paper, we assume that all RKHSs are \emph{separable}~\citep{Steinwart2008SVM,owhadi2015separability}.
Given $N$ i.i.d.\ samples $\smash{\{x^{(i)}\}_{i=1}^N}$ with $x_i\in\Xdomain$, the
\emph{Gram matrix} of $k$ is given by
$\smash{K_{N}}:={[k({x}^{(i)},{x}^{(j)})]_{i,j=1}^N}.$
Furthermore, we define the vector-valued function 
$\smash{k_{N}(x)} := \smash{[k(x,{x}^{(i)})]_{i=1}^N}.$

\section{Embedding-Framework Recoveries}\label{app:embedding_framework_recoveries}

\subsection{Barriers}\label{app:barriers}
Barrier certificates address safety with respect to an initial set $\Xdomain_0\subseteq\Xdomain$ \citep{prajna2006barrier}.
For stochastic systems with Markovian transitions, \cite{laurenti2025unifying} show that the barrier problem implicitly derives an upper bound $\Lambda(x)\geq 1 - V_l(x)$ on the \emph{un}safety probability\footnote{We choose to denote the value function for computing the unsafety probability as $\Lambda$ to contrast it with the standard definition of a value function $V$ to compute the \emph{safety} probability.}, computed recursively:
\begin{equation}
    \Lambda_T(x) := B(x),\qquad
    \Lambda_l(x) := \1_{\Xdomain\setminus S}(x) + \1_{S}(x) \, \E_{\law}[\Lambda_{l+1}(X_+)\mid X=x],
\label{eq:value_iteration_barriers}
\end{equation}
by finding a positive barrier function $B\colon\Xdomain\to\R_{\geq0}$,
which recovers $\Lambda_l(x)\equiv1 - V_l(x)$ for $B(x):=\1_{\Xdomain\setminus S}(x)$.
The barrier problem essentially places a bound on the maximum barrier value decrease---the maximum growth of the probability of exiting the safe set $S$ in a single time step---for all safe states $x\in S$, that is
\begin{equation}
    \beta \geq \sup_{x\in S} \big(\E_{\law}[B(X_+)\mid X=x]-B(x)\big).\label{eq:kushner}
\end{equation}
\citet[Theorem~2]{schoen2024CME} show how a bound $\beta$ can be obtained via CMEs to learn barriers from data, leading to the following proposition.
In particular, $(\beta, B)$ can be interpreted as the robust learning-based solution of the value iteration in \eqref{eq:value_iteration_barriers}.

\begin{proposition}[Barriers]\label{prop:barrier}
    Suppose Assumptions~\ref{asm:ambiguity_set} \& \ref{asm:markovianity} hold.
    If there exist constants $\gamma>\eta\geq0$ and a function $B\colon\Xdomain\to\R_{\geq0}$ satisfying $B\in\Hilbert$ with $\norm{B}_\Hilbert<\infty$ such that
    \begin{enumerate}
        \item[(a)] $B(x)\leq\eta$ for all $x\in\Xdomain_0$;
        \item[(b)] $B(x)\geq\gamma$ for all $x\in\Xdomain\setminus S$; and
        \item[(c)] $\beta \geq \sup_{x\in S} \sum_{i=1}^N w_i(x) B(x^{(i)}) - B(x) + \varepsilon\kappa\norm{B}_\Hilbert$;
    \end{enumerate}
    with $\kappa$ as in Proposition~\ref{prop:empirical_value_update},
    then $P^S_{\lawseq}(\Xdomain_0) \geq 1 - (\eta+\beta T)/\gamma$ with probability at least $1-\conf$.
\end{proposition}

\begin{remark}[Conservatism]
Even in the model-based setting, the bound \eqref{eq:kushner} generally prevents achieving the optimal solution of \eqref{eq:value_iteration_barriers}, attained for $B(x)=\1_{\Xdomain\setminus S}(x)$.
Moreover, for such discontinuous barriers, the corresponding RKHS norm $\norm{B}_\Hilbert$ is infinite for common kernels (e.g., Gaussian, Matérn). 
Smooth barriers yield a smaller $\varepsilon\kappa\norm{B}_\Hilbert$ term, yet also a larger per-step deviation $\sup_{x\in S}(\E_\law[B(X_+)\mid X=x]-B(x))$.
Thus, in the data-driven regime, one must trade barrier smoothness against the resulting error amplification.
As the robustified bound on $\beta$ compounds over all $T$ recursions in \eqref{eq:value_iteration_barriers}, the conservatism of barrier-based certificates grows rapidly with the horizon.
Unlike abstraction-based methods (see Subsection~\ref{sec:safety_via_finite_abstractions}), where the safety estimate converges to the true probability as $N\!\to\!\infty$ and the discretization vanishes, the barrier approach remains inherently biased, even in the model-based limit.
\end{remark}

\section{Proposition~\ref{prop:empirical_value_update}}\label{app:proof_empirical_value_update}
\begin{proof}
    Via \eqref{eq:innerProdCond} we reformulate the value update $V_l(x)$ in Proposition~\ref{prop:value_iteration} and use the ambiguity set $\amb_{\hat{N}}$ to obtain
    \begin{align*}
        V_l(x)  &= \1_{S}(x) \innerH{V_{l+1}}{\me_{\Hilbert|\Hilbert}^{\law}(x)}{\Hilbert},\\
                &\geq \inf_{\mu\in\amb_{\hat{N}}} \1_{S}(x) \innerH{V_{l+1}}{\me_{\Hilbert|\Hilbert}^{\mu}(x)}{\Hilbert}.
    \end{align*}
    Using Cauchy--Schwarz and $V_l(x)\in[0,1]$, $x\in\Xdomain$, yields
    \begin{equation*}
        V_l(x)  \geq \1_{S}(x)\Bigg[ \inf_{\mu\in\amb_{\hat{N}}}  \left( \innerH{V_{l+1}}{\hat\me_{\Hilbert|\Hilbert}^{\data_{\hat{N}}}(x)}{\Hilbert}  - \innerH{V_{l+1}}{\me_{\Hilbert|\Hilbert}^{\mu}(x) - \hat\me_{\Hilbert|\Hilbert}^{\data_{\hat{N}}}(x)}{\Hilbert} \right) \Bigg]_0^1.
    \end{equation*}
    Here, the truncation $[\cdotx]_0^1$ ensures the value remains a valid probability at each recursion.
    Via Cauchy--Schwarz and the definition of $\amb_{\hat{N}}$ in \eqref{eq:amb_set} we obtain the bound $\underline V_l(x)$, which holds for all $x\in\Xdomain_0$. 
\end{proof}

\begin{rem_}[Convergence rate at the first recursion step]
The penalty $\varepsilon\kappa\Vert\underline V_{l+1}\Vert_\Hilbert$ is finite whenever $\underline V_{l+1}$ is smooth, which holds for all steps $l < T-1$.
At the first recursion step ($l = T-1$), however, we have $\underline V_T = \1_S$, which does not lie in common RKHSs (e.g., Gaussian), so the penalty formally diverges.
Establishing a convergence rate for the ambiguity-set radius $\varepsilon$ at this step would in principle require an extension-theorem argument analogous to that used in Theorem~\ref{thm:empirical_safety_nonMarkov}, appealing to the Besov-space embedding of \cite{kanagawa2014recovering}.
We omit these details for brevity since the issue arises only at the single step $l = T-1$; for all subsequent steps the standard RKHS argument applies without modification.

The indirect methods we recover as special cases are not affected by this issue in the same way: abstraction-based methods employ a finite, inherently discontinuous kernel so RKHS membership of $\1_S$ is not required, and barrier-based methods work with smooth Lyapunov-type functions throughout.
In our experiments we set $\varepsilon = 0$ for the DP baseline, bypassing the issue entirely; we flag it here for completeness.
\end{rem_}

\section{Additional Details on Finite Embeddings}\label{app:additional_material_abstractions}
\begin{figure}[ht]
    \centering
    \begin{subfigure}[t]{0.48\linewidth}
        \centering
    \def\svgwidth{1.0\linewidth}
\begingroup%
  \makeatletter%
  \providecommand\color[2][]{%
    \errmessage{(Inkscape) Color is used for the text in Inkscape, but the package 'color.sty' is not loaded}%
    \renewcommand\color[2][]{}%
  }%
  \providecommand\transparent[1]{%
    \errmessage{(Inkscape) Transparency is used (non-zero) for the text in Inkscape, but the package 'transparent.sty' is not loaded}%
    \renewcommand\transparent[1]{}%
  }%
  \providecommand\rotatebox[2]{#2}%
  \newcommand*\fsize{\dimexpr\f@size pt\relax}%
  \newcommand*\lineheight[1]{\fontsize{\fsize}{#1\fsize}\selectfont}%
  \ifx\svgwidth\undefined%
    \setlength{\unitlength}{297.87482314bp}%
    \ifx\svgscale\undefined%
      \relax%
    \else%
      \setlength{\unitlength}{\unitlength * \real{\svgscale}}%
    \fi%
  \else%
    \setlength{\unitlength}{\svgwidth}%
  \fi%
  \global\let\svgwidth\undefined%
  \global\let\svgscale\undefined%
  \makeatother%
  \begin{picture}(1,0.4246302)%
    \lineheight{1}%
    \setlength\tabcolsep{0pt}%
    \put(0,0){\includegraphics[width=\unitlength,page=1]{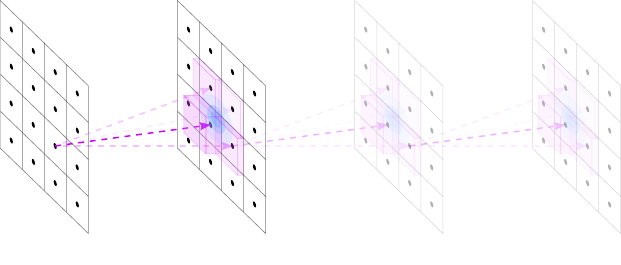}}%
    \put(0.35698053,0.00345874){\color[rgb]{0,0,0}\makebox(0,0)[t]{\smash{\begin{tabular}[t]{c}$t=1$\end{tabular}}}}%
    \put(0.35657098,0.10487653){\color[rgb]{0.78431373,0,1}\makebox(0,0)[t]{\smash{\begin{tabular}[t]{c}$\phi_f(\x_1)\in\Hilbert_{f}$\end{tabular}}}}%
    \put(0.07149672,0.00345874){\color[rgb]{0,0,0}\makebox(0,0)[t]{\smash{\begin{tabular}[t]{c}$t=0$\end{tabular}}}}%
    \put(0.64247052,0.00345874){\color[rgb]{0,0,0}\makebox(0,0)[t]{\smash{\begin{tabular}[t]{c}$t=2$\end{tabular}}}}%
    \put(0.92888024,0.00345874){\color[rgb]{0,0,0}\makebox(0,0)[t]{\smash{\begin{tabular}[t]{c}$t=3$\end{tabular}}}}%
  \end{picture}%
\endgroup%

        \caption{Finite-state Markov chain: $\Hilbert_f$ encodes continuous states and the underlying continuous transition probability distribution (in blue) by mapping them to a finite set of representative points $\hat x_i$ (in black), yielding a discrete transition probability distribution (in purple).}
        \label{fig:finite_abstractions}
    \end{subfigure}
    \hfill
    \begin{subfigure}[t]{0.48\linewidth}
        \centering
    \def\svgwidth{1.0\linewidth}
\begingroup%
  \makeatletter%
  \providecommand\color[2][]{%
    \errmessage{(Inkscape) Color is used for the text in Inkscape, but the package 'color.sty' is not loaded}%
    \renewcommand\color[2][]{}%
  }%
  \providecommand\transparent[1]{%
    \errmessage{(Inkscape) Transparency is used (non-zero) for the text in Inkscape, but the package 'transparent.sty' is not loaded}%
    \renewcommand\transparent[1]{}%
  }%
  \providecommand\rotatebox[2]{#2}%
  \newcommand*\fsize{\dimexpr\f@size pt\relax}%
  \newcommand*\lineheight[1]{\fontsize{\fsize}{#1\fsize}\selectfont}%
  \ifx\svgwidth\undefined%
    \setlength{\unitlength}{297.87482314bp}%
    \ifx\svgscale\undefined%
      \relax%
    \else%
      \setlength{\unitlength}{\unitlength * \real{\svgscale}}%
    \fi%
  \else%
    \setlength{\unitlength}{\svgwidth}%
  \fi%
  \global\let\svgwidth\undefined%
  \global\let\svgscale\undefined%
  \makeatother%
  \begin{picture}(1,0.4246302)%
    \lineheight{1}%
    \setlength\tabcolsep{0pt}%
    \put(0,0){\includegraphics[width=\unitlength,page=1]{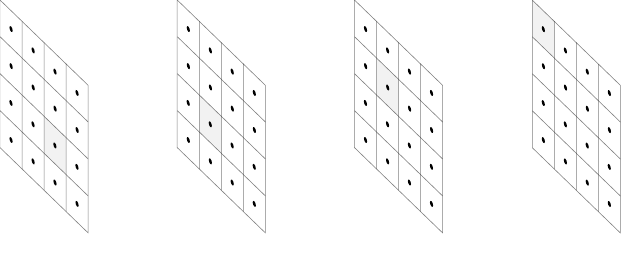}}%
    \put(0.35659078,0.00432366){\color[rgb]{0,0,0}\makebox(0,0)[t]{\smash{\begin{tabular}[t]{c}$t=1$\end{tabular}}}}%
    \put(0.07110704,0.00432366){\color[rgb]{0,0,0}\makebox(0,0)[t]{\smash{\begin{tabular}[t]{c}$t=0$\end{tabular}}}}%
    \put(0.49493058,0.21682025){\color[rgb]{0.78431373,0,1}\transparent{0.92156899}\makebox(0,0)[t]{\smash{\begin{tabular}[t]{c}$\x$\end{tabular}}}}%
    \put(0.47310915,0.28585787){\color[rgb]{1,0.69803922,0}\makebox(0,0)[t]{\smash{\begin{tabular}[t]{c}$\y$\end{tabular}}}}%
    \put(0.64208077,0.00432366){\color[rgb]{0,0,0}\makebox(0,0)[t]{\smash{\begin{tabular}[t]{c}$t=2$\end{tabular}}}}%
    \put(0.92849049,0.00432366){\color[rgb]{0,0,0}\makebox(0,0)[t]{\smash{\begin{tabular}[t]{c}$t=3$\end{tabular}}}}%
    \put(0,0){\includegraphics[width=\unitlength,page=2]{finite_markov_chain_trajectories.pdf}}%
  \end{picture}%
\endgroup%

        \caption{Two equivalent trajectories $\x,\y\in\Xdomain^4$, i.e., $\phi_f(\x)=\phi_f(\y)$.}
        \label{fig:finite_state_markov_chain_trajectories}
    \end{subfigure}
    \caption{Finite-state abstractions in the embedding framework.}
    \label{fig:finite_abstractions_combined}
\end{figure}

\begin{rem_}[Equivalence]
    The discrete embedding inherently induces \emph{equivalence classes of states} at each time step $t$ (see Fig.~\ref{fig:finite_state_markov_chain_trajectories}), so that finite abstractions can be interpreted as identifying trajectories that are indistinguishable under this partition.
    Since the kernel $k_f$ is not characteristic, its associated mean embedding does not encode probability distributions uniquely; consequently, the resulting CME characterizes not a single process, but an \emph{equivalence class of processes} consistent with the abstraction (cf. Fig.~\ref{fig:finite_state_markov_chain_trajectories}).
\end{rem_}

\subsection{Two Errors in Finite-State Abstractions}
In the Markovian setting of Asm.~\ref{asm:markovianity}, the stochastic law $\law$ is encoded as a (finite-dimensional) CME $\me_{\Hilbert_f|\Hilbert_f}^{\law}\colon\Xdomain\to\Hilbert_f$, which can be estimated empirically from data $\smash{\data_N}$ as $\smash{\hat\me_{\Hilbert_f|\Hilbert_f}^{\data_N}}$ with weights $w_j\colon\Xdomain_f\to\R$.
The corresponding robust value update in Prop.~\ref{prop:value_iteration} becomes\footnote{Note that using $\smash{\Hilbert_f}$ may require replacing the safe set $S$ with a subset $\hat S\subset\Xdomain_f$ such that $\hat S\subset S$.}
\begin{equation}
     V_l(\hat{x}_i) \geq \inf_{\bar\mu\in\amb^f_{\hat N},\,x\in A_i} \1_{S}(x)
    \innerH{V_{l+1}}{\me_{\Hilbert_f|\Hilbert_f}^{\bar\mu}(\hat{x}_i)}{\Hilbert_f}
    ,\label{eq:robust_finite_value_update}
\end{equation}
for each partition $A_i\subset\Xdomain$ represented by a symbol $\hat{x}_i$, where $V_{l+1}$ is piecewise constant over the partition elements.
The ambiguity set $\amb^f_{\hat N}$ accounts for both \emph{statistical error} (finite data) and \emph{abstraction error} (discretization). A possible construction analogous to \eqref{eq:amb_set} is
\begin{equation}
     \amb^f_{\hat N} := \left\lbrace \bar\mu\colon\Xdomain\times\borel{\Xdomain}\to[0,1] \,\Big\vert\, \Big\vert\Big\vert \hat\me_{\Hilbert_f|\Hilbert_f}^{\data_{\hat N}} - \Pi_{\mathcal{G}_f}\left(\me_{\Hilbert|\Hilbert}^{\bar\mu}\right) \Big\vert\Big\vert_{\mathcal{G}_f}\leq\varepsilon
    \right\rbrace,\label{eq:amb_set_finite}
\end{equation}
where $\Pi_{\mathcal{G}_f}$ is the orthogonal projection from $\mathcal{G}$ onto the finite vector-valued RKHS $\mathcal{G}_f$ of functions $\Xdomain_f\to\Hilbert_f$.
In this context, the projection $\Pi_{\Hilbert_f}$ captures the abstraction-induced approximation of the true continuous embedding within the partitioned space.
Both error sources enter at every step of the DP recursion in \eqref{eq:robust_finite_value_update} and therefore compound over the horizon $T$, in contrast to the one-time errors $\varepsilon_1, \varepsilon_2, \varepsilon_3$ of the direct estimator (Remark~\ref{rem:error_structure}).

\section{Simulation Relations and Interval Markov Processes}\label{app:sim_rel_imp}

\paragraph{Probabilistic Simulation Relations.}
\emph{Probabilistic simulation relations} provide a structured way to bound the mismatch between a (learned) finite abstraction and the true stochastic dynamics, allowing us to transfer (safety) guarantees between them~\citep{haesaert2017verification}. In this context, we write ${\SSR{\smash{\hat\me_{\Hilbert_f|\Hilbert_f}^{\data_{\hat N}}}}{\me_{\Hilbert_f|\Hilbert_f}^{\law}}{\epsilon}{\delta}}$ to mean that the system captured by the empirical CME $\smash{\hat\me_{\Hilbert_f|\Hilbert_f}^{\data_{\hat N}}}$ is in an \emph{$(\epsilon,\delta)$-sub-simulation relation} (SSR) \citep[Definition~6]{schoen2025bayesian} with the unknown true data-generating system governed by $\law$, where $\epsilon\geq0$ and $\delta\colon\Xdomain_f\to[0,1]$ quantify the output mismatch and unmatched step-wise transition probability, respectively.\footnote{Note that the error bound $\delta$ used exclusively in this section is distinct from the scalar confidence level $\conf\in[0,1]$ used in the main body of the paper.}

The following result shows that in the case of data-driven safety verification via the empirical CME, the SSR induces an ambiguity set in the total variation distance:
\begin{equation}
    \amb^{\text{SSR}}_{\hat{N}}(\hat x_i) := 
    \textstyle\Big\lbrace \bar\mu\colon\Xdomain\times\borel{\Xdomain}\to[0,1] \,\Big\vert\,\label{eq:amb_set_SSR}
    \Big\Vert\sum_{l=1}^{\hat{N}} w_l(\hat x_i) \1_{\cdotx}(x_+^{(l)})-\min_{x\in A_i} \bar\mu(\cdotx|x)\Big\Vert_{\text{TV}} \!\leq \delta(\hat x_i)\!
    \Big\rbrace,
\end{equation}%
where $\delta$ bounds the total variation $\norm{\hat p-\bar\mu}_{\text{TV}}:=1-\smash{\sum_{j=1}^n\min\lbrace \hat p(A_j),\,\allowbreak \bar\mu(A_j) \rbrace}$,
whilst $\epsilon$ bounds the discretization error of $\1_S(x)$ w.r.t. $x\in A_i$ in the infimum of \eqref{eq:robust_finite_value_update}.
The SSR yields a robust value update resulting directly from Proposition~1 by \citet{schoen2025bayesian}:
\begin{theorem}[SSR-Based Dynamic Programming]\label{thm:ssr}
    Suppose Asms.~\ref{asm:ambiguity_set} \& \ref{asm:markovianity} hold with the ambiguity set $\amb^{\text{SSR}}_{\hat{N}}$ in \eqref{eq:amb_set_SSR}.
    Then, we have ${\SSR{\smash{\hat\me_{\Hilbert_f|\Hilbert_f}^{\data_{\hat N}}}}{\me_{\Hilbert_f|\Hilbert_f}^{\law}}{\epsilon}{\delta}}$, $\smash{\hat\me_{\Hilbert_f|\Hilbert_f}^{\data_{\hat N}}}$ being the empirical CME with weights $w_j\colon\Xdomain_f\to\R$. 
    Furthermore, let $\underline V_l\colon\Xdomain_f\to[0,1]$, with $l=T,\ldots,0$, such that $\underline V_T(\hat{x}_i) := \1_{S}(\hat{x}_i)$ and
    \begin{equation*}
        \underline V_l(\hat{x}_i) := \!\left[\1(A_i\subset S) \sum_{j=1}^{\hat N} w_j(\hat{x}_i) \underline V_{l+1}(\Pi_{\Xdomain_f}\!(x_+^{(j)}))  - \delta\right]_0^1\!\!.
    \end{equation*}
    Then, we have $P^S_\lawseq(x_0)\geq\underline V_0(\Pi_{\Xdomain_f}\!(x_0))$ for all $x_0\in\Xdomain_0$ with probability at least $1-\conf$.
\end{theorem}

Intuitively, SSR-based verification can be seen as a finite-state realization of the robust CME-based value iteration in Proposition~\ref{prop:empirical_value_update}, with the SSR's sub-probability coupling implicitly defining the ambiguity set $\smash{\amb^{\mathrm{SSR}}_{\hat N}}$.
In this view, $\delta$ directly mirrors the robustness penalty $\varepsilon\kappa\smash{\norm{\smash{\underline V_{l+1}}}_\Hilbert}$ appearing in the continuous-state setting of Proposition~\ref{prop:empirical_value_update}.
The proof of Theorem~\ref{thm:ssr} is given in Appendix~\ref{app:proof_ssr}.

\paragraph{Interval Markov Processes.}
Perhaps the most widely used abstraction-based verification framework in the data-driven regime is that of \emph{interval Markov processes} (IMPs)~\citep{suilen2025robust}.
IMPs are a special instance of \emph{robust Markov processes} (RMPs) that models ambiguity by assigning probabilistic upper and lower bounds to each transition; i.e., an IMP defines functions $\check{p},\hat{p}\colon\Xdomain_f\times\borel{\Xdomain_f}\to[0,1]$ such that for all $x\in\Xdomain$ and $\hat{x}_j\in\Xdomain_f$ we have
$\check{p}(\hat{x}_j|\Pi_{\Xdomain_f}\!(x))\leq\law(A_j|x)\leq\hat{p}(\hat{x}_j|\Pi_{\Xdomain_f}\!(x)).$
This induces an $\hat{x}_i$-\emph{rectangular} ambiguity set~\citep{suilen2025robust}.
In the data-driven regime, the interval bounds $[\check{p},\hat{p}]$ are often constructed symmetrically around the learned estimator, giving rise to an error function $\epsilon\colon\Xdomain_f\times\borel{\Xdomain_f}\to[0,1]$ that controls the radius of the uncertainty interval:
\begin{align}
  \begin{split}\textstyle
    &\amb_{\hat{N}}^{\text{IMP}}(\hat{x}_i) = \Big\lbrace \bar\mu\colon\Xdomain\times\borel{\Xdomain}\to[0,1] \,\Big\vert\,
    \Big(\forall x\in A_i\colon \textstyle\sum_{\hat{x}_j\in\Xdomain_f}\bar\mu(A_j|x)=1\Big)
    \\
    &\textstyle\qquad\wedge
    \Big(\forall x\in A_i,\,\hat{x}_j\in\Xdomain_f\colon \Big\vert\,\bar\mu(A_j|x)-
    \sum_{l=1}^{\hat{N}} w_l(\hat x_i)\, \1_{A_j}(x_+^{(l)})
    \Big\vert \leq \epsilon(\hat{x}_j|\hat{x}_i)
    \Big)
    \Big\rbrace.
  \end{split}\label{eq:amb_set_IMP}
\end{align}

The following result is a direct consequence of \eqref{eq:amb_set_IMP} applied to \eqref{eq:robust_finite_value_update}, characterizing the robust value update of IMP-based approaches.
\begin{proposition}[IMP-Based Dynamic Programming]\label{prop:imp}
    Suppose Asms.~\ref{asm:markovianity} \& \ref{asm:ambiguity_set} hold with the ambiguity set $\amb^{\text{IMP}}_{\hat{N}}$ in \eqref{eq:amb_set_IMP}.
    Furthermore, let $\underline V_l\colon\Xdomain_f\to[0,1]$, with $l=T,\ldots,0$, such that $\underline V_T(\hat{x}_i) := \1_{S}(\hat{x}_i)$ and
    \begin{equation}
        \underline V_l(\hat{x}_i) := \1(A_i\subset S) \,\inf_{\bar\mu\in\amb_{\hat{N}}^{\text{IMP}}(\hat{x}_i)} \sum_{j=1}^{n} \bar\mu(A_j|\hat{x}_i)\, \underline V_{l+1}(\hat{x}_j).\label{eq:value_update_imc}
    \end{equation}
    Then, we have $P^S_\lawseq(x_0)\geq\underline V_0(\Pi_{\Xdomain_f}\!(x_0))$ for all $x_0\in\Xdomain_0$ with probability at least $1-\conf$.
\end{proposition}

The inner problem in \eqref{eq:value_update_imc} is a linear program that can be computed efficiently using, e.g., order-maximization \citep{givan2000bounded,mathiesen2024intervalmdp}.
Note that in comparison to Proposition~\ref{prop:empirical_value_update} and Theorem~\ref{thm:ssr}, the robust safety probability from Proposition~\ref{prop:imp} cannot be evaluated from the empirical CME directly.
In fact, in IMP-based safety verification most of the computational effort lies in determining the interval bounds $[\check{p},\hat{p}]$ that populate $\smash{\amb_{\hat{N}}^{\text{IMP}}}$, typically by propagating uncertainty from the learned model~\citep{mathiesen2024intervalmdp,wooding2024impact,schon2024btgp}---here, the empirical CME estimator.

\section{Proof of Theorem~\ref{thm:ssr}}\label{app:proof_ssr}
\textbf{Validity of the TV identity.}
The total variation distance identity $\norm{\hat p - \mu}_{\text{TV}} = 1 - \sum_j \min\{\hat p(A_j), \mu(A_j)\}$ and the subsequent coupling construction require $\hat p$ to be a non-negative probability measure.
The empirical probabilities \smash{$\hat p(A_j|\hat x_i) = \sum_{l=1}^{\hat N} w_l(\hat x_i)\,\1_{A_j}(x_+^{(l)})$} are formed from CME weights $w_l$, which can in principle be negative; in practice we clip the estimates to $[0,1]$ and renormalize to obtain a valid probability measure, which does not affect the asymptotic consistency of the CME.
The remaining proof proceeds under the assumption that $\hat p(\cdotx|\hat x_i)$ is a valid probability measure.
\begin{proof}
    We start by showing that the ambiguity set $\amb^{\text{SSR}}_{\hat{N}}$ of the form \eqref{eq:amb_set_SSR} indeed enforces an SSR \citep[Def.~6]{schoen2025bayesian}.
    For this, we select the measurable relation
    \begin{align*}
        \mathcal{R} &:= \{(\hat{x},x)\in\borel{\Xdomain_f\times\Xdomain}\mid \hat{x}=\Pi_{\Xdomain_f}\!(x)\},\\
        &\equiv \{(\hat{x}_i,x)\in\borel{\Xdomain_f\times\Xdomain}\mid x\in A_i\},
    \end{align*}
    induced by the partitioning $\Pi_{\Xdomain_f}$, which naturally relates a symbol $\hat{x}_i\in\Xdomain_f\subset\Xdomain$ with all points in its partition $A_i\subset\Xdomain$.
    Note that $\epsilon$ bounds the associated discretization error, i.e., $\epsilon\geq\sup_{(\hat{x},x)\in\mathcal{R}}\norm{\hat{x}-x}_2$ for all $\hat{x}\in\Xdomain_f$ \citep[Def.~6(c)]{schoen2025bayesian}.
    The initial condition $(\hat{x}_0,x_0)\in\mathcal{R}$, $\forall x_0\in\Xdomain_0$ \citep[Def.~6(a)]{schoen2025bayesian} holds by selecting $\hat{x}_0\in\Xdomain_{f,0}:=\Pi_{\Xdomain_f}\!(\Xdomain_0)$.
    For the case of verification, it thus suffices to show that there exists a sub-probability kernel $\mathbf{v}\colon\Xdomain_f\times\Xdomain\times\borel{\Xdomain_f\times\Xdomain}\to[0,1]$ that is a sub-probability coupling \citep[Def.~5]{schoen2025bayesian} of $$\hat p(d\hat{x}_i|\hat{x}):=\innerH{\1_{A_i}}{\hat\me_{\Hilbert_f|\Hilbert_f}^{\data_{\hat N}}(\hat{x})}{\Hilbert_f}$$ and $\law$ over $\mathcal{R}$ w.r.t. $\delta$ for all $(\hat{x},x)\in\mathcal{R}$~\citep[Def.~6(b)]{schoen2025bayesian}.
    To this end, there always exist a sub-probability kernel of the form $$\mathbf{v}(d\hat{x}_i\times B|\hat{x},x):=\delta_{d\hat{x}_i}(\Pi_{\Xdomain_f}\!(B))\min\{\hat p(d\hat{x}_i|\hat{x}),\,\law(B|x)\},$$ 
    satisfying the conditions Def.~5(a)--(c) of \citet{schoen2025bayesian} by construction.
    It remains to show that $\mathbf{v}(\mathcal{R}|\hat{x},x)\equiv\mathbf{v}(\Xdomain_f\times\Xdomain|\hat{x},x)\geq 1-\delta(\hat{x})$ for all $(\hat{x},x)\in\mathcal{R}$, revealing
    \begin{align*}
        \mathbf{v}(\Xdomain_f\times\Xdomain|\hat{x},x) &= \int_{A_j\in\Xdomain} \!\!\!\min\!\left\lbrace\hat p(\Pi_{\Xdomain_f}\!(A_j)|\hat{x}),\,\law(A_j|x)\right\rbrace,\\
        &= \sum_{j=1}^n \min\!\left\lbrace\hat p(\hat{x}_j|\hat{x}),\,\law(A_j|x)\right\rbrace,\\
        &= \sum_{j=1}^n \min\!\left\lbrace \innerH{\1_{A_j}}{\hat\me_{\Hilbert_f|\Hilbert_f}^{\data_{\hat N}}(\hat{x})}{\Hilbert_f}\!\!,\,\law(A_j|x)\right\rbrace.
    \end{align*}
    As via $(\hat{x},x)\in\mathcal{R}$ we have $\hat{x}=\Pi_{\Xdomain_f}\!(x)$, it follows that
    \begin{align*}
        \mathbf{v}(\Xdomain_f\times\Xdomain|\hat{x}_i,x) &\geq \sum_{j=1}^n \min\left\lbrace \innerH{\1_{A_j}}{\hat\me_{\Hilbert_f|\Hilbert_f}^{\data_{\hat N}}(\hat{x}_i)}{\Hilbert_f},\, \min_{x\in A_i}\law(A_j|x)\right\rbrace.
    \end{align*}
    According to Assumption~\ref{asm:ambiguity_set}, we have $\law\in\amb^{\text{SSR}}_{\hat{N}}$ with probability at least $1-\conf$, yielding that for all $i=1,\ldots,n$:
    \begin{equation*}
        \sum_{j=1}^n\min\Big\lbrace 
        \innerH{\1_{A_j}}{\hat\me_{\Hilbert_f|\Hilbert_f}^{\data_{\hat N}}(\hat{x}_i)}{\Hilbert_f},\,
        \min_{x\in A_i} \law(A_j\mid x)
         \Big\rbrace \geq 1-\delta(\hat x_i),
    \end{equation*}
    thus proving that $\mathbf{v}$ is indeed a valid sub-probability coupling. 
    Recognizing that $\norm{\hat p-\mu}_{\text{TV}}:=1-\sum_{j=1}^n\min\lbrace \hat p(A_j),\,\allowbreak \mu(A_j) \rbrace$ is the total variation distance between two discrete probability measures $\hat p,\mu\colon\Xdomain\to[0,1]$, we obtain \eqref{eq:amb_set_SSR}.
    This concludes the first part of the proof. 
    
    The form of the robust value update $\underline V_l(\hat{x}_i)$ follows directly from the \emph{$(\epsilon,\delta)$-robust Bellman operator} associated with the SSR \citep[Eq.~(13)]{schoen2025bayesian}, where the concrete transition kernel $\hat{\mathbf{t}}$ is replaced with the empirical CME, and the minimum w.r.t. the $\varepsilon$-robustified DFA update reduces to $\1(A_i\subset S)$ for a safety problem. The robust lower bound on the safety probability follows from Proposition~\ref{prop:value_iteration}, by applying the $(\epsilon,\delta)$-robust Bellman operator recursively over the horizon $T$, concluding the proof. 
\end{proof}

\section{Proof of Theorem~\ref{thm:empirical_safety_nonMarkov}}\label{app:proof_empirical_safety_nonMarkov}
\begin{proof}    
    The proof follows the ideas of \cite{kanagawa2014recovering}.
    We rewrite \eqref{eq:safety_probability} with $\rho^S$ to 
    \begin{align*}
        P^S_{\lawseq}(x) &= \E_{\lawseq}\!\!\left[ \rho^S(\X_{0:T})\mid X_0=x\right],\\
        &\geq \sum_{i=1}^N w_i(x)\, \rho^S(\x_{0:T}^{(i)}) - \Bigg\vert \sum_{i=1}^N w_i(x)\, \rho^S(\x_{0:T}^{(i)}) - \E_{\lawseq}\!\!\left[ \rho^S(\X_{0:T})\mid X_0=x\right] \Bigg\vert.
    \end{align*}
    The second term can be bounded through three terms:
    \begin{align*}
        &\left\vert \sum_{i=1}^N w_i(x)\, \rho^S(\x_{0:T}^{(i)}) - \E_{\lawseq}\!\!\left[ \rho^S(\X_{0:T})\mid X_0=x\right] \right\vert\\
        &\quad\leq\left\vert \sum_{i=1}^N w_i(x)\! \left(\rho^S(\x_{0:T}^{(i)}) -  \tilde\rho^S_N(\x_{0:T}^{(i)})\right) \right\vert\\
        &\qquad + \left\vert \sum_{i=1}^N w_i(x)\, \tilde\rho^S_N(\x_{0:T}^{(i)}) - \E_{\lawseq}\!\!\left[ \tilde\rho^S_N(\X_{0:T})\mid X_0=x\right] \right\vert\\
        &\qquad + \left\vert \E_{\lawseq}\!\!\left[ \tilde\rho^S_N(\X_{0:T}) - \rho^S(\X_{0:T}) \mid X_0=x\right] \right\vert.
    \end{align*}
    The first term is the empirical convolution error $\varepsilon_1(x)$.
    The second term reduces as follows.
    Note that $\tilde\rho^S_N\in\Hilbert_{\gamma_N}$ by construction (it is defined as the convolution with $K_{\gamma_N}$), so the reproducing property in $\Hilbert_{\gamma_N}^{T+1}$ applies:
    \begin{align*}
        &\left\vert \sum_{i=1}^N w_i(x)\, \tilde\rho^S_N(\x_{0:T}^{(i)}) - \E_{\lawseq}\!\!\left[ \tilde\rho^S_N(\X_{0:T})\mid X_0=x\right] \right\vert\\
        &\qquad= \left\vert\innerH{\tilde\rho^S_N}{\hat\me_{\Hilbert|\Hilbert^{T+1}}^{\data_{N}}(x) \!-\! \me_{\Hilbert|\Hilbert^{T+1}}^{\lawseq}(x)}{\Hilbert_{\gamma_N}}\right\vert,\\
        &\qquad\leq \norm{\tilde\rho^S_N}_{\Hilbert_{\gamma_N}}
        \norm{\hat\me_{\Hilbert|\Hilbert^{T+1}}^{\data_{N}} \!-\! \me_{\Hilbert|\Hilbert^{T+1}}^{\lawseq}}_{\Hilbert_{\gamma_N}}
        \sqrt{k_\gamma(x,x)},\\
        &\qquad\leq \norm{\tilde\rho^S_N}_{\Hilbert_{\gamma_N}}
        \!\cdot\left(\frac{\gamma}{\gamma_N}\right)^{\frac{d(T+1)}{2}}\!\norm{\hat\me_{\Hilbert|\Hilbert^{T+1}}^{\data_{N}} \!-\! \me_{\Hilbert|\Hilbert^{T+1}}^{\lawseq}}_{\Hilbert}
        \!\cdot\kappa,
    \end{align*}
    which reduces to $\varepsilon_2$, where the factor $(\gamma/\gamma_N)^{d(T+1)/2}$ arises from the norm conversion between product Gaussian RKHSs of dimension $d(T+1)$ with different bandwidths.
    The third term is bounded as follows
    \begin{equation*}
        \left\vert \E_{\lawseq}\!\!\left[ \tilde\rho^S_N(\X_{0:T}) - \rho^S(\X_{0:T}) \mid X_0=x\right] \right\vert\leq \norm{\tilde\rho^S_N - \rho^S}_{L_1(\lawseq)} \leq \norm{\tilde\rho^S_N - \rho^S}_{L_2(\lawseq)},
    \end{equation*}
    yielding $\varepsilon_3 = \norm{\tilde\rho^S_N - \rho^S}_{L_\infty(\Xdomain^{T+1})}$, which converges to zero for $N\to\infty$ (and $T$ fixed) by the $L_\infty$ approximation properties of Gaussian mollification in Besov spaces \citep[Proof of Theorem~1]{kanagawa2014recovering}.
\end{proof}

\section{Certified Lower Bounds via Histogram Binning}\label{app:conformal}

The goal is a \emph{certified lower bound} $p(x_0) \leq P^\mathcal{S}(x_0)$ that holds with high probability over the randomness in the calibration data, without any parametric or Markov assumption.
Standard split-conformal prediction \citep{vovk2005algorithmic} with score $s_i = \hat{p}(x_0^{(i)}) - y_i$ yields a \emph{prediction-set} guarantee covering the binary outcome $Y$, not a confidence interval for $E[Y|X=x_0] = P^\mathcal{S}_{\lawseq}(x_0)$ \citep{gupta2020distribution}.
For safety certification the latter is what is needed: we want to certify a specific initial condition $x_0$, not predict whether a single random rollout will be safe.

\paragraph{Histogram binning \citep{gupta2020distribution}.}
Given a held-out calibration set $\{(x_0^{(i)}, y^{(i)})\}_{i=1}^{n_{\mathrm{cal}}}\subset\Xdomain_0\times\{0,1\}$ and a predictor $\hat{p}\colon\Xdomain_0\to[0,1]$, the procedure is as follows:
\begin{enumerate}[leftmargin=*]
    \item \textbf{Bin:} choose quantile-based edges $\tau_0 \leq \ldots \leq \tau_B$ so each bin contains approximately $n_{\mathrm{cal}}/B$ calibration points; define the bin assignment map $m\colon[0,1]\to\{1,\ldots,B\}$, $m(p) = \min\{b \colon p < \tau_b\}$, and let $\mathcal{I}_b = \{i \colon m(\hat{p}(x_0^{(i)})) = b\}$, $n_b = |\mathcal{I}_b|$.
    \item \textbf{Estimate:} for each bin $b$, compute the empirical safety rate $\hat{\pi}_b := \frac{1}{n_b}\sum_{i \in \mathcal{I}_b} y^{(i)}$.
    \item \textbf{Certify:} define the Hoeffding width $\varepsilon_b := \sqrt{\frac{\ln(B/\conf)}{2 n_b}}$ and apply a one-sided bound, Bonferroni-corrected over $B$ bins:
    \[
        p_b = \max\!\left\lbrace0,\; \hat{\pi}_b - \varepsilon_b\right\rbrace.
    \]
    \item \textbf{Assign:} define the piecewise constant certified lower bound $p^{\mathrm{lb}} \colon \Xdomain \to [0,1]$, $p^{\mathrm{lb}} := p \circ m \circ \hat{p}$, where $p\colon\{1,\ldots,B\}\to[0,1]$, $p(b)=p_b$.
\end{enumerate}
Here $\hat{V}_b := \hat\pi_b(1-\hat\pi_b) \leq 1/4$ is the empirical variance bound for binary outcomes.
The guarantee (Theorem~4 in \citet{gupta2020distribution}, specialised to binary outcomes) is
\[
    \P\!\left(P^\mathcal{S}_{\lawseq}(x_0^{\mathrm{new}}) \geq p^{\mathrm{lb}}(x_0^{\mathrm{new}})\right) \geq 1 - \conf,
\]
where the probability is over the joint draw of calibration data $\{(x_0^{(i)},y^{(i)})\}_{i=1}^{n_{\mathrm{cal}}}$ and the new test point $(x_0^{\mathrm{new}},y^{\mathrm{new}})$ from the same distribution.
This is a marginal guarantee over the evaluation distribution; when the predictor $\hat{p}$ is well-calibrated and spatially discriminative, different initial states receive different certified values, enabling effective per-state certification even though the formal guarantee is marginal.

\textbf{Discriminativeness.}\;
The quality of the certificate depends on the quality of the predictor.
If $\hat{p}$ is well-calibrated and spatially varying, different bins have different $\hat{\pi}_b$, yielding a spatially informative $p$ that varies across $x_0\in\Xdomain_0$.
If $\hat{p}$ is degenerate (all predictions identical), all calibration points fall in one bin, and $p^{\mathrm{lb}} \equiv \hat{\pi}_{\mathrm{all}} - \varepsilon_{\mathrm{all}}$ is a single constant equal to the dataset mean minus the Hoeffding correction, where $\varepsilon_{\mathrm{all}} := \sqrt{\ln(1/\conf)/(2n_{\mathrm{cal}})}$ (i.e., $B=1$).
This constant is a valid \emph{marginal} lower bound but not a \emph{pointwise} one: it will be violated wherever $P^\mathcal{S}_{\lawseq}(x_0) < \hat{\pi}_{\mathrm{all}} - \varepsilon_{\mathrm{all}}$, i.e., in every genuinely dangerous region.
We report \emph{soundness} (fraction of evaluation points where $p^{\mathrm{lb}} \leq P^\mathcal{S}_{\mathrm{MC}}\approx P^\mathcal{S}_{\lawseq}$) and \emph{discriminativeness} (std of $p^{\mathrm{lb}}$ across the grid) as diagnostics.

\textbf{Implementation Details.}\;
We use $B = 10$ bins and $\conf = 0.1$ (90\% coverage) throughout.
For the synthetic quadrotor ($N = 1000$ training trajectories, $n_{\mathrm{cal}} = 1000$ calibration trajectories), this gives $n_b = 100$ per bin and a Hoeffding width of $\varepsilon_b \approx 0.15$.
When all predictions are numerically equal (as for the indirect estimator), we collapse to $B = 1$ to avoid spurious spatial variation from floating-point noise.

\section{Quantitative Semantics}\label{app:quantitative_semantics}
Theorem~\ref{thm:empirical_safety_nonMarkov} shows that strict safety indicators introduce smoothing effects and slow convergence.
This connects directly to quantitative semantics~\citep{maler2004monitoring,fainekos2009spatiotemporalrobustness}, where smooth robustness values quantify the margin of satisfaction.
Given a robustness functional $\tilde\rho^S\in\Hilbert^{T+1}$, Theorem~\ref{thm:empirical_safety_nonMarkov} simplifies as follows.
\begin{corollary}[Quantitative Semantics]\label{cor:quantitative}
    Suppose Asm.~\ref{asm:ambiguity_set} holds and let $\tilde\rho^S\in\Hilbert^{T+1}$ be a smooth robustness functional~\citep{maler2004monitoring,fainekos2009spatiotemporalrobustness} satisfying $\tilde\rho^S(\x)\leq\rho^S(\x)$ for all $\x\in\Xdomain^{T+1}$.
    Then
    \begin{equation*}
         P^S_{\lawseq}(x)
        \geq \sum_{i=1}^N w_i(x) \, \tilde\rho^{S}(\x^{(i)}) - \varepsilon\kappa\norm{\tilde\rho^S}_{\Hilbert^{T+1}},
    \end{equation*}
    with probability at least $1-\conf$.
\end{corollary}
\begin{proof}
    Since $\tilde\rho^S\in\Hilbert^{T+1}$, the reproducing property gives $$\E_{\lawseq}[\tilde\rho^S(\X_{0:T})\mid X_0=x]=\innerH{\tilde\rho^S}{\me_{\Hilbert|\Hilbert^{T+1}}^{\lawseq}(x)}{\Hilbert^{T+1}}.$$
    Applying Cauchy--Schwarz and $\me_{\Hilbert|\Hilbert^{T+1}}^{\lawseq}\in\amb_N^T$ (Asm.~\ref{asm:ambiguity_set}) yields the bound directly, with $P^S_{\lawseq}(x)\geq\E_{\lawseq}[\tilde\rho^S(\X_{0:T})\mid X_0=x]$ following from $\tilde\rho^S\leq\rho^S$.
\end{proof}
Notably, general quantitative semantics do not admit an equivalent DP recursion, highlighting that the direct inference approach is inherently more expressive.


\end{document}